\documentclass{article}

\usepackage{microtype}
\usepackage{graphicx}
\usepackage{subfigure}
\usepackage{float}
\usepackage{booktabs} %

\usepackage{hyperref}

\usepackage{arxiv}

\usepackage{balance}

\usepackage{array,booktabs}
\newcolumntype{M}[1]{>{\centering\arraybackslash}m{#1}}

\usepackage{amsthm}
\usepackage{amsmath}
\usepackage{amsfonts}
\usepackage{dsfont}

\usepackage{amssymb}

\usepackage{thmtools,thm-restate}

\usepackage[utf8x]{inputenc} 

\newtheorem{theorem}{Theorem} 
\newtheorem{formalTheorem}{Theorem} 

\newtheorem{assumption}{Assumption} 
 
\newtheorem{claim}{Claim} 

\declaretheorem[name=Lemma]{lemma}
\declaretheorem[name=Corollary]{corollary}

\title{Closing  the convergence gap of SGD without replacement}

\author
{
	  Shashank ~Rajput \\ \email{rajput3@wisc.edu}
	  \and Anant ~Gupta \\ \email{agupta225@wisc.edu}
	  \and Dimitris ~Papailiopoulos\\ \email{dimitris@papail.io}
}

\date{University of Wisconsin - Madison}

\begin{document}
\maketitle

\begin{abstract}
Stochastic gradient descent without replacement sampling is widely used in practice for model training. However, the vast majority of SGD analyses assumes data is sampled with replacement, and when the function minimized is strongly convex, an $\mathcal{O}\left(\frac{1}{T}\right)$ rate can be established when SGD is run for $T$ iterations.
A recent line of breakthrough works on SGD without replacement (SGDo) established an $\mathcal{O}\left(\frac{n}{T^2}\right)$ convergence rate when the function minimized is strongly convex and is a sum of $n$ smooth functions, and an $\mathcal{O}\left(\frac{1}{T^2}+\frac{n^3}{T^3}\right)$ rate for sums of quadratics. On the other hand, the tightest known lower bound postulates an $\Omega\left(\frac{1}{T^2}+\frac{n^2}{T^3}\right)$ rate, leaving open the possibility of better SGDo convergence rates in the general case.
In this paper, we close this gap and show that SGD without replacement achieves a rate of $\mathcal{O}\left(\frac{1}{T^2}+\frac{n^2}{T^3}\right)$ when the sum of the functions is a quadratic, and offer a new lower bound of $\Omega\left(\frac{n}{T^2}\right)$ for strongly convex functions that are sums of smooth functions.
\end{abstract}

\section{Introduction}
Stochastic gradient descent (SGD) is a widely used first order optimization technique used to approximately minimize a sum of functions
\begin{equation*}
    F(x) = \frac{1}{n} \sum_{i=1}^n f_i(x).%
\end{equation*}
In its most general form, SGD produces a series of iterates
$$x_{i+1} = x_i - \alpha\cdot g(x,\xi_i)$$ 
where $x_i$ is the $i$-th iterate, $g(x,\xi_i)$ is a stochastic gradient defined below, $\xi_i$ is a random variable that determines the choice of a single or a subset of sampled functions $f_i$, and $\alpha$ represents the step size.
With- and without replacement sampling of the individual component functions are regarded as some of the most popular variants of SGD. 
During SGD with replacement sampling, the stochastic gradient is equal to $g(x,\xi_i) = \nabla f_{\xi_i}(x)$ and $\xi_i$ is a uniform number in $\{1,\ldots, n\}$, \textit{i.e.}, a with replacement sample from the set of gradients $\nabla f_1,\ldots, \nabla f_n$. 
In the case of without replacement sapling, the stochastic gradient is equal to $g(x,\xi_i) = \nabla f_{\xi_i}(x)$ and $\xi_i$ is the $i$-th ordered element in a random permutation of the numbers in $\{1,\ldots, n\}$, \textit{i.e.}, a without-replacement sample.

In practice, SGD without replacement is much more widely used compared to its with replacement counterpart, as it can empirically converge significantly faster \cite{bottou2009curiously,recht2013parallel, recht2012beneath}. 
However, in the land of theoretical guarantees, with replacement SGD has been the focal point of convergence analyses.
This is because analyzing stochastic gradients sampled with replacement are significantly more tractable. The reason is simple: in expectation, the stochastic gradient is equal to the ``true'' gradient of $F$, \textit{i.e.},
$\mathbb{E}_{\xi_i}\nabla f_{\xi_i}(x) = \nabla F(x)$. 
This makes SGD amenable to analyses very similar to that of vanilla gradient descent (GD), which has been extensively studied under a large variety of function classes and geometric assumptions, {\it e.g.}, see \citet{bubeck2015convex}.

Unfortunately, the same cannot be said for SGD without replacement, which has long resisted non-vacuous convergence guarantees. 
For example, although we have long known that SGD with replacement can achieve a $\mathcal{O}\left(\frac{1}{T}\right)$ rate for strongly convex functions $F$,  for many years the best known bounds for SGD without replacement did not even match that rate, in contrast to empirical evidence.
However, a recent series of breakthrough results on SGD without replacement has established similar or better convergence rates than SGD with replacement.

{\small
\begin{table*}[t]
	\begin{center}
		{	\small
		\renewcommand{\arraystretch}{1.5}
			\begin{tabular}{|>{\centering}m{0.25\textwidth}|c|}
				\hline
				\multicolumn{2}{|c|}{$F$ is strongly convex and a sum of $n$ quadratics} \\
				\hline
				Lower bound, \citet{safran2019good}&  $\Omega\left(\dfrac{1}{T^2}+\dfrac{n^2}{T^3}\right)$ \\
				\hline
				Upper bound, \citet{haochen2018random}    & {\color{magenta}$\tilde{\mathcal{O}}\left(\dfrac{1}{T^2}+\dfrac{n^3}{T^3}\right)$} \\
				\hline 
				Our upper bound, Theorem~\ref{thm:upperBound} & {\color{cyan}$\tilde{\mathcal{O}}\left(\dfrac{1}{T^2}+\dfrac{n^2}{T^3}\right)$}\\
				\hline
			\end{tabular}
			\begin{tabular}{|>{\centering}m{0.25\textwidth}|c|}
				\hline 
				\multicolumn{2}{|c|}{$F$ is strongly convex and a sum of $n$ smooth functions} \\
				\hline
				Lower bound, \citet{safran2019good}  &  {\color{magenta}$\Omega\left(\dfrac{1}{T^2}+\dfrac{n^2}{T^3}\right)$} \\
				\hline
				Upper bound, \citet{jain2019sgd}    & $\tilde{\mathcal{O}}\left(\dfrac{n^{\vphantom{1}}}{T^2}\right)$ \\
				\hline
				Our lower bound, Theorem~\ref{thm:lowerBound} &  {\color{cyan}$\Omega\left(\dfrac{n^{\vphantom{1}}}{T^2}\right)$}\\
				\hline
			\end{tabular}
		}
	\end{center}	
	\caption{ Comparison of our  lower and upper bounds to current state-of-the-art results. 
	Our matching bounds establish information theoretically optimal rates for SGD. We note that the $\tilde{\mathcal{O}}(\cdot)$ notation hides logarithmic factors.}
	\label{table:UBandLB}
\end{table*}}

\citet{gurbuzbalaban2015random} established for the first time that for sums of quadratics or smooth functions, there exist parameter regimes under which SGDo achieves an $\mathcal{O}(n^2/T^2)$ rate compared to the $\mathcal{O}(1/T)$ rate of SGD with replacement sampling. In this case, if $n$ is considered a constant, then SGDo becomes $T$ times faster than SGD with replacement.
\citet{NIPS2016_6245} showed that for one epoch, {\it i.e.,} one pass over the $n$ functions, SGDo achieves a convergence rate of $\mathcal{O}(1/T)$.
More recently, \citet{haochen2018random} showed that for functions that are sums of quadratics, or smooth functions under a Hessian smoothness assumption, one could obtain an even faster rate of $\mathcal{O}\left(\frac{1}{T^2}+\frac{n^3}{T^3}\right)$. 
\citet{jain2019sgd} show that for Lipschitz convex functions, SGDo is at least as fast as SGD with replacement, and for functions that are strongly convex and sum of $n$ smooth components one can achieve a rate of $\mathcal{O}\left(\frac{n}{T^2}\right)$.
This latter result was the first convergence rate that provably establishes the superiority of SGD without replacement even for the regime that $n$ is not a constant, as long as the number of iterations $T$ grows faster than the number $n$ of function components.

This new wave of upper bounds has also been followed by new lower bounds.
\citet{safran2019good} establish that there exist sums of quadratics on which SGDo cannot converge faster than  $\Omega\left(\frac{1}{T^2}+\frac{n^2}{T^3}\right)$.
This lower bound gave rise to a gap between achievable rates and information theoretic impossibility.
On one hand, SGDo on $n$ quadratics has a rate of at least $\Omega\left(\frac{1}{T^2}+\frac{n^2}{T^3}\right)$ and at most $\mathcal{O}\left(\frac{1}{T^2}+\frac{n^3}{T^3}\right)$.
On the other hand, for the more general class of strongly convex functions that are sums of smooth functions the best rate is $\mathcal{O}\left(\frac{n}{T^2}\right)$. 
This leaves open the question of whether the upper or lower bounds are loose.
This is precisely the gap we close in this work.

\paragraph{Our Contributions:}
In this work, we establish tight bounds for SGDo. We close the gap between lower and upper bounds on two of the function classes that prior works have focused on: strongly convex functions that are  {\it i)} sums of quadratics and {\it ii)} sums of smooth functions.
Specifically, for {\it i)}, we offer tighter convergence rates, {\it i.e.}, an upper bound that matches the lower bound given by \citet{safran2019good}; as a matter of fact our convergence rates apply to general quadratic functions that are strongly convex, which is a little more general of a function class. For {\it ii)}, we provide a new lower bound that matches the upper bound by \citet{jain2019sgd}. A detailed comparison of current and proposed bounds can be found in Table~\ref{table:UBandLB}.

A few words on the techniques used are in order. 
For our convergence rate on quadratic functions, we heavily rely on and combine the approaches used by \citet{jain2019sgd} and \citet{haochen2018random}.
 The convergence rate analyses proposed by \citet{haochen2018random} can be tightened by a more careful analysis that employs iterate coupling similar to the one used by \citet{jain2019sgd}, combined with new bounds on the deviation of the stochastic, without-replacement gradient from the true gradient of $F$.

For our lower bound, we use a similar construction to the one used by \citet{safran2019good}, with the difference that each of the individual function components is not a quadratic function, but rather a piece-wise quadratic.
This particular function has the property we need: it is smooth, but not quadratic. By appropriately scaling the sharpness of the individual quadratics we construct a function that behaves in a way that SGD without replacement cannot converge faster than a rate of $n/T^2$, no matter what step size one chooses.

We note that although our methods have an optimal dependence on $n$ and $T$, we believe that the dependence on function parameters, {\it e.g.}, strong convexity, Lipschitz, and smoothness, can potentially be improved.

\section{Related Work}
The recent flurry of work on without replacement sampling in stochastic optimization extends to several variants of stochastic algorithms beyond SGD.
In \cite{lee2019random, wright2017analyzing}, the authors provide convergence rates for random cyclic coordinate descent, establishing for the first time that it can provably converge faster than stochastic coordinate descent with replacement sampling. This work is complemented by a lower bound on the gap between the random and non-random permutation variant of coordinate descent \cite{sun2016worst}. Several other works have focused on the random permutation variant of coordinate descent, {\it e.g.}, see 
\cite{gurbu2,sun2019efficiency}.
In 
\cite{gurbu}, novel bounds are given for incremental Newton based methods.
\citet{meng2019convergence} present convergence bounds for with replacement sampling and distributed SGD.
Finally, \citet{ying2018stochastic} present asymptotic bounds for SGDo for strongly convex functions, and show that with a constant step size it approaches the global optimizer to within smaller error radius compared to SGD with replacement. In \cite{NIPS2016_6245}, linear convergence is established for a without replacement variant of SVRG. 

\section{Preliminaries and Notation}\label{sec:prelimAndNotationNew}
We focus on using SGDo to approximately find $x^*$, the global minimizer of the following unconstrained minimization problem
\begin{equation*}
\min_{x\in\mathbb{R}^d} \left(F(x):=\dfrac{1}{n}\sum_{i=1}^n f_i(x)\right).
\end{equation*}

In our convergence bounds, we denote by $T$ the total number of iterations of SGDo, and by $K$ the number of epochs, {\it i.e.}, passes over the data. Hence, 
\begin{equation*}
T=nK.
\end{equation*}
In our derivations, we denote by $x_i^j$ the $i$-th iterate of the $j$-th epoch. 
Consequentially, we have that $x_0^{j+1}\equiv x_n^{j}$. 

Our results in the following sections rely on the following assumptions.
\begin{assumption}\label{ass:convex}(Convexity of Components)
$f_i$ is convex for all $i\in[n]$.
\end{assumption}
\begin{assumption}\label{ass:stronglyConvex}(Strong Convexity)
$F$ is strongly convex with strong convexity parameter $\mu$, that is $$\forall x,y:F(y)\geq F(x)+\langle\nabla F(x),y-x\rangle + \frac{\mu}{2}\|y-x\|^2$$
\end{assumption}
\begin{assumption}\label{ass:boundedD}(Bounded Domain)
$$\forall x: \|x-x^*\|\leq D.$$
\end{assumption}
\begin{assumption}\label{ass:boundedG}(Bounded Gradients)
$$\forall i,x: \|\nabla f_i(x)\|\leq G.$$
\end{assumption}
\begin{assumption}\label{ass:lipschitz}(Lipschitz Gradients) The functions $f_i$ are $L$-smooth, that is
$$\forall i,x,y: \|\nabla f_i(x)-\nabla f_i(y)\|\leq L\|x-y\|.$$
\end{assumption}

\section{Optimal SGDo Rates for Quadratics}

In this section, we will focus on strongly convex functions that are quadratic. 
We will provide a tight convergence rate that improves upon the 
the existing rates 
and matches the $\Omega\left(\frac{1}{T^2}+\frac{n^2}{T^3}\right)$ lower bound by \citet{safran2019good} up to logarithmic factors.

 For strongly convex functions  that are a sum of smooth functions, \citet{jain2019sgd} offer a rate of $\mathcal{O}\left(\frac{n}{T^2}\right)$, whereas for strongly convex quadratics \citet{haochen2018random} give a convergence rate of $\mathcal{O}\left(\frac{1}{T^2}+\frac{n^3}{T^3}\right)$. A closer comparison of these two rates reveals that neither of them can be tight due to the following observation.
 Assume that  $ n \ll  K$. Then, that implies 
 $$\left(\dfrac{1}{T^2} + \dfrac{n^3}{T^3}\right) < \dfrac{n}{T^2}.$$
 At the same time, if we assume that the number of data points is significantly larger than the number of epochs that we run SGDo for, {\it i.e.},  $n \gg  K $ we have that 
$$\left(\dfrac{1}{T^2} + \dfrac{n^3}{T^3}\right) > \dfrac{n}{T^2}.$$

In comparison, the known lower bound for quadratics given by \citet{safran2019good} is $\Omega\left(\frac{1}{T^2}+\frac{n^2}{T^3}\right)$. This makes one wonder what is the true convergence rate of SGDo in this case. We settle the optimal rates for quadratics here by providing an upper bound which, up to logarithmic factors, matches the best known lower bound. 

For the special case of one dimensional quadratics, \citet{safran2019good} proved an upper bound matching the one we prove in this paper. Further, the paper conjectures that the proof can be extended to the generic multidimensional case. 
However, the authors say that the main technical barrier for this extension is that it requires a special case of a matrix-valued arithmetic-geometric mean inequality, which has only been conjectured to be true but not yet proven. 
The authors further conjecture that their proof can be extended to general smooth and strongly convex functions, which turns out to not be true, as we show in Corollary \ref{cor:lowerBoundExtension}. 
On the other hand, we believe that our proof can be extended to the more general family of strongly convex functions, where the Hessian is Lipschitz, similar to the the way \citet{haochen2018random} extend their proof to that case.

In addition to Assumptions 1-5 above, here we also assume the following:
\begin{assumption}\label{ass:quad}
	$F$ is a quadratic function
	\begin{equation*}
	F(x) = \dfrac{1}{2}x^THx + b^Tx + c,
	\end{equation*}
	where $H$ is a positive semi-definite matrix.
\end{assumption}

Note that this assumption is a little more general than the assumption that $F$ is a sum of quadratics.
Also, note that this assumption, in combination with the assumptions on strong convexity and Lipschitz gradients implies bounds on the minimum and maximum eigenvalues of the Hessian of $F$, that is, 
$$\mu I\preccurlyeq H\preccurlyeq L I,$$ 
where $I$ is the identity matrix and $A\preccurlyeq B$ means that $x^T(A-B)x\leq 0$ for all $x$.

\begin{theorem}\label{thm:upperBound}
Under Assumptions \ref{ass:convex}-\ref{ass:quad}, let the step size of SGDo be 
$$\alpha = \dfrac{8\log T}{T\mu}$$ and the number of epochs be $$K \geq 128\dfrac{L^2}{\mu^2}\log T.$$
Then, after $T$ iterations SGDo achieves the following rate
\begin{equation*}
\mathbb{E}[\|x_T-x^*\|^2] = \tilde{\mathcal{O}}\left(\dfrac{1}{T^2}+\dfrac{n^2}{T^3}\right),
\end{equation*}
where $\tilde{\mathcal{O}}(\cdot)$ hides logarithmic factors.
\end{theorem}
The exact upper bound and the full proof of this Theorem are given in Appendix~\ref{app:up}, but we give a proof sketch in the next subsection.

At this point, we would like to remark that the bound on epochs $K\geq 128 \frac{L^2}{\mu^2}\log T$ may be a bit surprising as $K$ and $T$ are dependent. 
However, note that since $T=nK$, we can show that the bound on $K$ above is satisfied if we set the number of epochs to be greater than $C\log n$ for some constant $C$. 
Furthermore, we note that the dependence of $K$ on $\frac{L}{\mu}$ ({\it i.e.}, the condition number of $F$) is most probably not optimal. In particular both \cite{jain2019sgd} and \cite{haochen2018random} have a better dependence on the condition number.

The proof for Theorem ~\ref{thm:upperBound} uses ideas from the works of \citet{haochen2018random} and \citet{jain2019sgd}. 
In particular, one of the central ideas in these two papers is that they aim to quantify the amount of progress made by SGDo over a single epoch. 
Both analyses decompose the progress of the iterates in an epoch as $n$ steps of full gradient descent plus some noise term. 

Similar to \cite{haochen2018random}, we use the fact that the Hessian $H$ of $F$ is constant, which helps us better estimate the value of gradients around the minimizer. 
In contrast to that work, we do not require all individual components $f_i$ to be quadratic, but rather the entire $F$ to be a quadratic function.

An important result proved by \cite{jain2019sgd} is that during an epoch, the iterates do not steer off too far away from the starting point  of the epoch. 
This allows one to obtain a reasonably good bound on the noise term, when one tries to approximate the stochastic gradient with the true gradient of $F$. 
In our analysis, we prove a slightly different version of the same result using an iterate coupling argument similar to the one in \cite{jain2019sgd}. 

The analysis of \cite{jain2019sgd} relies on computing the Wasserstein distance between the unconditional distribution of iterates and the distribution of iterates given a function sampled during an iteration. 
In our analysis, we use the same coupling, but we bypass the Wasserstein framework that \cite{jain2019sgd} suggests and directly obtain a bound on how far the coupled iterates move away from each other during the course of an epoch. 
This results, in our view, to a somewhat simpler and shorter proof.

\subsection{Sketch of proof for Theorem~\ref{thm:upperBound}}
Now we give an overview of the proof. As mentioned before, similar to the previous works, the key idea is to perform a tight analysis of the progress made during an epoch. This is captured by the following Lemma.

\begin{restatable}{lemma}{lemupperBoundEpoch}
\label{lem:upperBoundEpoch}
Let the SGDo step size  be  
$\alpha = \dfrac{4l\log T}{T\mu}$
and the total number of epochs be
$K \geq 128\dfrac{L^2}{\mu^2}\log T$,
where $l\leq 2$.
Then for any epoch,
\begin{align}
\mathbb{E}\left[\|x_0^j-x^*\|^2\right] &\leq \left(1-\frac{n\alpha \mu}{4}\right)\|x_0^{j-1}-x^*\|^2 + 16 n\alpha^3 G^2L^2 \mu^{-1} + 20 n^3\alpha^4 G^2 L^2.\label{eq:lem1}
\end{align}
\end{restatable}

Given the result in Lemma~\ref{lem:upperBoundEpoch}, proving Theorem \ref{thm:upperBound} is a simple exercise. To do so, we simply unroll the recursion \eqref{eq:lem1} for $K$ consecutive epochs. For ease of notation, define $C_1:=16 G^2L^2 \mu^{-1}$ and $C_2:=20  G^2 L^2$. Then,
\begin{align*}
\mathbb{E}\big[\|x^{K}_n  - x^*\|^2\big] &\leq \left(1- \frac{n\alpha \mu}{4}\right)\mathbb{E}\left[\|x^K_0 - x^*\|^2\right] + C_1 n\alpha^3+ C_2n^3 \alpha^{4}&\\
&\leq \left(1- \frac{n\alpha \mu}{4}\right)^2\mathbb{E}\left[\|x^{K-1}_0 - x^*\|^2\right]  + ( C_1 n\alpha^3+ C_2n^3 \alpha^{4})\left(1+\left(1- \frac{n\alpha \mu}{4}\right)\right)&\\
&\quad\quad\quad\quad \vdots&\\
&\leq \left(1- \frac{n\alpha \mu}{4}\right)^{K+1}\mathbb{E}\left[\|x^0_0 - x^*\|^2\right] + ( C_1 n\alpha^3+ C_2n^3 \alpha^{4})\sum_{j=1}^K \left(1- \frac{n\alpha \mu}{4}\right)^{j-1}&\\
&= \left(1- \frac{n\alpha \mu}{4}\right)^{K+1}\|x^0_0 - x^*\|^2 + ( C_1 n\alpha^3+ C_2n^3 \alpha^{4})\sum_{j=1}^K \left(1- \frac{n\alpha \mu}{4}\right)^{j-1}.&
\end{align*}
We can now use the fact that
$(1-x)\leq e^{-x}$ and
$\left(1-\frac{n\alpha\mu}{4}\right)\leq 1$,
to get the following bound:
\begin{align*}
\mathbb{E}\big[\|x^{K}_n  - x^*\|^2\big] &\leq e^{-\frac{n\alpha \mu}{4}K}\|x^0_0 - x^*\|^2 + (C_1 n\alpha^3+ C_2n^3 \alpha^{4})K.
\end{align*}
By setting the step size to be $\alpha=\frac{4l\log T}{T \mu}$ and noting that $T=nK$, we get that
\begin{align*}
\mathbb{E}\big[\|x^{K}_n  &- x^*\|^2\big] \leq e^{-n\frac{4l\log T}{T \mu}\frac{\mu}{4}K}\|x^0_0 - x^*\|^2  + (n\alpha^3C_1+ \alpha^{4}n^{3} C_2)K\\
&= e^{-l\log T}\|x^0_0 - x^*\|^2 + \tilde{\mathcal{O}}\left(\dfrac{1}{T^2}+ \dfrac{n^2}{T^3}\right)\\
&= \frac{\|x^0_0 - x^*\|^2 }{T^l} + \tilde{\mathcal{O}}\left(\dfrac{1}{T^2}+ \dfrac{n^2}{T^3}\right).
\end{align*}
Noting that $\|x^0_0 - x^*\|\leq D$ and choosing $l=2$ gives us the result of Theorem~\ref{thm:upperBound}.

\subsection{With- and without-replacement stochastic gradients are close}
One of the key lemmas in \cite{jain2019sgd} establishes that once SGDo iterates get close enough to the global minimizer $x^*$, 
then any iterate at any time during an epoch $x_i^j$ stays close to the iterate at the beginning of that epoch. 
To be more precise, the lemma we refer to is the following.
\begin{lemma}\label{lem:jain}{\normalfont [\citet[Lemma~5]{jain2019sgd}]} Under the assumptions of Theorem~\ref{thm:upperBound},
\begin{align*}
\mathbb{E}[\|x_i^j-x_0^j\|^2]\leq 5i\alpha^2G^2+2i\alpha (F(x_0^j)-F(x^*)).
\end{align*}
\end{lemma}
We would like to note that Lemma~\ref{lem:jain} is slightly different from the one in \cite{jain2019sgd}, which instead uses $\mathbb{E}[F(x_i^j)-F(x^*)]$ rather than $(F(x_i^j)-F(x^*))$, but their proof can be adapted to obtain the version written above. 
For the formal version of Lemma~\ref{lem:jain}, please see Lemma~\ref{lem:jainLem} in the Appendix.

 Now, consider the case when the iterates are very close to the optimum and hence $F(x_i^j)-F(x^*)\approx 0$. 
 Then, Lemma \ref{lem:jain} implies that $\mathbb{E}[\|x_i^j-x_0^j\|^2]$ does not grow quadratically in $i$ which would generically happen for $i$ gradient steps, but it rather grows linearly in $i$. 
 This is an important and useful fact for SGDo: it shows that all iterates within an epoch remain close to $x^j_0$. 
 
 Hence, since the iterates of SGDo do not move too much during an epoch, then the gradients computed throughout the epoch at points $x_i^j$ should be well approximated by gradients computed on the $x_0^j$ iterate. 
 Roughly, this translates to the following observation:
 the $n$ gradient steps taken through a single epoch are almost equal to $n$ steps of full gradient descent computed at $x_0^j$. 
 This is in essence what allows SGDo to achieve better convergence than SGD - an epoch can be approximated by $n$ steps of gradient descent.
 
Now, let $\sigma^j$ represent the random permutation of the $n$ functions $f_i$ during the $j$-th epoch. Thus, $\sigma^j(i)$ is the index of the function chosen at the $i$-th iteration of the $j$-th epoch. Proving Lemma \ref{lem:jain} requires proving that the function value of $f_{\sigma^j(i)}(x_i^j)$, in expectation, is almost equal to $F(x_i^j)$. In particular, we prove the following claim in our supplemental material.

\begin{restatable}{claim}{claimJainNew}
\label{cl:helper3}{\normalfont [\citet[Lemma~4]{jain2019sgd}]}
If $\alpha \leq \frac{2}{L}$, then for any epoch $j$ and $i$-th ordered iterate during that epoch
\begin{equation}
\bigg|\mathbb{E}\left[ F (x_i^j)-f_{\sigma^j(i)} (x_i^j )\;\middle|\;x_0^j\right]\bigg|\leq 2\alpha G^2.\label{eq:jainDiff}
\end{equation}
\end{restatable}

This claim establishes  that SGDo behaves almost like SGD with replacement, for which the following is true:
$\mathbb{E}[f_{\sigma^j(i)}(x_i^j)]=\mathbb{E}[F(x_i^j)].$
To prove this claim, \cite{jain2019sgd} consider the conditional distribution of iterates, given the current function index, that is $x_i^j|\sigma_i(j)$, and the unconditional distribution of the iterates $x_i^j$. 
Then, they prove that the absolute difference $|\mathbb{E}[F(x_i^j)]-\mathbb{E}[f_{\sigma^j(i)}(x_i^j)]|$ can be upper bounded by the Wasserstein distance between these two distributions. 
To further upper bound the Wasserstein distance, they propose a coupling between the two distributions. To prove our slightly different version of Lemma \ref{lem:jain}, we proved \eqref{eq:jainDiff} without using this Wasserstein framework. Instead, we use the same coupling argument to directly get a bound on \eqref{eq:jainDiff}. Below we explain the coupling and provide a short intuition.

Consider the conditional distribution of $\sigma^j|\sigma^j(i)=s$. 
If we take the distribution of $\sigma|\sigma(i)=1$, we can generate the support of $\sigma^j|\sigma^j(i)=s$ by taking all permutations $\sigma|\sigma(i)=1$ and by swapping $1$ and $s$ among them. 
This is essentially a coupling between these two distributions, proposed in \cite{jain2019sgd}. 
Now, if we use this coupling to convert a permutation in $\sigma|\sigma(i)=1$ to a permutation $\sigma|\sigma(i)=s$, the corresponding $x_i|\sigma(i)=1$ and $x_i|\sigma(i)=s$ would be within a distance of $2\alpha G$. 
This distance bound is Lemma~2 of \cite{jain2019sgd}.

We can now use such distance bound, and let $v_{(1,s)}$ denote a (random) vector whose norm is less than $2\alpha G$. Then,
\begin{align*}
\mathbb{E}\left[ f_{\sigma\left(i\right)} \left(x_i \right)\right]&=\dfrac{1}{n}\sum_{s=1}^n\mathbb{E}\left[ f_{\sigma\left(i\right)} \left(x_i \right)|\sigma(i)=s\right]&\\
&=\dfrac{1}{n}\sum_{s=1}^n\mathbb{E}\left[ f_s \left(x_i \right)|\sigma(i)=s\right]&\\
&=\dfrac{1}{n}\sum_{s=1}^n\mathbb{E}\left[ f_s \left(x_i + v_{(1,s)} \right)|\sigma(i)=1\right]&\\
&\leq \dfrac{1}{n}\sum_{s=1}^n\mathbb{E}\left[ f_s \left(x_i\right)+(2\alpha G^2)|\sigma(i)=1\right]&\\
&=\mathbb{E}\left[ F \left(x_i\right)|\sigma(i)=1\right]+2\alpha G^2.
\end{align*}
Similarly, for any $s\in \{1,\dots, n\}$:
\begin{equation*}
\mathbb{E}\left[ f_{\sigma\left(i\right)} \left(x_i \right)\right]\leq \mathbb{E}\left[ F \left(x_i\right)|\sigma(i)=s\right]+2\alpha G^2.
\end{equation*}
Therefore,
\begin{align*}
\mathbb{E}\left[ f_{\sigma\left(i\right)} \left(x_i \right)\right]&\leq \dfrac{1}{n}\sum_{s=1}^n\mathbb{E}\left[ F \left(x_i\right)|\sigma(i)=s\right]+2\alpha G^2&\\
&\leq \mathbb{E}\left[ F \left(x_i\right)\right]+2\alpha G^2.
\end{align*}
Similarly, we can prove that 
$$\mathbb{E}\left[ f_{\sigma\left(i\right)} \left(x_i \right)\right] \geq \mathbb{E}\left[ F \left(x_i\right)\right]-2\alpha G^2.$$ 
Combining these two results we obtain \eqref{eq:jainDiff}. The detailed proof of Claim~\ref{cl:helper3} is provided in the appendix.

The full proof of Theorem 1 requires some more nuanced bounding derivations, and the complete details can be found in Appendix~\ref{app:up}.
\section{Lower Bound for General Case}
In the previous section, we establish that for quadratic functions the $\Omega\left(\frac{1}{T^2}+\frac{n^2}{T^3}\right)$ lower-bound by \citet{safran2019good} is essentially tight. 
This still leaves open the possibility that a tighter lower bound may exist for strongly convex functions that are not quadratic.
After all, the best convergence rate known for strongly convex functions that are sums of smooth functions is of  the order of $n/T^2$.

Indeed,  in this section, we show that the convergence rate of $\mathcal{O}\left(\frac{n}{T^2}\right)$ established by \citet{jain2019sgd} is tight.

For a certain constant $C$ (see Appendix~\ref{app:low} for the formal version of the theorem), we show the following theorem
\begin{theorem}\label{thm:lowerBound}
There exists a strongly convex function $F$ that is the sum of $n$ smooth convex functions, such that for any step size $$\frac{1}{T}\leq \alpha \leq \frac{C}{n},$$
the error after $T$ total iterations of SGDo satisfies
\begin{equation*}
\mathbb{E}[\|x_T-x^*\|^2]=\Omega\left(\frac{n}{T^2}\right).
\end{equation*}
\end{theorem}
The full proof of this theorem is provided in Appendix~\ref{app:low}, but we give an intuitive explanation of the proof later in this section.

Note that the theorem above establishes the existence of a function for which SGDo converges at rate $\Omega(\frac{n}{T^2})$, but only for the step size range $\frac{1}{T}\leq \alpha \leq \frac{C}{n}$. 
This is the range of the most interest because most of the upper bounds and convergence guarantees of SGDo (and SGD) work in this step size range. 
However, it would still be desirable to get a function on which SGDo converges at rate $\Omega(n/T^2)$ for all step sizes.
Such a function would be difficult to optimize, no matter how much we tune the step size. 
Indeed, we show that based on Theorem \ref{thm:lowerBound}, we can create such a function. 
To do that, we use a function proposed by \citet[Proposition 1]{safran2019good}, which converges slowly outside of the step size range $\frac{1}{T}\leq \alpha \leq \frac{C}{n}$.

\citet{safran2019good} show that there exists a strongly convex function $F_2$, which is the sum of $n$ quadratics, such that for step size $\alpha\leq \frac{1}{T}$, the expected error satisfies  $\mathbb{E}[\|x_T-x^*\|^2]=\Omega(1)$ (see the proof of Proposition~1, pg. 10-12 in their paper). Further, for the same function $F_2$, the proof of that proposition can be adapted directly to get $\mathbb{E}[\|x_T-x^*\|^2]=\Omega\left(\frac{1}{n}\right)$ for any step size $\alpha\geq \frac{C}{n}$, for any constant $C$.

Using this function $F_2$ and the function $F$ from Theorem \ref{thm:lowerBound}, we can create a function on which SGDo converges at rate $\Omega(\frac{n}{T^2})$ for all step sizes.

\begin{corollary}\label{cor:lowerBoundExtension}
There exists a 2-Dimensional strongly convex function that is the sum of $n$ smooth convex functions, such that for any $\alpha>0$
\begin{equation*}
\mathbb{E}[\|x_T-x^*\|^2]=\Omega\left( \frac{n}{T^2}\right).
\end{equation*}
\end{corollary}

\begin{flushleft}
The proof of this corollary is provided in Appendix~\ref{app:cor}.
\end{flushleft}

Thus overall, we get that for any fixed step size $\mathbb{E}[\|x_T-x^*\|^2]=\Omega\left(\frac{n}{T^2}\right)$.
Next, we try to explain the function construction and proof technique behind Theorem \ref{thm:lowerBound}. 
The construction of the lower bound is similar to the one used by \citet{safran2019good}. 
The difference is that the prior work considers quadratic functions, while we consider a slightly modified piece-wise quadratic function. 

Specifically, we construct the following function $F(x)=\frac{1}{n}\sum_{i=1}^{n}f_i(x)$ as
\begin{equation*}
F(x)=
\left\{
\begin{array}{cl}
\dfrac{ x^2}{2}, & \text{ if }x\geq 0\\
\dfrac{ L x^2}{2},& \text{ if }x < 0,
\end{array}
\right.\label{eq:nonLipHessFunctionDefinition}
\end{equation*}
where $n$ is an even number.
Of the $n$ component functions $f_i$, half of them are defined as follows:
\begin{equation*}
\text{if $i \leq \frac{n}{2}$, then }f_i(x)=
\left\{
\begin{array}{cl}
\dfrac{ x^2}{2}+\dfrac{Gx}{2}, &\text{ if }x\geq 0\\
\dfrac{ L x^2}{2}+\dfrac{Gx}{2},& \text{ if }x < 0,
\end{array}
\right.
\end{equation*}
and the other half of the functions are defined as follows:
\begin{equation*}
\text{if $i > \frac{n}{2}$, then }f_i(x)=
\left\{
\begin{array}{cl}
\dfrac{ x^2}{2}-\dfrac{Gx}{2},& \text{ if }x\geq 0\\
\dfrac{ L x^2}{2}-\dfrac{Gx}{2},& \text{ if }x < 0.
\end{array}
\right.
\end{equation*}
For our construction, we set $L$ to be a big enough positive constant. See for example, Fig.~\ref{fig:lowerbound}.
\begin{figure}[h]
	\begin{center}
	\includegraphics[width=0.7\textwidth]{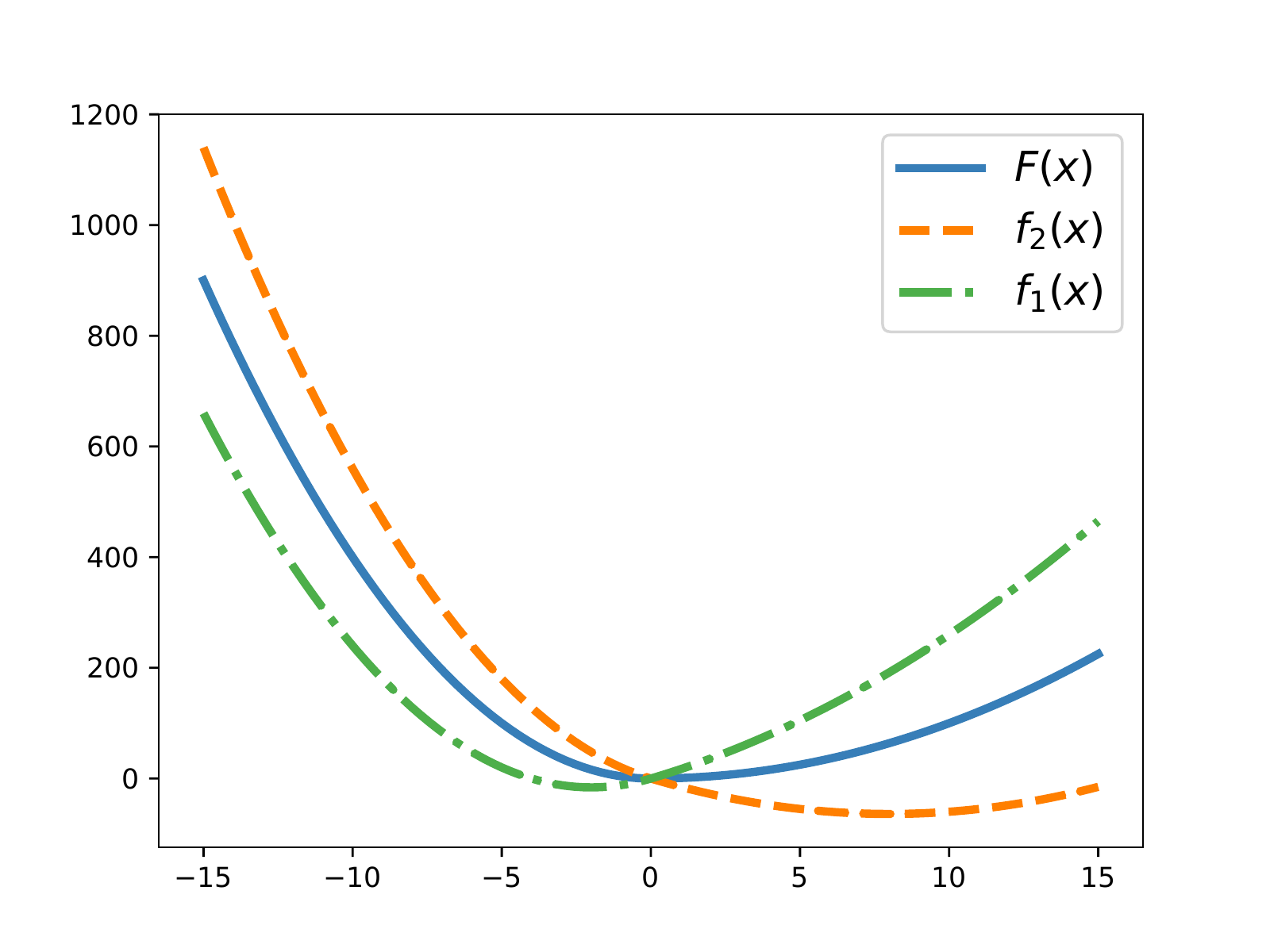}
	\caption{\small Lower bound construction. Note that $f_1(x)$ represents the component functions of the first kind, and $f_2(x)$ represents the component functions of the second kind, and $F(x)$ represents the overall function.}
	\label{fig:lowerbound}
	\end{center}
\end{figure}

 Next we ought to verify that this function abides to Assumptions \ref{ass:convex}-\ref{ass:lipschitz}.
Note that Assumption \ref{ass:convex} is satisfied, as it can be seen that functions $f_i$'s are all continuous and convex.
Next, we need to show that Assumption \ref{ass:stronglyConvex} holds, that is $F$ is strongly convex. 
We will show that this is true by proving the following equivalent definition of strong convexity: a function $f$ is $\mu$-strongly convex if $g(x):=f(x)-\frac{\mu}{2}\|x\|^2$ is convex. We can see that this is true for $F$ with $\mu=1$.

In the proof of Theorem~\ref{thm:lowerBound}, we initialize at the origin.
In that case, in the proof we also prove that Assumptions~\ref{ass:boundedD} and \ref{ass:boundedG} hold.
In particular, we show that the iterates do not go outside of a bounded domain, and inside this domain, the gradient is bounded by $G$.
Finally, let us focus on Assumption~\ref{ass:lipschitz}. To prove that these functions have Lipschitz gradients, we need to show 
$$\forall x,y: |\nabla f_i(x) - \nabla f_i(y)|\leq L  |x-y|.$$ 
If $xy\geq 0$, that is $x$ and $y$ lie on the same side of the origin, then this is simple to see because they both lie on the same quadratic. Otherwise WLOG, assume $x<0$ and $y>0$. 
Also, assume WLOG that $f_i$ is function of the first kind, that is $i\leq \frac{n}{2}$ and hence the linear term in $f_i(x)$ is $\frac{Gx}{2}$.
Then,
\begin{align*}
|\nabla f_i(x) - \nabla f_i(y)|&=\left|L x +\dfrac{G}{2} -  y -\dfrac{G}{2}\right|&\\
&= y-L x &\\
&\leq L y-L x &\\
&\leq L |y-x|.
\end{align*}

Overall, the difficulty in the analysis comes from the fact that unlike the functions considered by \citet{safran2019good}, our functions are piece-wise quadratics. 

Let us initialize at $x_0^1=0$ (the minimizer). We will show that in expectation, at the end of $K$ epochs, the iterate would be at a certain distance (in expectation). 
Note that the progress made over an epoch is just the sum of gradients (multiplied by $-\alpha$) over the epoch:
\begin{equation*}
x^j_n - x^j_0 = -\alpha \sum_{i=1}^{n}\nabla f_{\sigma^j(i)}(x_i)
\end{equation*}
where $\sigma^j(i)$ represents the index of the $i$-th function chosen in the $j$-th epoch. Next, note that the gradients from the linear components $\pm \frac{G}{2}x$ are equal to $\pm \frac{G}{2}$, that is they are constant. 
Thus, they will cancel out over an epoch.

 However the gradients from the quadratic components do not cancel out, and in fact that part of the gradient will not even be unbiased, in the sense that if $x_t\geq 0$, the gradient at $x_t$ from the quadratic component $\frac{ x^2}{2}$ will be less in magnitude than the gradient from the quadratic component $\frac{L x^2}{2}$ at $-x_t$. 
 
 The idea is to now ensure that if an epoch starts off near the minimizer, then the iterates spend a certain amount of time in the $x<0$ region, so that they ``accumulate'' a lot of gradients of the form $L x$, which makes the sum of the gradients at the end of the epoch biased away from the minimizer.
 
To ensure that the iterates spend some time in the $x<0$ region, we analyze the contribution of the linear components during the epoch. This is because when the iterates are already near the minimizer $x\approx 0$, the gradient contribution of the quadratic terms would be small, and the dominating component during an epoch would come from the linear terms. What this means is that in the middle of an epoch, it is the linear terms which contribute the most towards the ``iterate movement'', even though at the end of that epoch their gradients get cancelled out and what remains is the contribution of the quadratic terms.

Then, to obtain a lower bound  matching the upper bound given by \citet{jain2019sgd}, observe that it is indeed this contribution of the linear terms that we require to get a tight bound on. 
This is because, the upper bound from the aforementioned work was also in fact directly dependent on the movement of iterates away from the minimizer during an epoch, caused by the stochasticity in the gradients (cf. Lemma~5 of \citet{jain2019sgd}). 
We give below the informal version of the main lemma for the proof:
\begin{lemma}\label{lem:dummylabelnew}{\normalfont[Informal]}
Let $(\sigma_1,\dots, \sigma_{n})$ be a random permutation of $\{\underbrace{+1,\dots ,+1}_{\frac{n}{2}\text{ times}},\underbrace{-1,\dots ,-1}_{\frac{n}{2}\text{ times}}\}$. Then for $i < n/2$,
\begin{equation*}
\textstyle \mathbb{E}\left[ \left|  \sum_{j=1}^i \sigma_j\right| \right] \geq C \sqrt{i},
\end{equation*}
where $C$ is a universal constant.
\end{lemma}
\begin{flushleft}
Please see Lemma~\ref{lem:expDev1} in Appendix~\ref{app:low} for the formal version of this lemma.
\end{flushleft}

For the purpose of intuition, ignore the contribution of gradients from the quadratic terms. 
Then, the lemma above says that during an epoch, the gradients from the linear terms would move the iterates approximately $ \Omega\left(\alpha \sqrt{n}\frac{G}{2}\right)$ away from the minimizer (after we  multiply by the step size $\alpha $). 

This  implies that in the middle of an epoch, with (almost) probability $1/2$ the iterates would be near $x\approx-\Omega\left(\alpha \sqrt{n}\frac{G}{2}\right)$ and with (almost) probability $1/2$ the iterates would be near $x\approx\Omega\left(\alpha \sqrt{n}\frac{G}{2}\right)$. Hence, over the epoch, the accumulated quadratic gradients multiplied by the step size would look like 
\begin{align*}
\sum_{i=1}^n \mathbb{E}\left[-\alpha (L\mathds{1}_{x_i^j<0}+\mathds{1}_{x_i^j\geq0})x_i^j\right]&\approx-\alpha \sum_{i=1}^n \left( \dfrac{1}{2} L  \Omega\left(-\alpha \sqrt{n}\dfrac{G}{2}\right)+\dfrac{1}{2}   \Omega\left(\alpha \sqrt{n}\dfrac{G}{2}\right)\right)\\
&=\Omega(L \alpha^2n\sqrt{n}).
\end{align*}
If this happens for $K$ epochs, we get that the accumulated error would be $\Omega(L \alpha^2n\sqrt{n}K)=\Omega\left(\frac{1}{\sqrt{n}K}\right)$ for $\alpha \in \left[1/nK,1/n\right]$.
Since $E[|x_T|]\geq 1/\sqrt{n}K$, we know that $E[|x_T-0|^2]\geq 1/nK^2 = n/T^2$. 
Since $0$ is the minimizer of our function in this setting, we have constructed a case where SGDo achieves error 
$$E[|x_T-x^*|^2]\geq n/T^2.$$
This completes the sketch of the proof and the complete proof of Theorem~\ref{thm:lowerBound} is given in Appendix~\ref{app:low}.

\begin{figure}[h]
	\begin{center}
	\includegraphics[width=0.45\textwidth]{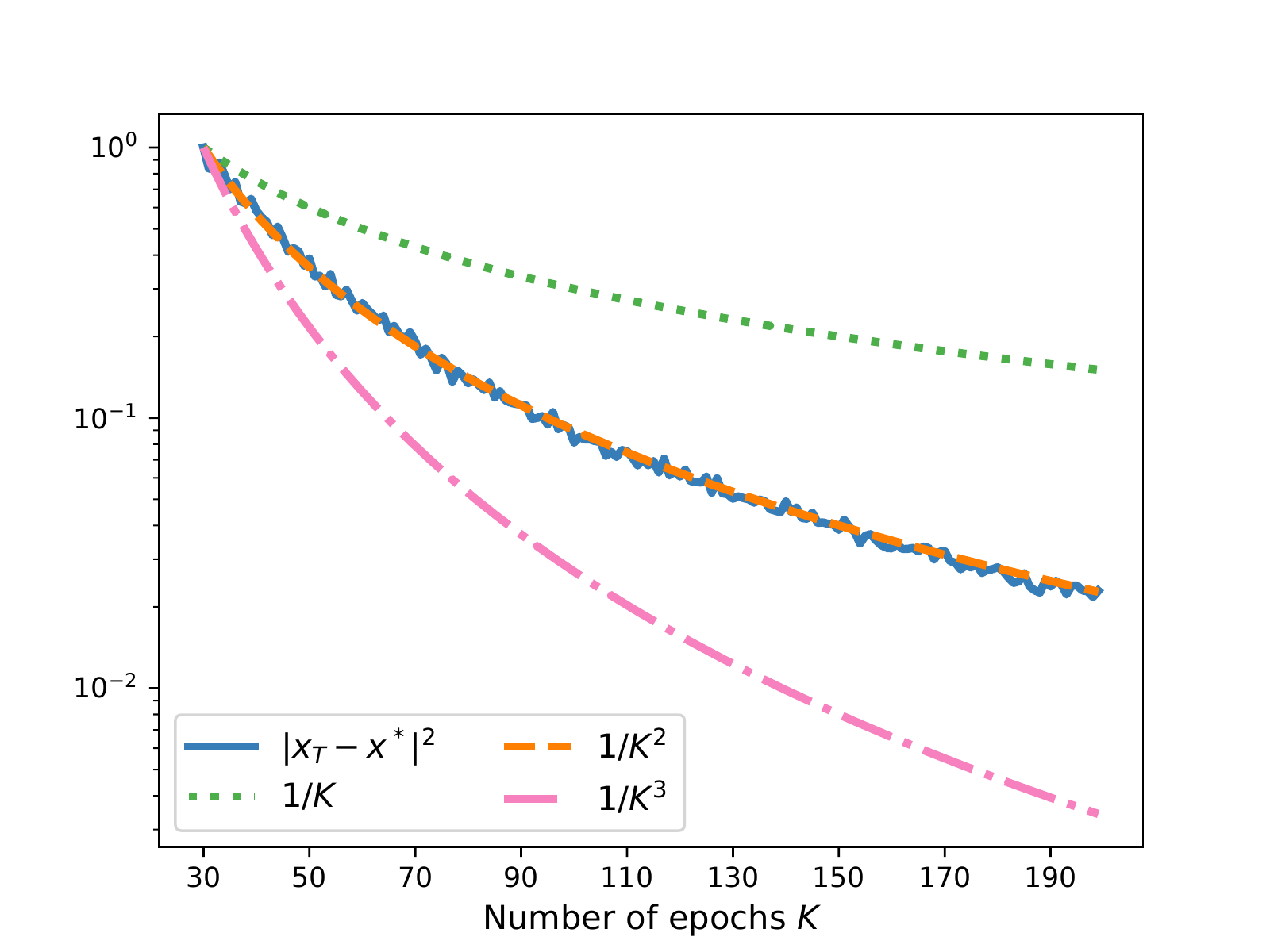}
	\includegraphics[width=0.45\textwidth]{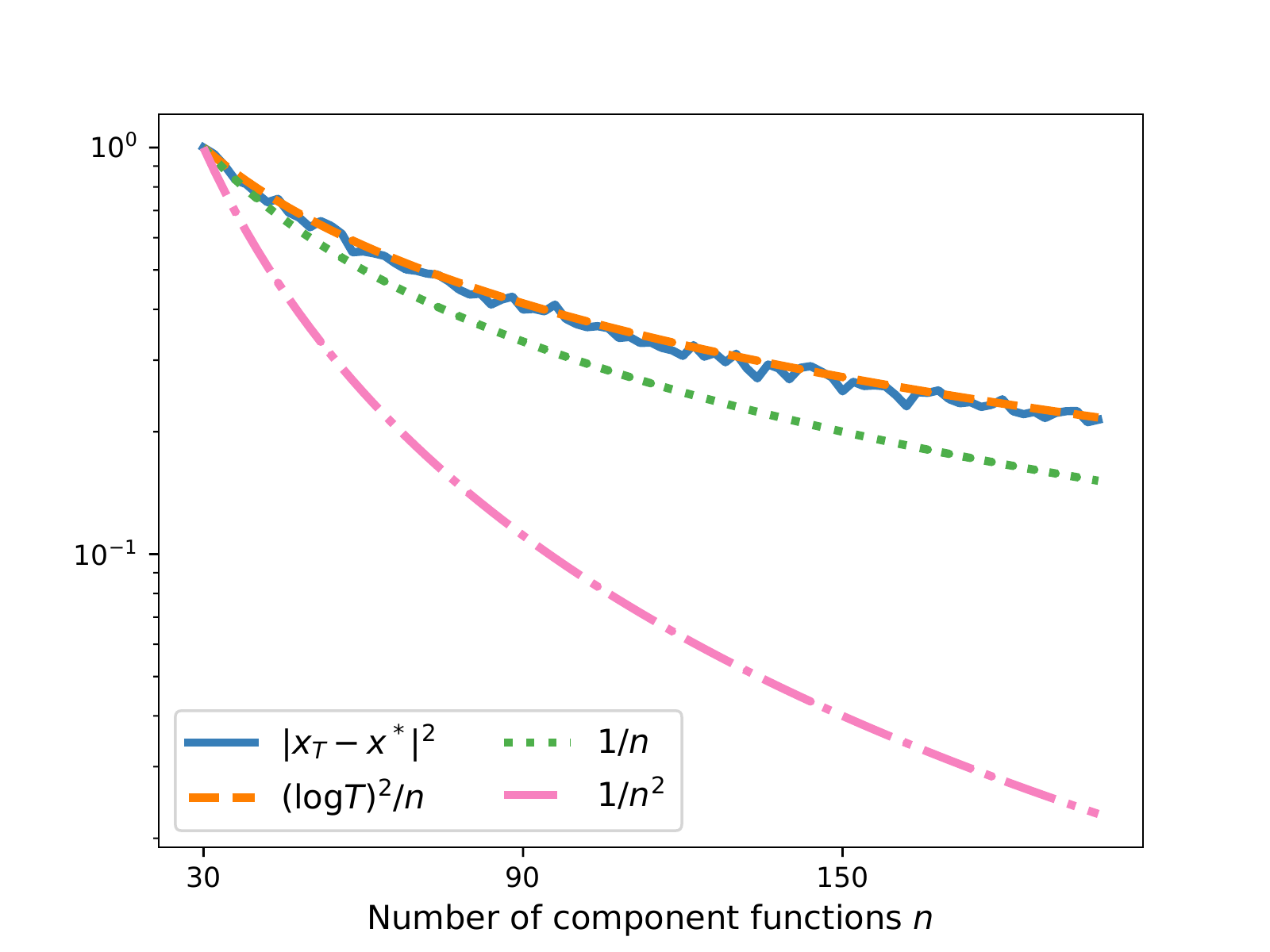}
	\caption{\small Running SGDo on the function $F$ used in our lower bound (Theorem~\ref{thm:lowerBound}) confirms that the rate of convergence of SGDo on this function is indeed $\Omega(\frac{1}{nK^2})=\Omega(\frac{n}{T^2})$. The curves are normalized so that they begin at the same point.}
    \label{fig:numericalVerification}
	\end{center}
\end{figure}

\subsection{Numerical verification}\label{sec:numVerificationnew}
To verify our lower bound of Theorem~\ref{thm:lowerBound}, we ran SGDo on the function described in Eq.~\eqref{eq:nonLipHessFunctionDefinition} with $L=4$. The step size regimes that were considered were $\alpha= \frac{1}{T},\frac{2\log T}{T},\frac{4\log T}{T},\frac{8\log T}{T},$ and $\frac{1}{n}$. The plot for $\alpha=\frac{4\log T}{T}$ is shown in Figure~\ref{fig:numericalVerification}. The plots for the other step size regimes are provided in Appendix~\ref{app:exp}. 

The step size regimes considered cover the range specified in the statement of Theorem~\ref{thm:lowerBound}. Looking at Figure~\ref{fig:numericalVerification} (and the figures in Appendix~\ref{app:exp} for the other step size regimes), the dependence of the convergence rate on $K$ indeed looks exactly like $1/K^2$. However, looking at the figures for the dependence of the convergence rate on $n$, we see that they look like $\frac{(\log T)^2}{n}$. This suggests that the tightest possible lower bound for SGDo with constant step size on strongly convex smooth functions might have a logarithmic term in the numerator. Next, we explain the details of the experiment.

Consider any one of the step size regimes specified above, say $\alpha=\frac{4\log T}{T}$. For this regime, we ran two experiments:
\begin{itemize}
    \item[1.] We fix $n=500$ and vary $K$ from $30$ to $200$, and 
    \item[2.] we fix $K=500$ and vary $n$ from $30$ to $200$.
\end{itemize}
Consider the first experiment, where $n=500$ and $K$ is varied. For each value of $K$, say $K=50$, we set $\alpha = \frac{4\log T}{T}= \frac{4\log (nK)}{nK} = \frac{4\log (500*50)}{500*50}$ and ran SGDo with this constant step size $\alpha$ on the sum of $n=500$ functions for $K=50$ epochs, and the final error was recorded. This was repeated 1000 times to reduce variance. The final mean error after these 1000 runs gave us one point, which we plotted for $K=50$ on the top subfigure of Figure~\ref{fig:numericalVerification}. Repeating the same for all values of $K$ from $30$ to $200$ gave us the top subfigure of Figure~\ref{fig:numericalVerification}. The same procedure was followed for the second experiment where we fix $K$ and vary $n$, and that gave us the bottom subfigure of Figure~\ref{fig:numericalVerification}. The optimization was initialized at the origin, that is $x_0^1=0$. These pairs of experiments were performed for all values of step size regimes in the list $(\textbf{}\frac{1}{T},\frac{2\log T}{T},\frac{4\log T}{T},\frac{8\log T}{T},$ and $\frac{1}{n})$.

Now, we justify the ranges of $n$ and $K$ considered in our experiments. We wanted to verify that the lower bound on the error of SGDo is indeed 
\begin{equation*}
    \Omega\left(\frac{n}{T^2}\right)=\Omega\left(\frac{1}{nK^2}\right) \tag*{[Theorem~\ref{thm:lowerBound}]}
\end{equation*}
 instead of the previously known best lower bound 
 \begin{equation*}
     \Omega\left(\frac{1}{T^2}+\frac{n^2}{T^3}\right)=\Omega\left(\frac{1}{nK^2}\left(\frac{1}{n}+\frac{1}{K}\right)\right).\tag*{[\citet{safran2019good}]}
 \end{equation*}
 
Looking at the RHS of the two equations above, we can see that the dependence of the two lower bounds on $K$ differs only when $n \gg K$ and the dependence on $n$ differs only when $K\gg n$. Thus for example, when we wanted to check dependence on $K$, we set $n=500$ which was bigger than every $K$ in the range $30$ to $200$.

The code for these experiments is available at \url{https://github.com/shashankrajput/SGDo}.

\subsection{Discussion on possible improvements}
Theorem \ref{thm:lowerBound} hints that for faster convergence rates in the epoch based random shuffling SGD, we would not just require smooth and strongly convex functions, but also potentially require that the Hessians of such functions to be Lipschitz. 

We conjecture that Hessian Lipschitzness is sufficient to get the convergence rate of Theorem \ref{thm:upperBound}. 
We think that this is interesting, because the optimal rates for both SGD with replacement and vanilla gradient descent only require strong convexity and gradient smoothness. 
However, here we prove that an optimal rate for SGDo requires the function to be quadratic as well (or at the very least have a Lipschitz Hessian), and SGDo seems to converge slower if the Hessian is not Lipschitz.
\balance
\section{Conclusions and Future Work}
SGD without replacement has long puzzled researchers. From a practical point of view, it always seems to outperform SGD with replacement, and is the algorithm of choice for training modern machine learning models.
From a theoretical point of view, SGDo has resisted tight convergence analysis that establish its performance benefits. 
A recent wave of work established that indeed SGDo can be faster than SGD with replacement sampling, however a gap still remained between the achievable rates and the best known lower bounds. 

In this paper we settle the optimal performance of SGD without replacement for functions that are quadratics, and strongly convex functions that are sums of $n$ smooth functions.
Our results indicate that a  possible  improvement in convergence rates may require a fundamentally different step size rule and significantly different function assumptions.

As future directions, we believe that it would be interesting to establish rates for variants of SGDo that do not  re-permute the functions at every epoch. This is something that is common in practice, where a random permutation is only performed once every few epochs without a significant drop in performance. 
Current theoretical bounds are inadequate to explain this phenomenon, and a new theoretical breakthrough may be required to tackle it. 

We however believe that one of the strongest new theoretical insights introduced by  \cite{jain2019sgd} and used in our analyses can be of significance in a potential  attempt  to  analyze  other variants of SGDo as the one above. This insight is that of iterate coupling. That is the property that SGDo iterates  are only mildly perturbed after swapping only two elements of a permutation. Such a property is   reminiscent to that of algorithmic  stability, and a deeper connection between that and iterate coupling is left as a meaningful intellectual endeavor for future work.
\section*{Acknowledgements}
We would like to thank the ICML reviewers for their constructive feedback in improving the structure of the Appendix. The authors also attribute the motivation for Corollary \ref{cor:lowerBoundExtension} (see the paragraph after Theorem \ref{thm:lowerBound}) to the comments of Reviewer \#3.

This research is supported by an NSF CAREER Award \#1844951, a Sony Faculty Innovation Award,
an AFOSR \& AFRL Center of Excellence Award FA9550-18-1-0166, and an NSF TRIPODS Award
\#1740707.

\bibliography{references}
\bibliographystyle{unsrtnat}

\appendix

\onecolumn

\section{Proof of Theorem \ref{thm:upperBound}}\label{app:up}

\begin{formalTheorem}(Formal version)
Under Assumptions \ref{ass:convex}-\ref{ass:quad}, let the step size of SGDo be 
$$\alpha = \dfrac{4 l \log T}{T\mu}\text{, where } l\leq 2$$ and the number of epochs be $$K \geq 128\dfrac{L^2}{\mu^2}\log T.$$
Then after $T$ iterations SGDo,
\begin{equation*}
\mathbb{E}\left[\|x^{K}_n  - x^*\|^2\right]\leq  \frac{\|x^0_0 - x^*\|^2}{T^l} + \frac{2^{13} G^2L^2 \log^3T }{T^2\mu^4}+ \frac{2^{15} G^2L^2 n^2 \log^4 T }{T^3\mu^4}.
\end{equation*}
\end{formalTheorem}

\begin{proof}
The proof for upper bound uses the framework of \citet{haochen2018random}, combined with some crucial ideas from \citet{jain2019sgd}.

In the block diagram below we connect the pieces needed to establish the proof. All lemmas and proofs follow.

\begin{figure}[H]
	\begin{center}
	\includegraphics[width=\textwidth]{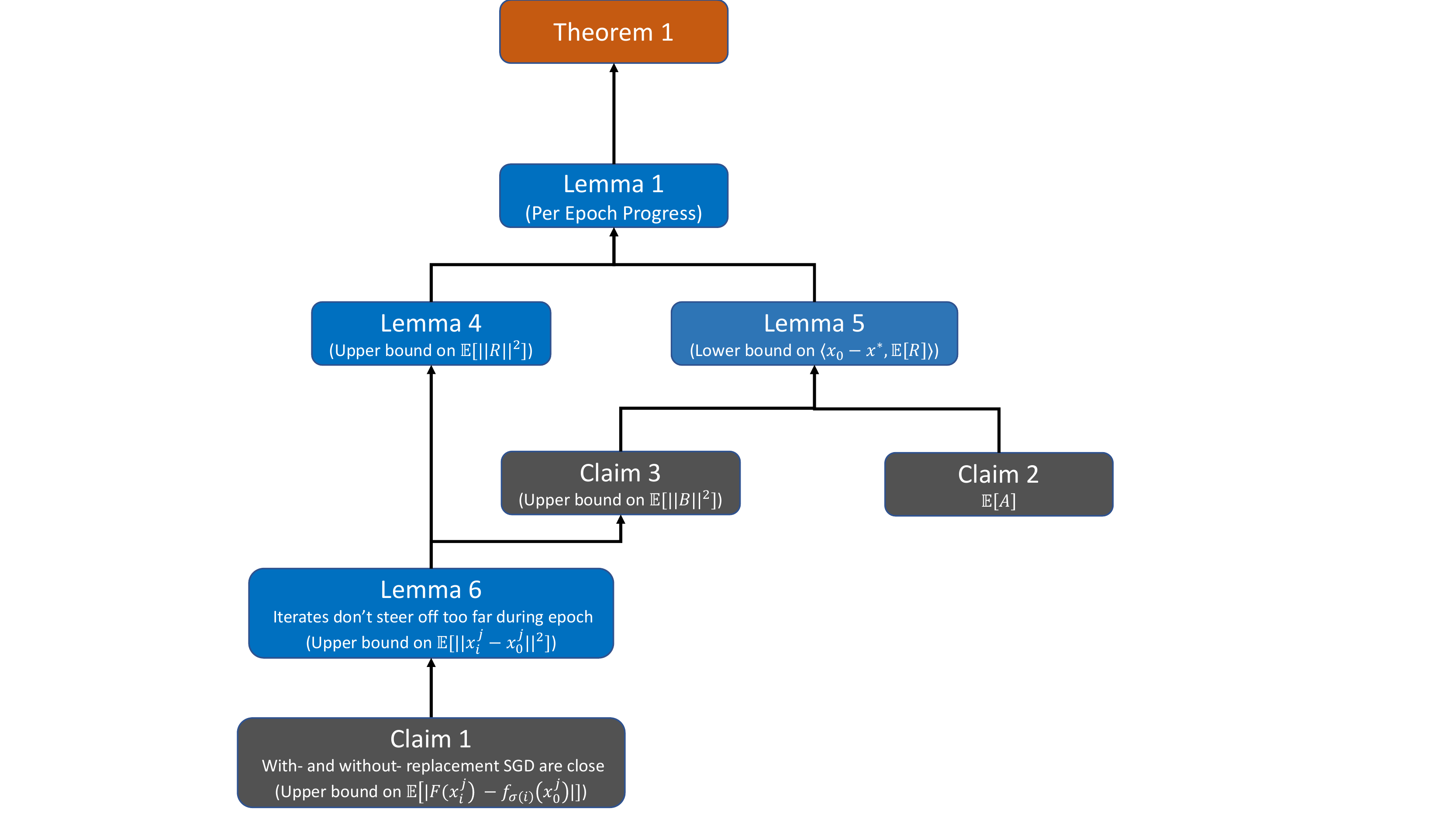}
	\caption{A dependency graph for the proof of Theorem \ref{thm:upperBound}, giving short descriptions of the components required.}
	\end{center}
\end{figure}

The proof strategy is to quantify the progress made during each epoch and then simply unrolling that for $K$ epochs. Towards that end, we have the following lemma
\lemupperBoundEpoch*

As mentioned before, we apply this lemma recursively to all epochs.
For ease of notation, define $C_1:=16 G^2L^2 \mu^{-1}$ and $C_2:=20  G^2 L^2$. 
Then,
\begin{align*}
\mathbb{E}\left[\|x^{K}_n  - x^*\|^2\right] \leq& \left(1- \frac{n\alpha \mu}{4}\right)\mathbb{E}\left[\|x^K_0 - x^*\|^2\right] + C_1 n\alpha^3+ C_2n^3 \alpha^{4}&\\
\leq & \left(1- \frac{n\alpha \mu}{4}\right)^2\mathbb{E}\left[\|x^{K-1}_0 - x^*\|^2\right]  + ( C_1 n\alpha^3+ C_2n^3 \alpha^{4})\left(1+\left(1- \frac{n\alpha \mu}{4}\right)\right)&\\
&\vdots&\\
\leq & \left(1- \frac{n\alpha \mu}{4}\right)^{K+1}\mathbb{E}\left[\|x^0_0 - x^*\|^2\right] + ( C_1 n\alpha^3+ C_2n^3 \alpha^{4})\sum_{j=1}^K \left(1- \frac{n\alpha \mu}{4}\right)^{j-1}&\\
= & \left(1- \frac{n\alpha \mu}{4}\right)^{K+1}\|x^0_0 - x^*\|^2 + ( C_1 n\alpha^3+ C_2n^3 \alpha^{4})\sum_{j=1}^K \left(1- \frac{n\alpha \mu}{4}\right)^{j-1}.&
\end{align*}

We can now use the fact that
$(1-x)\leq e^{-x}$ and
$\left(1-\frac{n\alpha\mu}{4}\right)\leq 1$,
to get the following bound:
\begin{align*}
&\mathbb{E}\left[\|x^{K}_n  - x^*\|^2\right] \leq e^{-\frac{n\alpha \mu}{4}K}\|x^0_0 - x^*\|^2+ (C_1 n\alpha^3+ C_2n^3 \alpha^{4})K.
\end{align*}

By setting the stepsize to be $\alpha=\frac{8\log T}{T \mu}$ and noting that $T=nK$, we get that
\begin{align*}
\mathbb{E}\left[\|x^{K}_n  - x^*\|^2\right] &\leq e^{-n\frac{4l\log T}{T \mu}\frac{\mu}{4}K}\|x^0_0 - x^*\|^2 + (n\alpha^3C_1+ \alpha^{4}n^{3} C_2)K&\\
&= e^{-l\log T}\|x^0_0 - x^*\|^2 + \frac{2^9 C_1 \log^3T }{T^2\mu^3}+ \frac{2^{12}C_2 n^2 \log^4 T }{T^3\mu^4}&.
\end{align*}

Substituting the values of $C_1$ and $C_2$ back into the inequality gives
\begin{equation*}
\mathbb{E}\left[\|x^{K}_n  - x^*\|^2\right]\leq  \frac{\|x^0_0 - x^*\|^2}{T^l} + \frac{2^{13} G^2L^2 \log^3T }{T^2\mu^4}+ \frac{2^{15} G^2L^2 n^2 \log^4 T }{T^3\mu^4}.
\end{equation*}

\end{proof}
\subsection{Proof of Lemma \ref{lem:upperBoundEpoch}}
\begin{proof}
Throughout this proof we will be working inside an epoch, so we skip using the super script $j$ in $x_i^j$ which denotes the $j$-th epoch. Thus in this proof, $x_0$ refers to the iterate at beginning of that epoch. Let $\sigma$ denote the permutation of $[n]$ used in this epoch. Therefore at the $i$-th iteration of this epoch, we take a descent step using the gradient of $f_{\sigma(i)}$.

Next we define the error term (same as the error term defined in \cite{haochen2018random})
\begin{equation}
R := \sum_{i=1}^n \left( \nabla f_{\sigma(i)}(x_{i-1})-  \nabla F(x_0)\right)  = \sum_{i=1}^n\left(\nabla f_{\sigma(i)}(x_{i-1}) - \nabla f_{\sigma(i)}(x_0)\right).
\end{equation}

Then, using the iterative relation for gradient descent, we get
\begin{align}
\|x_n-x^*\|^2 =& \left\|\left(x_0- \sum_{i=1}^n\alpha\nabla f_{\sigma(i)} (x_{i-1})\right)-x^*\right\|^2 \nonumber\\
=& \|x_0-x^*\|^2 - 2\alpha \left\langle x_0-x^*, \sum_{i=1}^n \nabla f_{\sigma(i)} (x_{i-1}) \right \rangle + \alpha^2 \left\|\sum_{i=1}^n \nabla f_{\sigma(i)} (x_{i-1})\right\|^2&\nonumber\\
=& \|x_0-x^*\|^2 - 2n\alpha  \left\langle x_0-x^*, \nabla F (x_0) \right \rangle - 2\alpha \left\langle x_0-x^*, R \right \rangle + \alpha^2 \|n \nabla F (x_0) + R\|^2&\nonumber\\
\stackrel{(a)}{\leq}& \|x_0-x^*\|^2 - 2n\alpha  \left\langle x_0-x^*, \nabla F (x_0) \right \rangle + 2\alpha^2 n^2\left\| \nabla F (x_0)\|^2 - 2\alpha \left\langle x_0-x^*, R \right \rangle + 2\alpha^2 \|R\right\|^2&\nonumber\\
\stackrel{(b)}{\leq}& \|x_0-x^*\|^2 - 2n\alpha  \left( \frac{\mu}{2}\|x_0-x^*\|^2 + \frac{1}{2L}\|\nabla F(x_0)\|^2\right)+ 2\alpha^2 n^2\| \nabla F (x_0)\|^2 + 2\alpha^2 \|R\|^2 &\nonumber\\
&- 2\alpha \left\langle x_0-x^*, R \right \rangle&\nonumber\\
=& \left(1-n\alpha \mu \right)\|x_0-x^*\|^2 - n\alpha  \left(\frac{1}{L}- 2n\alpha \right)\|\nabla F(x_0)\|^2  + 2\alpha^2 \|R\|^2- 2\alpha \left\langle x_0-x^*, R \right \rangle.\label{eq:expanded}
\end{align}
The inequality $(a)$ above comes from the fact that $\|a+b\|^2\leq 2\|a\|^2 + 2\|b\|^2$. For inequality $(b)$, we used the property of strong convexity given by Theorem \ref{thm:nesterov} below:
\begin{theorem}\label{thm:nesterov}[\citet[Theorem~2.1.11]{nesterov2013introductory}]
For an $L$-smooth and $\mu$-strongly convex function $F$,
\begin{equation*}
\langle \nabla F(x) - \nabla F(y), x - y\rangle \geq  \frac{\mu }{2}\|x - y\|^2 + \frac{1}{2 L}\|\nabla F(x) - \nabla F(y)\|^2
\end{equation*}
NOTE: We have slightly modified the original theorem by using the fact that $L\geq \mu$.
\end{theorem}

Thus, to upper bound the expected value of $\|x_n-x^*\|^2$, Ineq.~\eqref{eq:expanded} says that we need to control the expected magnitude of the error term, which is $\mathbb{E}[2\alpha^2\|R\|^2]$; and its expected alignment with $(x_0-x^*)$, which is $\mathbb{E}[-2\alpha \left\langle x_0-x^*, R \right \rangle]$. To achieve that we introduce  the following two lemmas.
\begin{lemma}\label{lem:normR} 
If $\alpha \leq 1/2nL$, then the magnitude of the error is bounded above, in expectation:
\begin{equation*}
\mathbb{E}[\|R\|^2]\leq L^2n^2\|x_0-x^*\|^2+5n^3\alpha^2 L^2G^2.
\end{equation*}
\end{lemma}
\begin{lemma}\label{lem:innerProdR} If $\alpha \leq 1/2nL$, then the error term's alignment with $x_0-x^*$ is bounded below, in expectation:
\begin{equation*}
\left\langle x_0 - x^*, \mathbb{E}[R] \right\rangle \geq - \frac{\alpha n^2 \|\nabla F\left(x_0\right)\|^2 }{2} - \left(\frac{n\mu}{4}+\frac{2\alpha^2n^3 L^4}{\mu}\right)\|x_0-x^*\|^2  - \frac{10\alpha^4L^4G^2n^4}{\mu} - \frac{8n\alpha^2 G^2L^2}{ \mu}.
\end{equation*}
\end{lemma}
For Lemma \ref{lem:normR} and Lemma \ref{lem:innerProdR}, we will show later (see Ineq.~\eqref{ineq:alphaNew}) that our choice of parameters ensure $\alpha \leq 1/2nL$.

Substituting the inequalities from these two lemmas into Ineq.~\eqref{eq:expanded}, we get that
{ 
\begin{align}
\mathbb{E}[\|x_n&-x^*\|^2] \nonumber\\
\leq& \left(1-n\alpha\mu \right)\|x_0-x^*\|^2 - n\alpha  \left(\frac{1}{2L}- 2n\alpha \right)\|\nabla F(x_0)\|^2  + 2\alpha^2 (L^2n^2\|x_0-x^*\|^2+5n^3\alpha^2 L^2G^2)&\nonumber\\
&- 2\alpha \left(- \frac{\alpha n^2 \|\nabla F\left(x_0\right)\|^2 }{2} - \left(\frac{n\mu}{4}+\frac{2\alpha^2n^3 L^4}{\mu}\right)\|x_0-x^*\|^2  - \frac{10\alpha^4L^4G^2n^4}{\mu} - \frac{8n\alpha^2 G^2L^2}{ \mu}\right)&\nonumber\\
=& \left(1-n\alpha\mu+2\alpha^2 L^2n^2 +\frac{\alpha\mu n}{2} + \frac{ 4n^3L^4\alpha^3 }{\mu}\right)\|x_0-x^*\|^2   &\nonumber\\
&- n\alpha  \left(\frac{1}{L}- 3n\alpha\right)\|\nabla F(x_0)\|^2+ 10 n^3\alpha^4 L^2G^2 +\frac{20  \alpha^5L^4G^2n^4}{\mu} + \frac{16 n\alpha^3 G^2L^2}{ \mu}\nonumber\\
=& \left(1-\frac{n\alpha\mu}{2}+2\alpha^2 L^2n^2  +4 \mu^{-1} n^3L^4\alpha^3 \right)\|x_0-x^*\|^2   &\nonumber\\
&- n\alpha  \left(\frac{1}{L}- 3n\alpha\right)\|\nabla F(x_0)\|^2+ 10 n^3\alpha^4 L^2G^2 +20  \mu^{-1}\alpha^5L^4G^2n^4 + 16 n\alpha^3 G^2L^2 \mu^{-1}.\label{eq:intermediate}
\end{align}}
We will prove the following inequalities shortly
\begin{align*}
 \frac{n\alpha\mu}{4}-2\alpha^2 L^2n^2 -4 \mu^{-1} n^3L^4\alpha^3 &\stackrel{(c)}{\geq}  0\\
\frac{1}{L}- 3\alpha n &\stackrel{(d)}{\geq}  0 \\
10 n^3\alpha^4 L^2G^2  &\stackrel{(e)}{\geq}  20  \mu^{-1}\alpha^5L^4G^2n^4.%
\end{align*}
Finally, using %
$(c)$, $(d)$ and $(e)$ in Ineq.~\eqref{eq:intermediate}, we get
\begin{equation*}
\mathbb{E}[\|x_n-x^*\|^2] \leq \left(1-\frac{n\alpha \mu}{4}\right)\|x_0-x^*\|^2 + 20 n^3\alpha^4 L^2G^2 + 16 n\alpha^3 G^2L^2 \mu^{-1}.
\end{equation*}
This completes the proof. The only thing left is to prove the inequalities $(c)$, $(d)$ and $(e)$, which we'll do next.

$(c)$ and $(d)$:
It can be shown that $2\alpha^2 L^2n^2 \geq 4 \mu^{-1} n^3L^4\alpha^3$ (See $(e)$ below). So to prove $(c)$, it is sufficient to show that
\begin{equation*}
\frac{n\alpha \mu}{4}\geq 4\alpha^2 L^2n^2.
\end{equation*}
Recall that $\alpha= \frac{4l\log T}{T\mu}\leq \frac{8\log T}{T\mu}$,  $T=nK$, and $K \geq 128\frac{L^2}{\mu^2}\log T$ . Then,
\begin{align}
\alpha &\leq \frac{8 \log T}{T\mu}&\nonumber\\
&= \frac{8 \log T}{nK\mu}&\nonumber\\
&\leq \frac{\mu^2 \log T}{16 n L^2 \mu \log T}&\nonumber\\
&= \frac{\mu }{16 n  L^2 }.\label{ineq:alphaNew2}&
\end{align}
This is equivalent to $\frac{n\alpha \mu}{4}\geq 4\alpha^2 L^2n^2$. Thus, we have proven $(c)$. To prove $(d)$, we continue on the the series of inequalities:
\begin{align}
\alpha &\leq \frac{\mu }{16 n  L^2 }\nonumber\\
&\leq \frac{ 1}{16 n L }&\nonumber\\
&\leq \frac{ 1}{3 n L }.\label{ineq:alphaNew}&
\end{align}
This proves $(d)$.

$(e)$:
\begin{align*}
\alpha&\leq  \frac{\mu}{16 L^2  n}&\tag*{[Using Ineq.~\eqref{ineq:alphaNew2}]}\\
&\leq \frac{\mu}{2 L^2  n}&\\
\implies 10 n^3\alpha^4 L^2G^2 &\geq 20  \mu^{-1}\alpha^5L^4G^2n^4.
\end{align*}
\end{proof}

\subsection{Proof of Lemma \ref{lem:normR}}\label{sec:proofLem7New}
\begin{proof}
Throughout this proof we will be working inside an epoch, so we skip using the super script $j$ in $x_i^j$ which denotes the $j$-th epoch. Thus in this proof, $x_0$ refers to the iterate at beginning of that epoch. Let $\sigma$ denote the permutation of $[n]$ used in this epoch. Therefore at the $i$-th iteration of this epoch, we take a descent step using the gradient of $f_{\sigma(i)}$.

\begin{align*}
\mathbb{E}\left[\|R\|^2\right]&=\mathbb{E}\left[\left\|\sum_{i=1}^n\left(\nabla f_{\sigma(i)}(x_{i-1}) - \nabla f_{\sigma(i)}(x_0)\right)\right\|^2\right]&\\
&\leq \mathbb{E}\left[\left(\sum_{i=1}^n\left\|\nabla f_{\sigma(i)}(x_{i-1}) - \nabla f_{\sigma(i)}(x_0)\right\|\right)^2\right].&\tag*{[Triangle inequality]}
\end{align*}

Now, if we can bound $\|x_{i-1} - x_0\|$, then we can also bound $\|\nabla f_{\sigma(i)}(x_{i-1}) - \nabla f_{\sigma(i)}(x_0)\|$ using gradient Lipschitzness. To bound $\|x_{i-1} - x_0\|$, we use Lemma~\ref{lem:jainLem} below which says that if an epoch starts close to the minimizer $x^*$, then during that epoch, the iterates do now wander off too far away from the beginning of that epoch.

\begin{lemma}\label{lem:jainLem}{\normalfont [\citet[Lemma~5]{jain2019sgd}, but proved slightly differently]}
If $\alpha \leq 2/L$ then,
\begin{equation*}
\mathbb{E}[\|x_i^j-x_0^j\|]^2\leq \mathbb{E}[\|x_i^j-x_0^j\|^2]\leq 5 i \alpha^2 G^2 +2i\alpha (F(x_0^j)-F(x^*))\leq 5i\alpha^2G^2 + i\alpha L \|x_0^j-x^*\|^2.
\end{equation*}
\end{lemma}

Thus, continuing the sequence of inequalities,
\begin{align}
\mathbb{E}\Bigg[\bigg(\sum_{i=1}^n\|\nabla f_{\sigma(i)}&(x_{i-1}) - \nabla f_{\sigma(i)}(x_0)\|\bigg)^2\Bigg]\leq L^2\mathbb{E}\left[\left(\sum_{i=1}^n\left\|x_{i-1} - x_0\right\|\right)^2\right]&\tag*{[Gradient Lipschitzness]}\nonumber\\
&= L^2\mathbb{E}\left[\sum_{i=1}^n\sum_{j=1}^n\left\|x_{i-1} - x_0\right\|\left\|x_{j-1} - x_0\right\|\right]&\nonumber\\
&= L^2\sum_{i=1}^n\sum_{j=1}^n\mathbb{E}\left[\left\|x_{i-1} - x_0\right\|\left\|x_{j-1} - x_0\right\|\right]&\nonumber\\
&\leq L^2\sum_{i=1}^n\sum_{j=1}^n\sqrt{\mathbb{E}\left[\left\|x_{i-1} - x_0\right\|^2\right]}\sqrt{\mathbb{E}\left[\left\|x_{j-1} - x_0\right\|^2\right]}&\tag*{[Cauchy-Scwartz inequality]}\nonumber\\
&\leq L^2n^2(2n\alpha L\|x_0-x^*\|^2 + 5n\alpha^2G^2)&\tag*{[Using Lemma~\ref{lem:jainLem}]}\nonumber\\
&\leq L^2n^2\|x_0-x^*\|^2 + 5n^3\alpha^2G^2L^2,\label{ineq:lemNormRnew238}%
\end{align}
where we used the assumption that $\alpha\leq 1/2nL$ in the last step.

\end{proof}

\subsection{Proof of Lemma \ref{lem:innerProdR}}
\begin{proof}
Throughout this proof we will be working inside an epoch, so we skip using the super script $j$ in $x_i^j$ which denotes the $j$-th epoch. Thus in this proof, $x_0$ refers to the iterate at beginning of that epoch. Let $\sigma$ denote the permutation of $[n]$ used in this epoch. Therefore at the $i$-th iteration of this epoch, we take a descent step using the gradient of $f_{\sigma(i)}$.
\begin{align}
R &= \sum_{i=1}^n \left[\nabla f_{\sigma(i)} (x_{i-1}) - \nabla f_{\sigma\left(i\right)}(x_{0})\right] \nonumber\\
&= \sum_{i=1}^n \left[\nabla f_{\sigma\left(i\right)} \left(x_0 -\alpha \sum_{j=1}^{i-1}\nabla f_{\sigma\left(j\right)} \left(x_{j-1}\right)\right) - \nabla f_{\sigma\left(i\right)}\left(x_{0}\right)\right] \tag*{[By the definition of SGDo iterations: $x_{i+1}=x_i - \alpha \nabla f_{\sigma(i)}(x_i)$]}\nonumber\\
&= \sum_{i=1}^n \left[\nabla f_{\sigma(i)} \left(x_0 -\alpha \sum_{j=1}^{i-1}\nabla f_{\sigma(j)} (x_{j-1})\right) - \nabla f_{\sigma(i)}(x_{0})\right.& \nonumber\\
&\quad + \left. \nabla f_{\sigma(i)} \left(x_0 -\alpha \sum_{j=1}^{i-1}\nabla f_{\sigma(j)} (x_{0})\right) - \nabla f_{\sigma(i)} \left(x_0 -\alpha \sum_{j=1}^{i-1}\nabla f_{\sigma(j)} (x_{0})\right)\right] \nonumber\\
&= \sum_{i=1}^n \left[    \nabla f_{\sigma\left(i\right)} \left(x_0 -\alpha \sum_{j=1}^{i-1}\nabla f_{\sigma\left(j\right)} \left(x_{0}\right)\right)         - \nabla f_{\sigma\left(i\right)}\left(x_{0}\right)\right]& \nonumber\\
&\quad  + \sum_{i=1}^n \left[\nabla f_{\sigma\left(i\right)} \left(x_0 -\alpha \sum_{j=1}^{i-1}\nabla f_{\sigma\left(j\right)} \left(x_{j-1}\right)\right)- \nabla f_{\sigma\left(i\right)} \left(x_0 -\alpha \sum_{j=1}^{i-1}\nabla f_{\sigma\left(j\right)} \left(x_{0}\right)\right)\right] \nonumber\\
&= A + B, \label{a18}
\end{align}
where 
\begin{align*}
A &:= \sum_{i=1}^n \left[    \nabla f_{\sigma\left(i\right)} \left(x_0 -\alpha \sum_{j=1}^{i-1}\nabla f_{\sigma\left(j\right)} \left(x_{0}\right)\right)         - \nabla f_{\sigma\left(i\right)}\left(x_{0}\right)\right], \\
\text{and}&\\
B &:= \sum_{i=1}^n \left[\nabla f_{\sigma\left(i\right)} \left(x_0 -\alpha \sum_{j=1}^{i-1}\nabla f_{\sigma\left(j\right)} \left(x_{j-1}\right)\right)- \nabla f_{\sigma\left(i\right)} \left(x_0 -\alpha \sum_{j=1}^{i-1}\nabla f_{\sigma\left(j\right)} \left(x_{0}\right)\right)\right].
\end{align*}

$A$ and $B$ are the same terms as the ones defined in \citet{haochen2018random}. The difference in our analysis and \citet{haochen2018random} is that we get tighter bounds on these terms.

In the following, we use $u$ to denote a random vector with norm less than or equal to 1. Also, assume that $H$ is the Hessian of the quadratic function $F$.

\begin{claim}\label{cl:A}
\begin{equation*}
\mathbb{E}\left[A\right]= (2n\alpha GL)u - \alpha \frac{n(n-1)}{2} H\nabla F\left(x_{0}\right).%
\end{equation*}
\end{claim}

\begin{claim}\label{cl:B}
If $\alpha \leq 1/2nL$, then
\begin{equation*}
\|\mathbb{E}[B]\|^2\leq n^4L^4\alpha^2\|{x_{0} - x^*}\|^2+5n^5 L^4\alpha^4 G^2.
\end{equation*}
\end{claim}

Using Eq.~\eqref{a18} and Claim \ref{cl:A}:
\begin{align}
 \left\langle x_0 - x^*, \mathbb{E}\left[R\right] \right\rangle &=  \left\langle x_0 - x^*, \mathbb{E}\left[A\right] + \mathbb{E}\left[B\right] \right\rangle\nonumber\\
&= \left\langle x_0 - x^*, \mathbb{E}\left[A\right] \right\rangle + \left\langle x_0 - x^*, \mathbb{E}\left[B\right] \right\rangle\nonumber\\
&= - \alpha \frac{n(n-1)}{2}  \left\langle x_0 - x^*, H \nabla F\left(x_0\right) \right\rangle + \langle (2n\alpha GL)u, x_0-x^*\rangle + \left\langle x_0 - x^*, \mathbb{E}\left[B\right] \right\rangle\nonumber\\
&= - \alpha \frac{n(n-1)}{2}  \|\nabla F(x_0)\|^2 + \langle (2n\alpha GL)u, x_0-x^*\rangle + \left\langle x_0 - x^*, \mathbb{E}\left[B\right] \right\rangle.
\label{a6}
\end{align}

For the middle term in Eq.~\eqref{a6}, we use Cauchy-Schwarz inequality and the AM-GM inequality as follows
\begin{align}
\langle (2n\alpha GL)u, x_0-x^*\rangle&\geq -\| (2n\alpha GL)u\|\| x_0-x^*\|\nonumber\\
&\geq -\left[\frac{\lambda}{2}\|x_0- x^*\|^2 + \frac{1}{2\lambda}(2n\alpha GL)^2\right], \label{a8a}
\end{align}
where $\lambda$ is a positive number.

We bound the last term in Eq.~\eqref{a6} similarly,
\begin{align}
\left\langle x_0 - x^*, \mathbb{E}\left[B\right]\right\rangle &\geq-\| x_0 - x^*\|\| \mathbb{E}\left[B\right]\| \nonumber\\
&\geq -\left[\frac{\lambda}{2}\|x_0- x^*\|^2 + \frac{1}{2\lambda}\|\mathbb{E}\left[B\right]\|^2\right]. \label{a8b}
\end{align}

Setting $\lambda = \frac{1}{4}\mu n$ and continuing on from Eq.~\eqref{a6}, we get %
\begin{align*}
\langle x_0 - x^*&, \mathbb{E}[R] \rangle = - \alpha \frac{n(n-1)}{2}  \|\nabla F(x_0)\|^2 + \langle (2n\alpha GL)u, x_0-x^*\rangle + \left\langle x_0 - x^*, \mathbb{E}\left[B\right] \right\rangle \\
&\geq - \frac{1}{2} \alpha n^2 \|\nabla F\left(x_0\right)\|^2 -\left[\frac{\lambda}{2}\|x_0- x^*\|^2 + \frac{1}{2\lambda}(2n\alpha GL)^2\right] -\left[\frac{\lambda}{2}\|x_0- x^*\|^2 + \frac{1}{2\lambda}\|\mathbb{E}\left[B\right]\|^2\right]\tag*{[Using Ineq.~\eqref{a8a} and \eqref{a8b}]}\\
&= - \frac{1}{2} \alpha n^2 \|\nabla F\left(x_0\right)\|^2 -\frac{1}{4}\mu n\|x_0- x^*\|^2 - \frac{2}{\mu n}(2n\alpha GL)^2 - \frac{2}{\mu n}\|\mathbb{E}\left[B\right]\|^2\tag*{[Using the value of $\lambda = \frac{1}{4}\mu n$]}\\
&\geq - \frac{\alpha n^2 \|\nabla F\left(x_0\right)\|^2}{2}  - \left(\frac{n \mu}{4}+2\alpha^2n^3\mu^{-1} L^4\right)\|x_0-x^*\|^2 - 10\mu^{-1}\alpha^4L^4G^2n^4 - 8 n\alpha^2 G^2L^2 \mu^{-1}.\tag*{[Using Claim~\ref{cl:B}]}
\end{align*}
\end{proof}

\subsubsection{Proof of Claim~\ref{cl:A}}
\begin{proof}
The proof strategy of this claim is to use the fact that the Hessian of a quadratic function is constant, and thus we can use it calculate exactly the gradient difference of a quadratic function.
In this proof, we will use $u$ to denote a random variable with norm at most 1.
\begin{align*}
\mathbb{E}\left[A\right] &=\mathbb{E}\left[\sum_{i=1}^n \left[    \nabla f_{\sigma\left(i\right)} \left(x_0 -\alpha \sum_{j=1}^{i-1}\nabla f_{\sigma_t\left(j\right)} \left(x_{0}\right)\right)         - \nabla f_{\sigma\left(i\right)}\left(x_{0}\right)\right]\right]&\\
&=\sum_{i=1}^n \mathbb{E}\left[    \nabla f_{\sigma\left(i\right)} \left(x_0 -\alpha \sum_{j=1}^{i-1}\nabla f_{\sigma\left(j\right)} \left(x_{0}\right)\right)         - \nabla f_{\sigma_t\left(i\right)}\left(x_{0}\right)\right]&\\
&\stackrel{(f)}{=}\sum_{i=1}^n \mathbb{E}\left[    \nabla F \left(x_0 -\alpha \sum_{j=1}^{i-1}\nabla f_{\sigma\left(j\right)} \left(x_{0}\right)\right)         + (2\alpha GL)u - \nabla F\left(x_{0}\right)\right]&\\
&= (2n\alpha GL)u + \sum_{i=1}^n \mathbb{E}\left[    \nabla F \left(x_0 -\alpha \sum_{j=1}^{i-1}\nabla f_{\sigma_t\left(j\right)} \left(x_{0}\right)\right)          - \nabla F\left(x_{0}\right)\right]&\\
&= (2n\alpha GL)u - \alpha \sum_{i=1}^n \mathbb{E}\left[    H \left( \sum_{j=1}^{i-1}\nabla f_{\sigma_t\left(j\right)} \left(x_{0}\right)\right)          \right]&\tag*{[For quadratics, $\nabla F(x)-\nabla F(y) = H(x-y)$]}\\
&= (2n\alpha GL)u - \alpha H   \sum_{i=1}^n  \left( \sum_{j=1}^{i-1}\mathbb{E}\left[ \nabla f_{\sigma_t\left(j\right)} \left(x_{0}\right)\right]\right)          &\\
&= (2n\alpha GL)u - \alpha \frac{n(n-1)}{2} H\nabla F\left(x_{0}\right).&
\end{align*}
We used Claim~\ref{cl:helper2} for $(f)$. This concludes the proof of Claim \ref{cl:A}, and the proof of Claim \ref{cl:helper2} is provided next.
\begin{claim}\label{cl:helper2}
\begin{equation*}
\mathbb{E}\left[\nabla f_{\sigma\left(i\right)} \left(x_0 -\alpha \sum_{j=1}^{i-1}\nabla f_{\sigma\left(j\right)} \left(x_{0}\right)\right)\right]=\mathbb{E}\left[\nabla F \left(x_0 -\alpha \sum_{j=1}^{i-1}\nabla f_{\sigma\left(j\right)} \left(x_{0}\right)\right)\right]+(2\alpha G L)u
\end{equation*}
\end{claim}
\end{proof}

\paragraph*{Proof of Claim \ref{cl:helper2}}
The proof for this claim uses the iterate coupling technique.
\begin{proof}
\begin{align*}
\mathbb{E}\left[\nabla f_{\sigma\left(i\right)} \left(x_0 -\alpha \sum_{j=1}^{i-1}\nabla f_{\sigma\left(j\right)} \left(x_{0}\right)\right)\right]&=\frac{1}{n}\sum_{s=1}^n\mathbb{E}\left[\nabla f_{\sigma\left(i\right)} \left(x_0 -\alpha \sum_{j=1}^{i-1}\nabla f_{\sigma\left(j\right)} \left(x_{0}\right)\right)\middle|\sigma(i)=s\right]&\tag*{[Since $\forall i,j: \mathbb{P}(\sigma(i)=j)=1/n$]}\\
&=\frac{1}{n}\sum_{s=1}^n\mathbb{E}\left[\nabla f_{s} \left(x_0 -\alpha \sum_{j=1}^{i-1}\nabla f_{\sigma(j)} \left(x_{0}\right)\right)\middle|\sigma(i)=s\right].
\end{align*}
Note that the distribution of $(\sigma|\sigma(i)=s)$ can be created from the distribution of $(\sigma|\sigma(i)=1)$ by taking all permutations from $\sigma|\sigma(i)=1$ and swapping 1 and $s$ in the permutations (this is essentially a coupling between the two distributions, same as the one in \cite{jain2019sgd}). 
This means that when we convert a permutation from $(\sigma|\sigma(i)=1)$ to $(\sigma|\sigma(i)=s)$ in this manner, the sum $\sum_{j=1}^{i-1}\nabla f_{\sigma(j)} \left(x_{0}\right)$ would have a change of at most one component $f_{\sigma(j)}$ before and after the swap, and furthermore the component will have a norm of at most $ G$. Thus because of the swap (adding a component and removing one), the norm changes by at most $2 G$. In the following, we use $u_{p}$, $v_{(p,q)}$, $w_{(p,q)}$ and $u$ to denote random vectors with norms at most 1. Hence, the sum $(\sum_{j=1}^{i-1}\nabla f_{\sigma(j)} \left(x_{0}\right)|\sigma(i)=s)$ is equal to $(2 G v_{(1,s)}+ \sum_{j=1}^{i-1}\nabla f_{\sigma(j)} \left(x_{0}\right) |\sigma(i)=1)$. Then, continuing on the sequence of equalities,
\begin{align*}
\mathbb{E}\left[\nabla f_{\sigma\left(i\right)} \left(x_0 -\alpha \sum_{j=1}^{i-1}\nabla f_{\sigma\left(j\right)} \left(x_{0}\right)\right)\right]&=\frac{1}{n}\sum_{s=1}^n\mathbb{E}\left[\nabla f_s \left(x_0 -\alpha \sum_{j=1}^{i-1}\nabla f_{\sigma(j)} \left(x_{0}\right)\right)|\sigma(i)=s\right]&\\
&=\frac{1}{n}\sum_{s=1}^n\mathbb{E}\left[\nabla f_s \left(x_0 -\alpha \sum_{j=1}^{i-1}\nabla f_{\sigma(j)} (x_{0}) + 2\alpha Gv_{(1,s)}\right)|\sigma(i)=1\right]&\\
&=\frac{1}{n}\sum_{s=1}^n\mathbb{E}\left[\nabla f_s \left(x_0 -\alpha \sum_{j=1}^{i-1}\nabla f_{\sigma(j)} \left(x_{0} \right)\right)+ (2\alpha GL)w_{(1,s)}|\sigma(i)=1\right]&\tag*{[Using gradient Lipschitzness]}\\
&=\mathbb{E}\left[\nabla F \left(x_0 -\alpha \sum_{j=1}^{i-1}\nabla f_{\sigma(j)} \left(x_{0} \right)\right)|\sigma(i)=1\right]+ (2\alpha GL)u_1.
\end{align*}
Similarly, for any $s$:
\begin{equation*}
\mathbb{E}\left[\nabla f_{\sigma\left(i\right)} \left(x_0 -\alpha \sum_{j=1}^{i-1}\nabla f_{\sigma\left(j\right)} \left(x_{0}\right)\right)\right]=\mathbb{E}\left[\nabla F \left(x_0 -\alpha \sum_{j=1}^{i-1}\nabla f_{\sigma(j)} \left(x_{0} \right)\right)|\sigma(i)=s\right]+ (2\alpha GL)u_s.
\end{equation*}
Hence,
\begin{align*}
\mathbb{E}\left[\nabla f_{\sigma\left(i\right)} \left(x_0 -\alpha \sum_{j=1}^{i-1}\nabla f_{\sigma\left(j\right)} \left(x_{0}\right)\right)\right]&=\frac{1}{n}\sum_{s=1}^n\left(\mathbb{E}\left[\nabla F \left(x_0 -\alpha \sum_{j=1}^{i-1}\nabla f_{\sigma(j)} \left(x_{0} \right)\right)|\sigma(i)=s\right]+ (2\alpha GL)u_s\right)&\\
&=\mathbb{E}\left[\nabla F \left(x_0 -\alpha \sum_{j=1}^{i-1}\nabla f_{\sigma(j)} \left(x_{0} \right)\right)\right]+ (2\alpha GL)u.
\end{align*}
\end{proof}

\subsubsection{Proof of Claim \ref{cl:B}}
\begin{proof}
\begin{align*}
\|\mathbb{E}[B]\|^2 &\leq (\mathbb{E}\left[\|{B}\|\right])^2 &\tag*{[Jensen's inequality]}\\
&=\left(\mathbb{E}\left[\left\|\sum_{i=1}^n \left[\nabla f_{\sigma\left(i\right)} \left(x_0 -\alpha \sum_{j=1}^{i-1}\nabla f_{\sigma\left(j\right)} \left(x_{j-1}\right)\right)- \nabla f_{\sigma\left(i\right)} \left(x_0 -\alpha \sum_{j=1}^{i-1}\nabla f_{\sigma\left(j\right)} \left(x_{0}\right)\right)\right]\right\|\right]\right)^2&\\
 &\leq \left(\mathbb{E}\left[\sum_{i=1}^n \left\|\nabla f_{\sigma\left(i\right)} \left(x_0 -\alpha \sum_{j=1}^{i-1}\nabla f_{\sigma\left(j\right)} \left(x_{j-1}\right)\right)- \nabla f_{\sigma\left(i\right)} \left(x_0 -\alpha \sum_{j=1}^{i-1}\nabla f_{\sigma\left(j\right)} \left(x_{0}\right)\right)\right\|\right]\right)^2\tag*{[Triangle inequality]}\\
 &\leq \left(\mathbb{E}\left[ \sum_{i=1}^n L\alpha \left\| \sum_{j=1}^{i-1} {\nabla f_{\sigma\left(j\right)} \left(x_{j-1}\right) - \nabla f_{\sigma\left(j\right)} \left(x_0\right)}\right\|\right]\right)^2\tag*{[Gradient Lipschtizness]}\\
 &\leq L^2\alpha^2\left(\mathbb{E}\left[ \sum_{i=1}^n  \sum_{j=1}^{i-1} \|{\nabla f_{\sigma\left(j\right)} \left(x_{j-1}\right) - \nabla f_{\sigma\left(j\right)} \left(x_0\right)}\|\right]\right)^2\tag*{[Triangle inquality]}\\
  &\leq L^2\alpha^2\left(\mathbb{E}\left[ \sum_{i=1}^n  \sum_{j=1}^{n} \|{\nabla f_{\sigma\left(j\right)} \left(x_{j-1}\right) - \nabla f_{\sigma\left(j\right)} \left(x_0\right)}\|\right]\right)^2\\
    &= n^2L^2\alpha^2\left(\mathbb{E}\left[  \sum_{j=1}^{n} \|\nabla f_{\sigma\left(j\right)} \left(x_{j-1}\right) - \nabla f_{\sigma\left(j\right)} \left(x_0\right)\|\right]\right)^2\tag*{[The inner sum is independent of $i$]}\\
    &\leq n^2L^2\alpha^2\mathbb{E}\left[ \left( \sum_{j=1}^{n} \|\nabla f_{\sigma\left(j\right)} \left(x_{j-1}\right) - \nabla f_{\sigma\left(j\right)} \left(x_0\right)\|\right)^2\right].\tag*{[Jensen's inequality]}
        \end{align*}

Now, we have already proved an upper bound on $\mathbb{E}\left[ \left( \sum_{j=1}^{n} \|\nabla f_{\sigma\left(j\right)} \left(x_{j-1}\right) - \nabla f_{\sigma\left(j\right)} \left(x_0\right)\|\right)^2\right]$ in the proof of Lemma \ref{lem:normR} in Subsection \ref{sec:proofLem7New} under exactly the same assumptions as the ones for this claim (See Ineq.~\eqref{ineq:lemNormRnew238} in that proof). In particular, there we have proved that 
  \begin{equation*}
  \mathbb{E}\left[ \left( \sum_{j=1}^{n} \|\nabla f_{\sigma\left(j\right)} \left(x_{j-1}\right) - \nabla f_{\sigma\left(j\right)} \left(x_0\right)\|\right)^2\right] \leq L^2n^2\|x_0-x^*\|^2 + 5n^3\alpha^2G^2L^2.
  \end{equation*}

  Substituting this inequality in the set of inequalities above gives us the result.

\end{proof}

\subsection{Proof of Lemma \ref{lem:jainLem}}
We will only prove $\mathbb{E}[\|x_i^j-x_0^j\|^2]\leq 5 i \alpha^2 +2i\alpha (F(x_0^j)-F(x^*))$, the rest follows from Jensen's inequality and gradient Lipschitzness.

We will use the following claim, which is just \citet[Lemma~4]{jain2019sgd} proved in a slightly different way: we skip the Wasserstein framework but use the same coupling.

\claimJainNew*

The rest of the proof is identical to the proof in \cite{jain2019sgd}. Because in this proof we work inside an epoch, so we skip the super script in the notation. 
\begin{align*}
\|x_{i+1}-x_0\|^2 &= \|x_{i}-x_0\|^2 -2\alpha\langle \nabla f_{\sigma(i)}(x_i), x_i-x_0 \rangle + \alpha^2\|\nabla f_{\sigma(i)}(x_i)\|^2&\\
&\leq \|x_{i}-x_0\|^2 -2\alpha\langle \nabla f_{\sigma(i)}(x_i), x_i-x_0 \rangle + \alpha^2G^2&\tag*{[Bounded gradients]}\\
&\leq \|x_{i}-x_0\|^2 +2\alpha(f_{\sigma(i)}(x_0)-f_{\sigma(i)}(x_j)) + \alpha^2G^2&\tag*{[Convexity of $f_{\sigma(i)}$]}
\end{align*}
Taking expectation both sides:
\begin{align*}
\mathbb{E}[\|x_{i+1}-x_0\|^2|x_0] &\leq \mathbb{E}[\|x_{i}-x_0\|^2|x_0] +2\alpha\mathbb{E}[f_{\sigma(i)}(x_0)-f_{\sigma(i)}(x_j) |x_0]+ \alpha^2G^2&\\
  &= \mathbb{E}[\|x_{i}-x_0\|^2|x_0] +2\alpha F(x_0) +2\alpha\mathbb{E}[-f_{\sigma(i)}(x_j) |x_0]+ \alpha^2G^2&\\
  &= \mathbb{E}[\|x_{i}-x_0\|^2|x_0] +2\alpha F(x_0) +2\alpha\mathbb{E}[F(x_j)-f_{\sigma(i)}(x_j) - F(x_j) |x_0]+ \alpha^2G^2&\\
  &\leq \mathbb{E}[\|x_{i}-x_0\|^2|x_0] +2\alpha F(x_0) +2\alpha\mathbb{E}[F(x_j)-f_{\sigma(i)}(x_j) - F(x^*) |x_0]+ \alpha^2G^2&\tag*{[Since $x^*$ is the minimizer of $F$]}\\
  &= \mathbb{E}[\|x_{i}-x_0\|^2|x_0] +2\alpha (F(x_0)-F(x^*)) +2\alpha\mathbb{E}[F(x_j)-f_{\sigma(i)}(x_j) |x_0]+ \alpha^2G^2&\\
  &\leq \mathbb{E}[\|x_{i}-x_0\|^2|x_0] +2\alpha (F(x_0)-F(x^*)) +2\alpha(2\alpha G^2)+ \alpha^2G^2&\tag*{[Using Claim \ref{cl:helper3}]}\\
   &= \mathbb{E}[\|x_{i}-x_0\|^2|x_0] +2\alpha (F(x_0)-F(x^*)) + 5\alpha^2G^2.
\end{align*}
Unrolling this for $i$ iterations gives us the required result.

\subsubsection{Proof of Claim~\ref{cl:helper3}}
The proof for this claim also uses the iterate coupling technique, similar to the proof of Claim \ref{cl:helper2}.

\begin{proof}
As written in the claim statement, we assume that the start of the epoch, $x_0^j$ is given.
In this proof, we work inside an epoch, so we skip the superscript $j$ in $x_i^j$. 
\begin{align*}
\mathbb{E}\left[ f_{\sigma\left(i\right)} \left(x_i \right)\right]&=\frac{1}{n}\sum_{s=1}^n\mathbb{E}\left[ f_{\sigma\left(i\right)} \left(x_i \right)|\sigma(i)=s\right]&\\
&=\frac{1}{n}\sum_{s=1}^n\mathbb{E}\left[ f_s \left(x_i \right)|\sigma(i)=s\right].
\end{align*}
Note that the distribution of $(\sigma|\sigma(i)=s)$ can be created from the distribution of $(\sigma|\sigma(i)=1)$ by taking all permutations from $(\sigma|\sigma(i)=1)$ and swapping 1 and $s$ in the permutations (this is essentially a coupling between the two distributions, same as the one in \cite{jain2019sgd}). This means that when we convert a permutation from the distribution $(\sigma|\sigma(i)=1)$ to a permutation from the distribution $(\sigma|\sigma(i)=s)$ in this manner, the corresponding $(x_i|\sigma(i)=1)$ and $(x_i|\sigma(i)=s)$ would be within a distance of $2\alpha G$. Here is why this is true: let $x'$ be an iterate reached using a permutation $\sigma'$ from the distribution $(\sigma|\sigma(i)=1)$. Now, create $\sigma''$ by swapping 1 and $s$ in $\sigma'$. Then $\sigma''(i)=s$ and hence it lies in the distribution $(\sigma|\sigma(i)=s)$. Let $x''$ be an iterate reached using $\sigma''$. Then can use Lemma 2 from \cite{jain2019sgd}, adapted to our setting:
\begin{lemma}\label{lem2jain2019new}
[\citet[Lemma~2]{jain2019sgd}]
Let $\alpha \leq 2/L$. Then almost surely, $\forall i \in [n]$,
\begin{equation*}
\|x'-x''\|\leq 2 G \alpha.
\end{equation*}

\end{lemma}

 In the following, we use $v_{(p,q)}$ to denote a random vector with norm less than or equal to 1; and $w_{(p,q)}$, $u_p$ and $u$ to denote a random scalar with absolute value less than or equal to 1.

Then using Lemma \ref{lem2jain2019new}, $(x_i|\sigma(i)=s)$ is equal to $(x_i+v_{(1,s)}|\sigma(i)=1)$ (Similar to what we did in the proof of Claim \ref{cl:helper2}.).
\begin{align*}
\mathbb{E}\left[ f_{\sigma\left(i\right)} \left(x_i \right)\right]&=\frac{1}{n}\sum_{s=1}^n\mathbb{E}\left[ f_s \left(x_i \right)|\sigma(i)=s\right]\\
&=\frac{1}{n}\sum_{s=1}^n\mathbb{E}\left[ f_s \left(x_i + (2\alpha G)v_{(1,s)}\right)|\sigma(i)=1\right]\\
&=\frac{1}{n}\sum_{s=1}^n\mathbb{E}\left[ f_s \left(x_i\right)+(2\alpha G^2)w_{(1,s)}|\sigma(i)=1\right]\\
&=\mathbb{E}\left[ F \left(x_i\right)|\sigma(i)=1\right]+(2\alpha G^2)u_1.
\end{align*}

Similarly, for any $s$:
\begin{equation*}
\mathbb{E}\left[ f_{\sigma\left(i\right)} \left(x_i \right)\right]=\mathbb{E}\left[ F \left(x_i\right)|\sigma(i)=s\right]+(2\alpha G^2)u_s.
\end{equation*}
Hence,
\begin{align*}
\mathbb{E}\left[ f_{\sigma\left(i\right)} \left(x_i \right)\right]&=\frac{1}{n}\sum_{s=1}^n\left(\mathbb{E}\left[ F \left(x_i\right)|\sigma(i)=s\right]+(2\alpha G^2)u_s\right)&\\
&=\mathbb{E}\left[ F \left(x_i\right)\right]+(2\alpha G^2)u.
\end{align*}
All the calculations in this proof assumed that the initial point of the epoch, $x_0^j$ is known. Thus, the equation above implies
\begin{equation*}
\left|\mathbb{E}\left[ F (x_i^j)-f_{\sigma(i)} (x_i^j )\middle|x_0^j\right]\right|\leq 2\alpha G^2.
\end{equation*}
\end{proof}

\newpage
\section{Proof of Theorem~\ref{thm:lowerBound}}\label{app:low}

\begin{formalTheorem}(Formal version)
There exists an initialization point $x_0^1$ and a $1$-strongly convex function $F$ that is the mean of $n$ smooth convex functions which have $L$-Lipschitz gradients ($L\geq 2^{17}$), such that if 
\begin{equation*}
\frac{1}{nK}\leq \alpha \leq \frac{2^{-14}}{nL}\text{, }K\geq 2^{14}L,
\end{equation*}
$n\geq 256$ and $n$ is a multiple of $4$, then,
\begin{equation*}
\mathbb{E}[\|x_T-x^*\|^2]\geq \frac{2^{-56} G^2 n}{T^2}. 
\end{equation*}
\end{formalTheorem}
\textbf{Remarks: } 
\begin{enumerate}
\item Because $\mu=1$, the condition number is just $L$. Note that the lower bound provided above is independent of $L$.
\item The theorem and proof have not been optimized with respect to the dependence on universal constants and it can probably be much better. In particular, the experiments in Subsection \ref{sec:numVerificationnew} use the same construction of functions with much better values of constants.
\end{enumerate}

\begin{proof}

As with Theorem \ref{thm:upperBound}, we first start with a block diagram of the components required to establish the lower bound. 

\begin{figure}[H]
	\begin{center}
	\includegraphics[width=\textwidth]{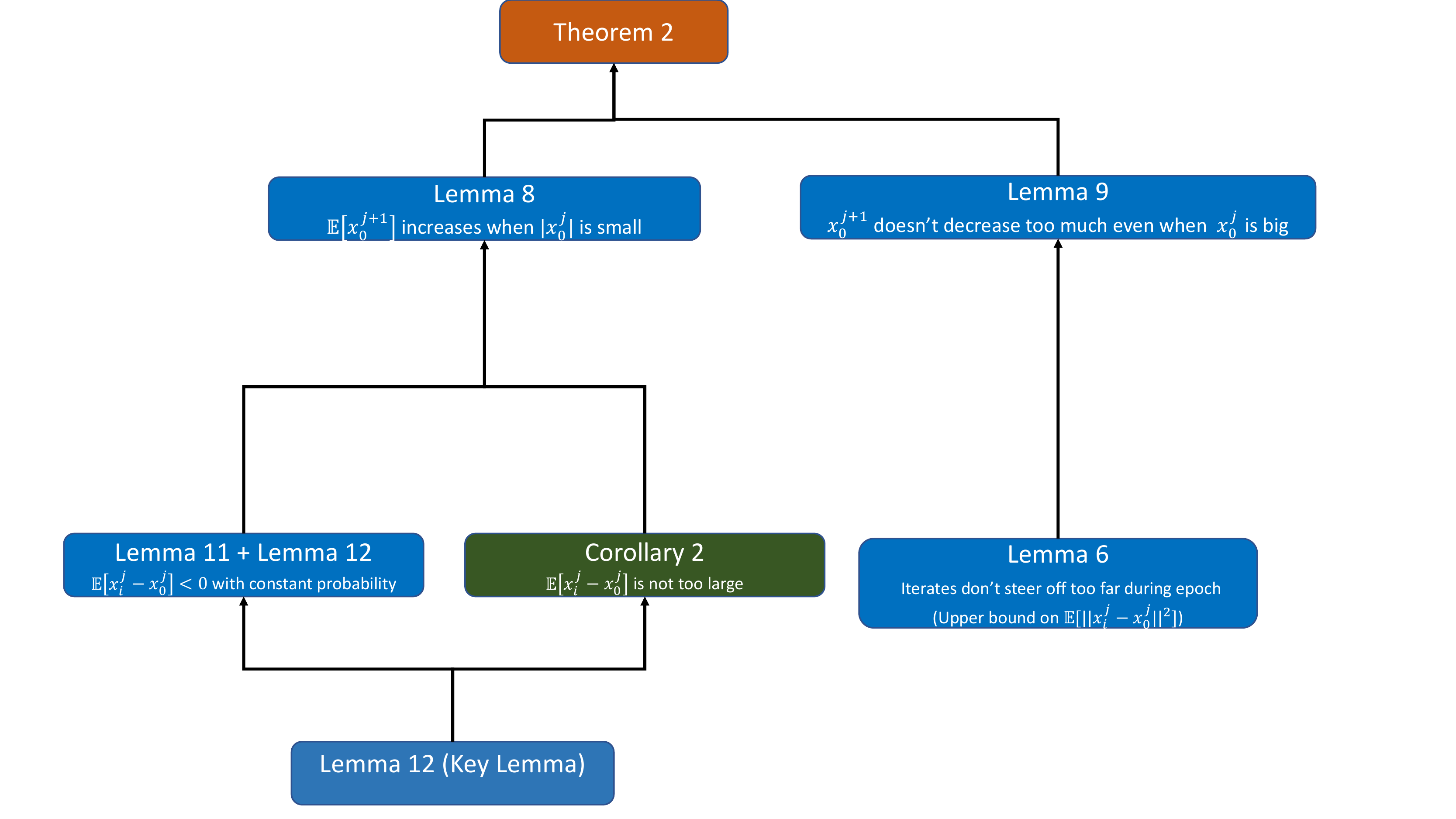}
	\caption{A dependency graph for the proof of Theorem \ref{thm:lowerBound}, giving short descriptions of the components for the proof of Theorem \ref{thm:lowerBound}.}
	\end{center}
\end{figure}

The function $F(x)=\frac{1}{n}\sum_{i=1}^{n}f_i(x)$ that we construct is of the form
\begin{align*}
F(x)=\begin{cases}
\frac{ x^2}{2}& \text{if }x\geq 0\\
\frac{ L x^2}{2}& \text{if }x < 0.
\end{cases}
\end{align*}
where $n\geq 8$ and it is a multiple of 4. 
Of the $n$ component functions $f_i$, half of them are defined as follows (we call these as functions of first kind):
\begin{equation*}
\text{if $i \leq \frac{n}{2}$, then }f_i(x)=
\left\{
\begin{array}{cl}
\dfrac{ x^2}{2}+\dfrac{Gx}{2}, &\text{ if }x\geq 0\\
\dfrac{ L x^2}{2}+\dfrac{Gx}{2},& \text{ if }x < 0,
\end{array}
\right.
\end{equation*}
and the other half of the functions are defined as follows (we call these as functions of second kind):
\begin{equation*}
\text{if $i > \frac{n}{2}$, then }f_i(x)=
\left\{
\begin{array}{cl}
\dfrac{ x^2}{2}-\dfrac{Gx}{2},& \text{ if }x\geq 0\\
\dfrac{ L x^2}{2}-\dfrac{Gx}{2},& \text{ if }x < 0.
\end{array}
\right.
\end{equation*}

Let $\sigma^j$ be the permutation of functions $f_i$'s that is used in the $j$-th epoch. Then, $\sigma^j$ can be represented by a permutation of the following multiset: $$\{\underbrace{+1,\dots ,+1}_{\frac{n}{2}\text{ times}},\underbrace{-1,\dots ,-1}_{\frac{n}{2}\text{ times}}\}$$ 
Accordingly, if $\sigma_i^j=+1$, we assume in the $i$-th iteration of the $j$-th epoch, a function of the 1st kind was sampled. Similarly if $\sigma_i^j=-1$, we assume in the $i$-th iteration of the $j$-th epoch, a function of the 2nd kind was sampled.

\begin{figure}[H]
	\begin{center}
	\includegraphics[width=0.7\textwidth]{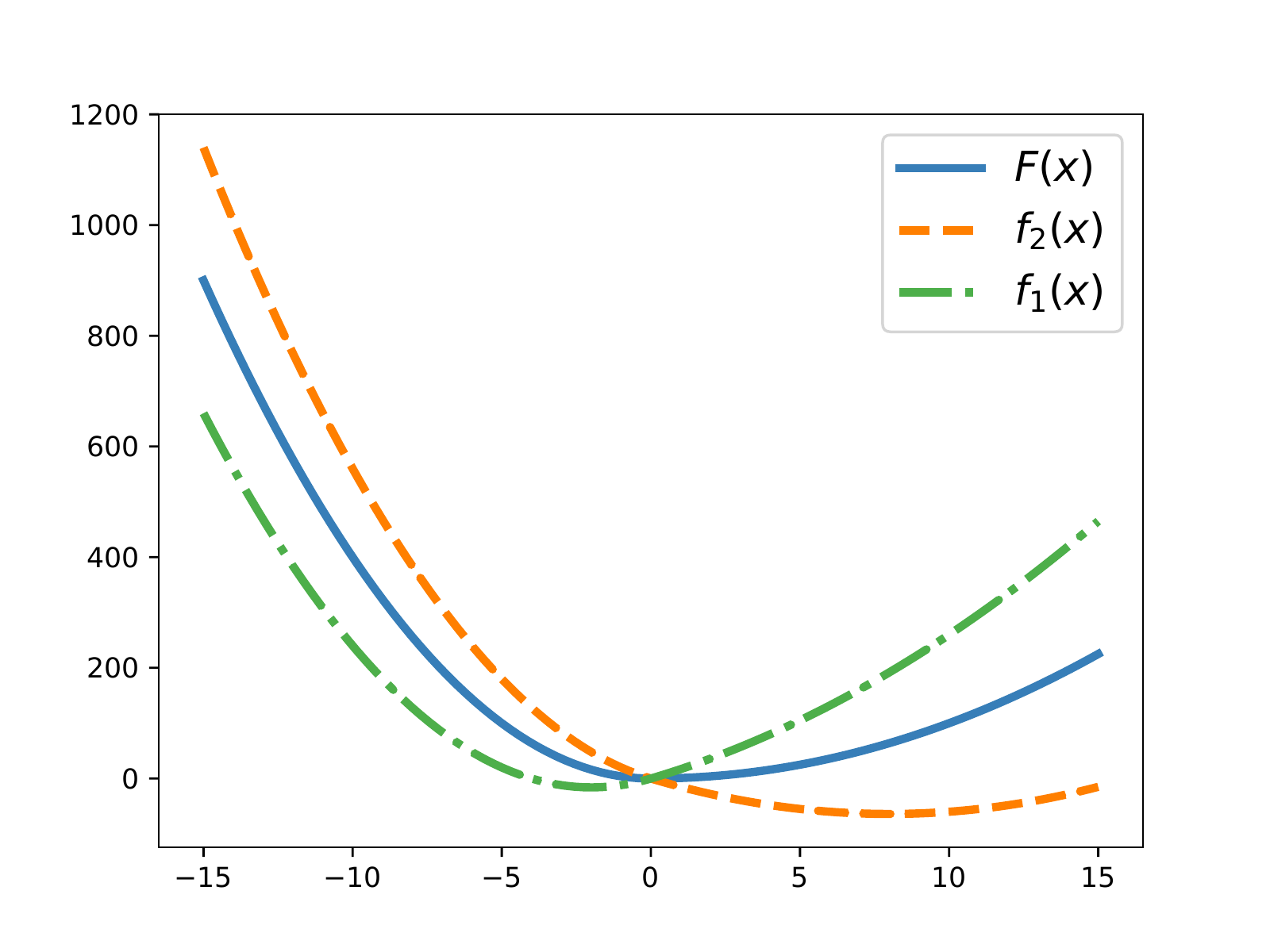}
	\caption{Lower bound construction. Note that $f_1(x)$ represents functions of the first kind, and $f_2(x)$ represents functions of the second kind, and $F(x)$ represents the overall function.}
	\end{center}
\end{figure}

Assumptions \ref{ass:convex}-\ref{ass:lipschitz}, except Assumptions \ref{ass:boundedD} and \ref{ass:boundedG} have been already proved for this function in the main text of this paper. Thus, we only prove \ref{ass:boundedD} and \ref{ass:boundedG} here.

We will initialize at $x_0^1=0$. Next we prove that Assumption \ref{ass:boundedG} is satisfied for our lowerbound construction. The minimizer of the functions of first kind is at $x=-G/2L$, and the minimizer of the functions of the second kind is at $x=G/2$. Between these two values, the norm of gradient of any of the two function kinds is always less than $G$. Thus, it is sufficient show that the iterates stay between the two minimizers. The step size we have chosen is small enough (smaller than $1/L$) to ensure that the iterates do not go outside the two minimizers. To see why this is so, consider the case when one is doing gradient descent on $g(x)=ax^2/2$. If the step length $\alpha <1/a$, then the next iterate $x_{t+1}=x_t-\alpha (ax_t)= x_t(1-\alpha a)$ and the current iterate $x_t$ lie on the same side of the minimizer (which is $x^*=0$), that is the iterates never `cross over' the minimizer. A similar logic here implies that iterates in our case stay within $[-G/2L,G/2]$. Thus gradient norm would never exceed $G$.

Next, we have the following key lemma of the proof:

\begin{lemma}\label{lem:lb1}
If $n\geq 256$ is a multiple of $4$, $L\geq 2^{17}$, $\alpha \leq \frac{2^{-11}}{nL}$ and $|x_0^j| \leq 2^{-9} G\alpha \sqrt{n}$, then 
\begin{equation*}
\mathbb{E}\left[x_0^{j+1}\right] \geq x_0^{j} + 2^{-12} L  G \alpha^2 n\sqrt{n}
\end{equation*}

\end{lemma}
Intuitively, what Lemma \ref{lem:lb1} implies is that $\mathbb{E}[|x_0^{j+1}|]$ should keep increasing at rate $\Omega (\alpha^2 n\sqrt{n})$ until $|x_0^j|=\Omega(\alpha \sqrt{n})$. This is exactly what we need, because for the range of step length $\alpha$ specified in the theorem statement, $|x_0^j|=\Omega(\alpha \sqrt{n})$ is the claimed error lower bound. The rest of the proof is just making this intuition rigorous.

The following helper lemma says that $\mathbb{E}[|x_0^{j+1}|]$ does not decrease by too much, even if $|x_0^j|> 2^{-9}G\alpha \sqrt{n}$. 

\begin{lemma}\label{cor:lb2}
If $\alpha \leq 1/nL$, then
\begin{equation*}
\mathbb{E}\left[|x_0^{j+1}| \right] \geq |x_0^j|(1- 2 L  \alpha n) - L G\alpha^2 n\sqrt{5n},
\end{equation*}
and further if $|x_0^j| > 2^{-9} G \alpha \sqrt{n}$, then 
\begin{equation*}
\mathbb{E}\left[x_0^{j+1}\right] \geq x_0^j - 2^{-11} |x_0^j| L   \alpha n.
\end{equation*}
\end{lemma}

Next, we aim to use the two lemmas above to get lower bounds on unconditioned expectation $\mathbb{E}[|x_0^{j+1}|]$. For this, we consider the following two cases. 

\textbf{Case 1: If $\mathbb{P}(|x_0^j| > 2^{-9}G\alpha \sqrt{n})\mathbb{E}\left[|x_0^{j}|\middle| |x_0^j|> 2^{-9}G\alpha\sqrt{n}\right] > 2^{-10} G\alpha \sqrt{n}$}. Then decomposing the expectation into conditional expectations,
\begin{align}
\mathbb{E}\left[|x_0^{j}|\right]&= \mathbb{P}\left(|x_0^j| > 2^{-9}G\alpha \sqrt{n}\right)\mathbb{E}\left[|x_0^j|\middle||x_0^j|> 2^{-9}G\alpha\sqrt{n}\right]+\mathbb{P}\left(|x_0^j| \leq 2^{-9}G\alpha \sqrt{n}\right)\mathbb{E}\left[|x_0^j|\middle||x_0^j|\leq 2^{-9}G\alpha\sqrt{n}\right]&\nonumber\\
&\geq \mathbb{P}\left(|x_0^j| > 2^{-9}G\alpha \sqrt{n}\right)\mathbb{E}\left[|x_0^j|\middle||x_0^j|> 2^{-9}G\alpha\sqrt{n}\right]&\nonumber\\
&> 2^{-10}G\alpha \sqrt{n}.\label{eq:case1new}&
\end{align}

\textbf{Case 2: Otherwise, $\mathbb{P}(|x_0^j| > 2^{-9}G\alpha \sqrt{n})\mathbb{E}\left[|x_0^{j}|\,\middle|\,|x_0^j|> 2^{-9}G\alpha\sqrt{n}\right] \leq 2^{-10}G\alpha \sqrt{n}$}. Again decomposing the expectation into conditional expectations,
\begin{align}
\mathbb{E}\left[x_0^{j+1}\right]&=\mathbb{P}\left(|x_0^j|\leq 2^{-9}G\alpha \sqrt{n}\right)\mathbb{E}\left[x_0^{j+1}\middle||x_0^j|\leq 2^{-9}G\alpha\sqrt{n}\right]+ \mathbb{P}\left(|x_0^j| > 2^{-9}G\alpha \sqrt{n}\right)\mathbb{E}\left[x_0^{j+1}\middle||x_0^j|> 2^{-9}G\alpha\sqrt{n}\right]&\nonumber\\
 &\geq\mathbb{P}\left(|x_0^j|\leq 2^{-9}G\alpha \sqrt{n}\right)\left (\mathbb{E}\left[x_0^{j}\middle||x_0^j|\leq 2^{-9}G\alpha\sqrt{n}\right] + 2^{-12}L G\alpha^2 n \sqrt{n}\right)&\nonumber\\
&\quad + \mathbb{P}\left(|x_0^j| > 2^{-9}G\alpha \sqrt{n}\right)\left (\mathbb{E}\left[x_0^{j}\middle||x_0^j|> 2^{-9}G\alpha\sqrt{n}\right] - 2^{-11} L \mathbb{E}\left[|x_0^{j}|\middle||x_0^j|> 2^{-9}G\alpha\sqrt{n}\right] \alpha n \right)&\nonumber\\
&\tag*{[Using Lemma \ref{lem:lb1} and \ref{cor:lb2}]}&\nonumber\\
&= \mathbb{E}\left[x_0^{j}\right] + \mathbb{P}\left(|x_0^j|\leq 2^{-9}G\alpha \sqrt{n}\right)  2^{-12}L G\alpha^2 n \sqrt{n}&\nonumber\\
&\quad - \mathbb{P}\left(|x_0^j| > 2^{-9}G\alpha \sqrt{n}\right) \mathbb{E}\left[|x_0^{j}|\middle||x_0^j|> 2^{-9}\alpha\sqrt{n}\right]2^{-11} L \alpha n .&\label{eq:case2new}
\end{align}
In the last step above, we gathered the sum of conditional expectations back into an unconditional expectation.

By assumption of this case, 
\begin{equation*}
\mathbb{P}\left(|x_0^j| > 2^{-9}G\alpha \sqrt{n}\right)\mathbb{E}\left[|x_0^{j}|\middle||x_0^j|> 2^{-9}G\alpha\sqrt{n}\right] \leq 2^{-10}G\alpha \sqrt{n}.
\end{equation*}
This implies that $\mathbb{P}\left(\left|x_0^j\right| > 2^{-9}G\alpha \sqrt{n}\right) \leq \frac{1}{2}$ and thus, $\mathbb{P}\left(\left|x_0^j\right| \leq 2^{-9}G\alpha \sqrt{n}\right) \geq \frac{1}{2}$. Using this in Ineq.~\eqref{eq:case2new},
\begin{align}
\mathbb{E}\left[x_0^{j+1}\right]&\geq \mathbb{E}\left[x_0^{j}\right] + \frac{1}{2}  2^{-12}L G\alpha^2 n \sqrt{n} -   \mathbb{P}\left(\left|x_0^j\right| > 2^{-9}G\alpha \sqrt{n}\right) \mathbb{E}\left[|x_0^{j}|\middle|\left|x_0^j\right|> 2^{-9}G\alpha\sqrt{n}\right]2^{-11} L \alpha n &\nonumber\\
&\geq \mathbb{E}\left[x_0^{j}\right] +   2^{-13}L G\alpha^2 n \sqrt{n}-  \left(2^{-10} G\alpha \sqrt{n}\right) 2^{-11} L\alpha n &\nonumber\tag*{[By assumption of this case]}\\
&\geq \mathbb{E}\left[x_0^{j}\right] +   2^{-14}L G\alpha^2 n \sqrt{n}.&\label{eq:case1}
\end{align}

What we have shown using the two cases is the following: If for some epoch $j$, $\mathbb{E}[|x_0^{j}|] \leq 2^{-10} G \alpha \sqrt{n}$ then looking at Ineq.~\eqref{eq:case1new} tells us that we are in Case 2. Then, $\mathbb{E}[x_0^{j+1}]\geq \mathbb{E}[x_0^{j}] +   2^{-14}L G\alpha^2 n \sqrt{n}$, that is $\mathbb{E}[x_0^{j+1}]$ increases by $\Omega(\alpha^2 n\sqrt{n})$. This shows that if we initialize $x_0^1$ at 0, then until $\mathbb{E}[|x_0^{j}|] > 2^{-10} G \alpha \sqrt{n}$, the expected error will keep increasing at rate $\Omega(\alpha^2 n\sqrt{n})$. Thus given the step size regime considered in this theorem, there is some epoch where error reaches $\Omega(\alpha \sqrt{n})$, which is the desired lower bound. However, what we want to show is that at the end of $K$ epoch, the error is still $\Omega(\alpha \sqrt{n})$. We will prove this next.

We initialize $x_0^1=0$ and we run $K$ epochs. Then, because $\mathbb{E}[|x_0^{j}|]\geq \mathbb{E}[x_0^{j}]$, we have shown in the previous paragraph that $\mathbb{E}[|x_0^{j}|]\geq \min \left\{ 2^{-10}, 2^{-14}L \alpha nK  \right\}G\alpha \sqrt{n}$ for some $0\leq j\leq K$. Now, for our given range of $\alpha $, we know that $L  \alpha n K $ is greater than $L$, which in turn is greater than $2^{17}$.
 Thus, $\mathbb{E}[|x_0^{j}|]\geq 2^{-10} G\alpha \sqrt{n}$. 
 To complete the proof, next we prove that once $\mathbb{E}[|x_0^{j}|]\geq 2^{-10} G\alpha \sqrt{n}$, then $\mathbb{E}[|x_0^{t}|]$ remains above $C_l G\alpha \sqrt{n}$ for $t>j$ and some universal constant $C_l$. The strategy is that we will show that if $\mathbb{E}[|x_0^{j}|]$ starts falling below $2^{-10} G\alpha \sqrt{n}$, then $\mathbb{E}[x_0^{j}]$ starts increasing. We also have Lemma \ref{lem:posExp} (given below) that says that $\mathbb{E}[x_0^{j}]\geq 0$ always; and the trivial fact that $\mathbb{E}[|x_0^{j}|] \geq \mathbb{E}[x_0^{j}]$ always. 
 All these together and some simple arithmetic will give a bound on how much $\mathbb{E}[|x_0^{j}|]$ can decrease.

\begin{lemma}\label{lem:posExp} If $\alpha \leq 1/L$, then
$\forall i, j: \mathbb{E}[x_i^j] \geq 0.$
\end{lemma}

We formalize the argument of the previous paragraph next. 
Let $j$ be such that $\mathbb{E}[|x_0^j|] \geq 2^{-10} G \alpha \sqrt{n}$ and $\mathbb{E}[|x_0^{j+1}|] < 2^{-10} G \alpha \sqrt{n}$. Then,
\begin{align*}
\mathbb{E}\left[|x_0^{j+1}|\right] &\geq \mathbb{E}\left[|x_0^{j}|\right](1-2L \alpha n)-\sqrt{5} L G \alpha^2 n\sqrt{n}\tag*{[Using Lemma \ref{cor:lb2}]}&\\
&\geq \frac{1}{2}\mathbb{E}\left[\left|x_0^{j}\right|\right]-\sqrt{5} LG   \alpha^2 n\sqrt{n}\tag*{[Since $\alpha \leq \frac{1}{4 nL }$]}&\\
&\geq 2^{-11} G \alpha \sqrt{n}-\sqrt{5} L G \alpha^2 n\sqrt{n}&\\
&\geq 2^{-12} G \alpha \sqrt{n}.\tag*{[Since $\alpha \leq \frac{2^{-14}}{nL }$]}&
\end{align*}

For subsequent epochs $l>j$, we want to show that $\mathbb{E}[|x_0^{l}|]$ doesn't fall below $\Omega(\alpha \sqrt{n})$. Thus assume that for $l>j$, $\mathbb{E}[|x_0^{l}|] < 2^{-10} G \alpha \sqrt{n}$, because otherwise $\mathbb{E}[|x_0^{l}|] =\Omega(\alpha \sqrt{n})$.

Because $\mathbb{E}[|x_0^{l}|] < 2^{-10} G \alpha \sqrt{n}$ 
, then using Ineq.~\eqref{eq:case1new}, we can infer that we are in Case 2 in the epoch $l$. Therefore, Ineq.~\eqref{eq:case1} implies that for each such epoch $l$, $\mathbb{E}[x_0^{l+1}]$ increases by at least $2^{-14} L G \alpha^2 n \sqrt{n}$ per epoch; whereas Lemma \ref{cor:lb2} says that $\mathbb{E}[|x_0^{l+1}|]$ can decrease by at most $2 L \mathbb{E}[|x_0^{l}|]\alpha n +\sqrt{5} LG\alpha^2 n\sqrt{n}\leq (2^{-11} +\sqrt{5})L G\alpha^2 n\sqrt{n}$ per epoch.
Further, we have the two facts that $\forall l: \mathbb{E}[x_0^l] \geq 0$ (Lemma \ref{lem:posExp}) and $\mathbb{E}[x_0^l]\leq \mathbb{E}[|x_0^l|]$. Combining all these and using simple arithmetic gives that $\mathbb{E}[|x_0^l|]$ can decrease to at most 
\begin{equation*}
2^{-12} G \alpha \sqrt{n}\left(\frac{2^{-14}L G\alpha^2n\sqrt{n}}{2^{-14} L G\alpha^2n\sqrt{n}+(2^{-11}+\sqrt{5})L G\alpha^2 n\sqrt{n}}\right) \geq 2^{-28} G\alpha \sqrt{n}.
\end{equation*}

After this $\mathbb{E}[|x_0^l|]$ will have to keep increasing because $\mathbb{E}[x_0^l]$ will keep increasing till we enter Case 1, and then this cycle may repeat, but we have shown that regardless, $\mathbb{E}[|x_0^l|]$ always remains above $2^{-28} G\alpha \sqrt{n}$. 

Finally, given $\alpha \geq \frac{1}{nK} $, we get that $\mathbb{E}[|x_0^K|]\geq 2^{-28} G\alpha \sqrt{n} \geq \frac{2^{-28} G }{\sqrt{n}K}$. Applying Jensen's inequality on this gives $\mathbb{E}[|x_0^K|^2] \geq \frac{2^{-56} G^2 }{nK^2}= \frac{2^{-56} G^2 n}{T^2}$.

\end{proof}

\subsection{Proof of Lemma \ref{lem:lb1}}

The gradient computed at $x_i^j$ can be written as $(\mathds{1}_{x_i^j\leq 0}L+\mathds{1}_{x_i^j > 0}1) x_i^j + \frac{G}{2}\sigma_i^j$. Then, $x_0^{j+1}-x_0^j$ is just the sum of the gradient steps taken through the epoch. 
\begin{align}
\mathbb{E}[x_0^{j+1}]&=x_0^{j}-\alpha \frac{G}{2}\sum_{i=1}^n \sigma_i^j -\alpha \sum_{i=0}^n\mathbb{E}\left[(\mathds{1}_{x_i^j\leq 0}L+\mathds{1}_{x_i^j > 0}1) x_i^j\right]&\nonumber\\
&=x_0^{j} -\alpha \sum_{i=1}^n\mathbb{E}\left[(\mathds{1}_{x_i^j\leq 0}L+\mathds{1}_{x_i^j > 0}1)  x_i^j\right]&\tag*{[Since $\sum_{i=1}^n \sigma_i^j=0$]}\\
&=x_0^{j} -\alpha \sum_{i=n/4}^{n/2}\mathbb{E}\left[(\mathds{1}_{x_i^j\leq 0}L+\mathds{1}_{x_i^j > 0}1)  x_i^j\right]-\alpha \sum_{i\notin [\frac{n}{4},\frac{n}{2}]}\mathbb{E}\left[(\mathds{1}_{x_i^j\leq 0}L+\mathds{1}_{x_i^j > 0}1)  x_i^j\right]&\nonumber\\
&= x_0^j - \alpha  \sum_{i=n/4}^{n/2} \mathbb{P}\left( \sum_{p=1}^{i}\sigma_p^j> 0\right)\mathbb{E}\left[(\mathds{1}_{x_i^j\leq 0}L+\mathds{1}_{x_i^j > 0}1) x_i^j\middle| \sum_{p=1}^{i}\sigma_p^j> 0\right] &\nonumber\\
&\quad - \alpha  \sum_{i=n/4}^{n/2} \mathbb{P}\left( \sum_{p=1}^{i}\sigma_p^j\leq 0\right)\mathbb{E}\left[(\mathds{1}_{x_i^j\leq 0}L+\mathds{1}_{x_i^j > 0}1) x_i^j\middle| \sum_{p=1}^{i}\sigma_p^j\leq 0\right] &\nonumber\\
&\quad -\alpha  \sum_{i\notin [\frac{n}{4},\frac{n}{2}]}\mathbb{E}\left[(\mathds{1}_{x_i^j\leq 0}L+\mathds{1}_{x_i^j > 0}1) x_i^j\right].\label{eq:long}
\end{align}
We decomposed the expectation into the sum of conditional expectations to achieve the last equality.

The following inequalities help us bound $\mathbb{E}\left[(\mathds{1}_{x_i^j\leq 0}L+\mathds{1}_{x_i^j > 0}1) x_i^j\right]$ in Eq.~\eqref{eq:long} with simple expressions. 
Let $r$ be any random variable. Then,
\begin{align}
\mathbb{E}\left[(\mathds{1}_{r\leq 0}L+\mathds{1}_{r > 0}1) r\right] &= L \mathbb{P}(r<0)\mathbb{E}[r\mid r<0]+ \mathbb{P}(r\geq0)\mathbb{E}[r|r\geq0]&\nonumber\\
&= L \left(\mathbb{P}(r<0)\mathbb{E}[r|r<0]+ \mathbb{P}(r\geq0)\mathbb{E}[r|r\geq0]\right) + (1-L)\mathbb{P}(r\geq0)\mathbb{E}[r|r\geq0]&\nonumber\\
&\leq L \left(\mathbb{P}(r<0)\mathbb{E}[r|r<0]+ \mathbb{P}(r\geq0)\mathbb{E}[r|r\geq0]\right)&\nonumber\tag*{[Since $L\geq 1$]}\\
&\leq L \mathbb{E}[r]\label{eq:minEq1}
\end{align}
and
\begin{align}
\mathbb{E}\left[(\mathds{1}_{r\leq 0}L+\mathds{1}_{r > 0}1) r\right] &= L \mathbb{P}(r<0)\mathbb{E}[r|r<0]+ \mathbb{P}(r\geq0)\mathbb{E}[r|r\geq0]&\nonumber\\
&= 1 (\mathbb{P}(r<0)\mathbb{E}[r|r<0]+ \mathbb{P}(r\geq0)\mathbb{E}[r|r\geq0]) + (L-1)\mathbb{P}(r<0)\mathbb{E}[r|r<0]&\nonumber\\
&\leq 1 (\mathbb{P}(r<0)\mathbb{E}[r|r<0]+ \mathbb{P}(r\geq0)\mathbb{E}[r|r\geq0]) &\nonumber \tag*{[Since $L\geq 1$]}\\
&\leq 1 \mathbb{E}[r].\label{eq:minEq2}
\end{align}

We also have the following lemmas and corollary that along with the two inequalities above, help lower bound the RHS of Eq.~\eqref{eq:long}.
\begin{lemma}\label{lem:lb1Supp}
If $|x_0^{j}| \leq  \sqrt{5n} G \alpha$ and $\alpha \leq \frac{2^{-11}}{nL}$, then for $n/4 \leq i\leq n/2$, we have 
\begin{equation*}
\mathbb{E}\left[x_i^j - x_0^{j}\middle| \sum_{p=1}^{i}\sigma_p^j> 0\right]\leq -\frac{1}{128}G\sqrt{i}\alpha.
\end{equation*}
\end{lemma}

\begin{restatable}[]{lemma}{lemthirteen}\label{lem:expDev1}
If $n\geq 256$ is a multiple of $4$ and $i\leq n/2$, then
\begin{equation*}
\frac{\sqrt{i}}{32}\leq \mathbb{E}\left[\left|\sum_{p=1}^i\sigma_p^j\right|\right]\leq \sqrt{i}.
\end{equation*}
Further, for any $i\in [n/4, n/2]$,
\begin{equation*}
\mathbb{P}\left( \sum_{p=1}^{i}\sigma_p^j< 0\right)\geq \frac{1}{4}\qquad\text{ and }\qquad\mathbb{P}\left( \sum_{p=1}^{i}\sigma_p^j> 0\right)\geq \frac{1}{4}.
\end{equation*}

\end{restatable}

\begin{restatable}[]{corollary}{cortwo}
\label{cor:2}
If $\alpha \leq 2/L$ and $i\in [\frac{n}{4},\frac{n}{2}]$ then for any $h\in [1,n]$,
\begin{align*}
\mathbb{E}\left[|x_h^j - x_0^j|\middle|\sum_{p=1}^{i}\sigma_p^j\leq 0\right]&\leq 4\left( \sqrt{5h}G\alpha +| x_0^j|\sqrt{h\alpha L}\right)\text{, and}\\
\mathbb{E}\left[|x_h^j - x_0^j|\middle|\sum_{p=1}^{i}\sigma_p^j> 0\right]&\leq 4 \left( \sqrt{5h}G\alpha +| x_0^j|\sqrt{h\alpha L}\right).
\end{align*}
\end{restatable}

Now, we handle each of the three terms in Eq.~\eqref{eq:long} individually.
\begin{itemize}
\item If $i\in [n/4,n/2]$ and $\left( \sum_{p=1}^{i}\sigma_p^j\right)> 0$:

\begin{align}
\mathbb{E}\left[(\mathds{1}_{x_i^j\leq 0}L+\mathds{1}_{x_i^j > 0}1) x_i^j\middle|\sum_{p=1}^{i}\sigma_p^j> 0\right]&\leq L\mathbb{E}\left[x_i^j\middle| \sum_{p=1}^{i}\sigma_p^j> 0\right] &\tag*{[Using Ineq.~\eqref{eq:minEq1} with $r=(x_i^j\mid\sum_{p=1}^{i}\sigma_p^j> 0)$]}\nonumber\\
&= L\mathbb{E}\left[x_i^j - x_0^j\middle|\sum_{p=1}^{i}\sigma_p^j> 0\right] +Lx_0^j&\nonumber\\
&\leq -\frac{L}{128}\sqrt{\frac{n}{4}}G\alpha + \frac{L}{512} G\sqrt{n}\alpha&\tag*{[Using Lemma \ref{lem:lb1Supp}]}\nonumber\\
&\leq -\frac{L}{512}G\sqrt{n}\alpha.&\label{ineq:i1}
\end{align}

\item If $i\in [n/4,n/2]$ and $\left( \sum_{p=1}^{i}\sigma_p^j\right)\leq 0$:
\begin{align}
\mathbb{E}\left[(\mathds{1}_{x_i^j\leq 0}L+\mathds{1}_{x_i^j > 0}1)x_i^j\middle| \sum_{p=1}^{i}\sigma_p^j\leq 0\right]&\leq 1\mathbb{E}\left[x_i^j\middle|\sum_{p=1}^{i}\sigma_p^j\leq 0\right]&\tag*{[Using Ineq.~\eqref{eq:minEq2} with $r=(x_i^j\mid\sum_{p=1}^{i}\sigma_p^j\leq 0)$]}\nonumber\\
&\leq \mathbb{E}\left[x_i^j-x_0^j\middle| \sum_{p=1}^{i}\sigma_p^j\leq 0\right] + |x_0^j|&\nonumber\\
&\leq 4 \left( \sqrt{5i}G\alpha +|x_0^j|\sqrt{2i\alpha L}\right)+|x_0^j|\tag*{[Using Corollary \ref{cor:2} with $h=i$.]}&\\
&\leq 4 \left( \sqrt{5i}G\alpha +|x_0^j|\right)+|x_0^j|\tag*{[Since $\alpha \leq 1/nL$]}&\\
&\leq 4 \left( \sqrt{5i}G\alpha +2^{-9}G\alpha\sqrt{n}\right)+2^{-9}G\alpha\sqrt{n}\tag*{[Since $|x_0^j| \leq 2^{-9}G\alpha\sqrt{n}$ by assumption]}&\\
&\leq  8G\sqrt{n}\alpha.&\label{ineq:i2}
\end{align}

\item If $i\notin [n/4,n/2]$:
\begin{align}
\mathbb{E}\left[(\mathds{1}_{x_i^j\leq 0}L+\mathds{1}_{x_i^j > 0}1)x_i^j\right]&\leq 1\mathbb{E}\left[x_i^j\right]&\tag*{[Using Ineq.~\eqref{eq:minEq2} with $r=x_i^j$]}\nonumber\\
&\leq \mathbb{E}\left[x_i^j-x_0^j\right] + |x_0^j|&\nonumber\\
&\leq \left( \sqrt{5i}G\alpha +|x_0^j|\sqrt{2i\alpha L}\right)+|x_0^j|&\tag*{[Using Lemma \ref{lem:jainLem} with $x^*=0$]}\nonumber\\
&\leq \left( \sqrt{5i}G\alpha +|x_0^j|\sqrt{2}\right)+|x_0^j|&\tag*{[Since $\alpha \leq 1/nL$]}\nonumber\\
&\leq  8 G\sqrt{n}\alpha.&\label{ineq:i3}
\end{align}
\end{itemize}

Continuing on from Eq.~\eqref{eq:long} and substituting Ineq.~\eqref{ineq:i1}, Ineq.~\eqref{ineq:i2} and Ineq.~\eqref{ineq:i3}:

{
\begin{align*}
\mathbb{E}\left[x_0^{j+1}\right] &= x_0^j - \alpha  \sum_{i=n/4}^{n/2} \mathbb{P}\left( \sum_{p=1}^{i}\sigma_p^j> 0\right)\mathbb{E}\left[(\mathds{1}_{x_i^j\leq 0}L+\mathds{1}_{x_i^j > 0}1) x_i^j\middle| \sum_{p=1}^{i}\sigma_p^j> 0\right] &\nonumber\\
&\quad- \alpha  \sum_{i=n/4}^{n/2} \mathbb{P}\left(\sum_{p=1}^{i}\sigma_p^j\leq 0\right)\mathbb{E}\left[(\mathds{1}_{x_i^j\leq 0}L+\mathds{1}_{x_i^j > 0}1) x_i^j\middle| \sum_{p=1}^{i}\sigma_p^j\leq 0\right] \\
&\quad -\alpha \sum_{i\notin [\frac{n}{4},\frac{n}{2}]}\mathbb{E}\left[(\mathds{1}_{x_i^j\leq 0}L+\mathds{1}_{x_i^j > 0}1) x_i^j\right]&\\
&\geq  x_0^j - \alpha  \sum_{i=n/4}^{n/2} \mathbb{P}\left(\sum_{p=1}^{i}\sigma_p^j> 0\right)\left(-\frac{L}{512}G\sqrt{n}\alpha\right)\nonumber -\alpha  \sum_{i=n/4}^{n/2} \mathbb{P}\left( \sum_{p=1}^{i}\sigma_p^j\leq 0\right)(8G\sqrt{n}\alpha)&\\ 
&\quad -\alpha  \sum_{i\notin [\frac{n}{4},\frac{n}{2}]}(8G\sqrt{n}\alpha)&\tag*{[Using Ineq.~\eqref{ineq:i1}, \eqref{ineq:i2} and \eqref{ineq:i3}]}\\
&\geq  x_0^j - \alpha  \sum_{i=n/4}^{n/2} \mathbb{P}\left( \sum_{p=1}^{i}\sigma_p^j> 0\right)\left(-\frac{L}{512}G\sqrt{n}\alpha\right)-\alpha  \sum_{i= 1}^n(8G\sqrt{n}\alpha)&\\
&\geq  x_0^j + \alpha  \frac{n}{4} \frac{1}{4}\left(\frac{L}{512}G\sqrt{n}\alpha\right)-\alpha  n(8G\sqrt{n}\alpha)&\tag*{[Using Lemma \ref{lem:expDev1}]}&\\
&\geq  x_0^j +  2^{-12}LG\alpha^2n\sqrt{n}.&\tag*{[Since $L\geq 2^{17} $]}
\end{align*}
}

\subsection{Proof of Lemma \ref{cor:lb2}}

We will start off by bounding the difference of iterates between the start of two epochs.
The gradient computed at $x_i^j$ can be written as $(\mathds{1}_{x_i^j\leq 0}L+\mathds{1}_{x_i^j > 0}1) x_i^j + \frac{G}{2}\sigma_i^j$. Then, $x_0^{j+1}-x_0^j$ is just the sum of the gradient steps taken through the epoch.

\begin{align*}
\mathbb{E}\left[|x_0^{j+1}-x_0^{j}|\right]=&\mathbb{E}\left[\left|-\alpha  \frac{G}{2}\sum_{i=1}^n \sigma_i^j -\alpha \sum_{i=0}^n(\mathds{1}_{x_i^j\leq 0}L+\mathds{1}_{x_i^j > 0}1) x_i^j\right|\right]&\\
=&\mathbb{E}\left[\left|-\alpha  \sum_{i=0}^n(\mathds{1}_{x_i^j\leq 0}L+\mathds{1}_{x_i^j > 0}1) x_i^j\right|\right]&\tag*{[Since $\sum_{i=1}^n \sigma_i^j=0$]}&\\
\leq&\alpha  \mathbb{E}\left[\sum_{i=0}^n (\mathds{1}_{x_i^j\leq 0}L+\mathds{1}_{x_i^j > 0}1) |x_i^j|\right]&\\
\leq&\alpha L  \mathbb{E}\left[\sum_{i=0}^n |x_i^j|\right]&\\
\leq&\alpha L  \mathbb{E}\left[\sum_{i=0}^n(| x^j_0 - x_i^j| + | x^j_0 |)\right]&\\
\leq&\alpha L  \mathbb{E}\left[\sum_{i=0}^n(\sqrt{i\alpha L  }| x^j_0| + \sqrt{5i}\alpha G + | x^j_0 |)\right]&\tag*{[Using Lemma \ref{lem:jainLem}]}\\
\leq&\alpha L  \mathbb{E}\left[\sum_{i=0}^n( \sqrt{5i}\alpha G + 2| x^j_0 |)\right]&\tag*{[Since $\alpha \leq 1/nL$]}\\
\leq&\alpha L  n \left(\sqrt{5n}\alpha G + 2| x^j_0 |\right).&
\end{align*}
Thus, we have shown the following:
\begin{equation}
\mathbb{E}\left[|x_0^{j+1}-x_0^{j}|\right] \leq 2 L|x_0^{j}| \alpha n +  LG \alpha^2 n\sqrt{5n}.\label{lb2new}
\end{equation}
Then,
\begin{align*}
\mathbb{E}\left[|x_0^{j+1}| \right] &=\mathbb{E}\left[|x_0^{j}+x_0^{j+1}-x_0^{j}| \right] &\\
 &\geq  |x_0^{j}|- \mathbb{E}\left[|x_0^{j+1}-x_0^{j}| \right] &\\
&\geq |x_0^j|(1- 2 L \alpha n) - LG \alpha^2 n\sqrt{5n}.&\tag*{[Using Ineq.~\eqref{lb2new}]}
\end{align*}
This proves the first inequality of the lemma. For the second inequality,
\begin{align*}
\mathbb{E}\left[x_0^{j+1}\middle||x_0^j| > \frac{G \alpha \sqrt{n}}{512} \right]& =\mathbb{E}\left[x_0^{j} + x_0^{j+1}-x_0^j\middle||x_0^j| > 2^{-9} G \alpha \sqrt{n}\right]&\\
&\geq x_0^{j} - \mathbb{E}\left[|x_0^{j+1}-x_0^j|\middle||x_0^j| > 2^{-9} G \alpha \sqrt{n}\right]&\\
&\geq x_0^{j} - (2|x_0^{j}| L \alpha n +  LG \alpha^2 n\sqrt{5n})&\tag*{[Using Ineq.~\eqref{lb2new}]}\\
&\geq x_0^{j} - \left(2 L|x_0^{j}| \alpha n + 512 \sqrt{5} L|x_0^{j}| \alpha n\right)&\tag*{[Since $|x_0^j| > \frac{G \alpha \sqrt{n}}{512} $]}\\
& \geq x_0^j - 2^{-11} L |x_0^j| \alpha n.&
\end{align*}

\subsection{Proof of Lemma \ref{lem:posExp}}
Let us denote the iterate at iteration $t$ by $x_{t,L}$, where the $L$ in subscript is to show dependence on $L$.
Now, consider the problem setup when $L=1$. In that case, the two kinds of functions are simply quadratics instead of piecewise quadratics. The two kinds of functions in this case are $f_1(x)= \frac{1}{2}x^2+\frac{G}{2}x$ and $f_2(x)= \frac{1}{2}x^2-\frac{G}{2}x$. Then, due to symmetry of the function, the expected value of iterate at any iteration $t$ is 0: $\mathbb{E}[x_{t,1}]=0$.\\
Now, take the problem setup for the case $L\geq 1$, where we denote the iterates by $x_{t,L}$. We couple the two iterates $x_{t,1}$ and $x_{t,L}$ such that both $x_{t,L}$ and $x_{t,1}$ are created using the exact same random permutations of the functions but $x_{t,1}$ had $L=1$ and $x_{t,L}$ had $L\geq 1$. Then, we show that $x_{t,L}\geq x_{t,1}$ and since $\mathbb{E}[x_{t,1}]= 0$, we get that $\mathbb{E}[x_{t,L}]\geq 0$.\\
We prove this by induction. 
\begin{itemize}
\item Base case: Both are initialized at 0, and thus $x_{0,1}=x_{0,L}=0$.
\item Inductive case (assume $ x_{i,L} \geq x_{i,1}$): We break the analysis into the following three cases:
\begin{enumerate}
\item \textbf{Case: If $x_{i,1}\leq x_{i,L}\leq 0$.} Then note that regardless of the choice of the functions, the contribution of the linear gradients in the gradients would be the same for both $x_{i+1,1}$ and $ x_{i+1,L}$ (since we use the exact same permutation of functions for both). Hence,
\begin{align*}
x_{i+1,L}-x_{i+1,1} &= (1-\alpha L) x_{i,L} - (1-\alpha) x_{i,1}&\\
&\geq (1-\alpha L) x_{i,L} - (1-L\alpha) x_{i,1}&\\
&\geq 0.&\tag*{[Since $\alpha \leq 1/L$]}\\
\end{align*}
\item \textbf{Case: If $0\leq x_{i,1}\leq x_{i,L}$.} Again, regardless of the choice of the functions, the contribution of the linear gradients in the gradients would be the same for both $x_{i+1,1}$ and $ x_{i+1,L}$ (since we use the exact same permutation of functions for both). Hence,
\begin{align*}
x_{i+1,L}-x_{i+1,1} &= (1-\alpha ) x_{i,L} - (1-\alpha) x_{i,1}&\\
&\geq 0.\tag*{[Since $\alpha \leq 1/L$]}&\\
\end{align*}
\item \textbf{Case: If $ x_{i,1}\leq 0\leq x_{i,L}$.} Again, regardless of the choice of the functions, the contribution of the linear gradients in the gradients would be the same for both $x_{i+1,1}$ and $ x_{i+1,L}$ (since we use the exact same permutation of functions for both). Hence,
\begin{align*}
x_{i+1,L}-x_{i+1,1} &= (1-\alpha ) x_{i,L} - (1-\alpha) x_{i,1}&\\
&\geq 0.\tag*{[Since $\alpha \leq 1/L$]}&\\
\end{align*}
\end{enumerate}
\end{itemize}

\subsection{Proof of Lemma \ref{lem:lb1Supp}}

The gradient computed at $x_i^j$ can be written as $(\mathds{1}_{x_i^j\leq 0}L+\mathds{1}_{x_i^j > 0}1) x_i^j + \frac{G}{2}\sigma_i^j$. Then,
\begin{align}
\mathbb{E}\left[x_i^j - x_0^j\middle|\sum_{p=1}^{i}\sigma_p^j> 0\right]&=\mathbb{E}\left[-\alpha \sum_{p=1}^{i}\nabla f_{\sigma^j_p}( x_{p-1}^j)\middle| \sum_{p=1}^{i}\sigma_p^j> 0\right]&\nonumber\\
&=-\alpha\mathbb{E}\left[\sum_{p=1}^{i}\left( (\mathds{1}_{x_{p-1}^j\geq0}+L\mathds{1}_{x_{p-1}^j<0})x_{p-1}^j+\frac{G}{2}\sigma_p^j\right)\middle| \sum_{p=1}^{i}\sigma_p^j> 0\right]&\nonumber\\
&\leq -\alpha \mathbb{E}\left[- \left|\sum_{p=1}^{i}  (\mathds{1}_{x_{p-1}^j\geq0}+L\mathds{1}_{x_{p-1}^j<0})x_{p-1}^j\right|+\sum_{p=1}^{i}\frac{G}{2}\sigma_p^j \; \middle|\sum_{p=1}^{i}\sigma_p^j> 0\right]&\nonumber\\
&\leq - \alpha  \mathbb{E}\left[-  L\sum_{p=1}^{i}| x_{p-1}^j |+\frac{G}{2} \sum_{p=1}^{i}\sigma_p^j \; \middle| \sum_{p=1}^{i}\sigma_p^j> 0\right]&\nonumber\\
&\leq -\alpha  \mathbb{E}\left[\frac{G}{2}\left| \sum_{p=1}^{i}\sigma_p^j\right| -  L\sum_{p=1}^{i}\left(| x_0^j - x_{p-1}^j |+| x_0^j|\right)\middle| \sum_{p=1}^{i}\sigma_p^j> 0\right].&\label{eq:lwrDev}
\end{align}

We have the following helpful lemma which will help us control the first term in the sum above.

\lemthirteen*

We also use the following corollary (of Lemma \ref{lem:expDev1} and Lemma \ref{lem:jainLem}).

\cortwo*

Continuing from \eqref{eq:lwrDev}, and using Lemma \ref{lem:expDev1} and Corollary \ref{cor:2}, we have that for $n/4\leq i\leq n/2$
\begin{align*}
\mathbb{E}\left[x_i^j - x_0^j\middle| \sum_{p=1}^{i}\sigma_p^j> 0\right]&\leq -\frac{\sqrt{i}\alpha G}{64}  + \alpha  L i \left( 4\sqrt{5i}G\alpha +| x_0^j|\left(1+4\sqrt{i\alpha L}\right)\right).&\\
\end{align*}

Thus, if $|x_0^{j}| \leq \sqrt{5n}G\alpha$ and $\alpha\leq \frac{2^{-11}}{nL}$, then for $n/4 \leq i\leq n/2$, we have 
\begin{equation*}
\mathbb{E}\left[x_i^j - x_0^{j}\middle| \sum_{p=1}^{i}\sigma_p^j> 0\right]\leq -\frac{1}{128}\sqrt{i}\alpha G.
\end{equation*}

\subsection{Proof of Lemma \ref{lem:expDev1}}
\begin{proof}
We skip the superscript in the notation because we'll be working inside an epoch. 

Let $s_i:=\sum_{p=1}^i\sigma_p$. First, we prove the upper bound. We have that 
\begin{align*}
\mathbb{E}[|s_i|] &= \mathbb{E}\left[\left|\sum_{p=1}^i\sigma_p\right|\right]&\\
 &\leq \sqrt{ \mathbb{E}\left[\left(\sum_{p=1}^i\sigma_p\right)^2\right]}&\tag*{[Jensen's inequality]}\\
 &= \sqrt{ \sum_{p=1}^i\mathbb{E}[(\sigma_p)^2]+2\sum_{k<p\leq i}\mathbb{E}[\sigma_k \sigma_p]}&\\
 &= \sqrt{ i+2\sum_{k<p\leq i}\mathbb{E}[\sigma_k \sigma_p]}.&
\end{align*}
Now, $\mathbb{E}(\sigma_k \sigma_p) < 0$ because they are negatively correlated. Hence,
\begin{equation*}
\mathbb{E}[|s_i|] \leq \sqrt{ i}.
\end{equation*}
For the lower bound we start by decomposing the expectation into sum of conditional expectations:
\begin{align}
\mathbb{E}[|s_i|] &= \mathbb{P}(s_{i-1}\sigma_i\geq0)\mathbb{E}\left[|s_{i}|\,\Big|\,s_{i-1}\sigma_i\geq 0\right]+\mathbb{P}(s_{i-1}\sigma_i<0)\mathbb{E}\left[|s_{i}|\,\Big|\, s_{i-1}\sigma_i<0\right]&\nonumber\\
&= \mathbb{P}(s_{i-1}\sigma_i\geq0)\mathbb{E}\left[|s_{i-1}|+1\,\Big|\,s_{i-1}\sigma_i\geq0\right]\tag*{[Since $s_{i}=s_{i-1}+1$ if $s_{i-1}\sigma_i\geq0$]}\nonumber\\
&\quad +\mathbb{P}(s_{i-1}\sigma_i<0)\mathbb{E}\left[|s_{i-1}|-1\,\Big|\,s_{i-1}\sigma_i<0\right] \tag*{[Since $s_{i}=s_{i-1}-1$ if $s_{i-1}\sigma_i<0$]}\nonumber\\
&= \mathbb{P}(s_{i-1}\sigma_i\geq0)\mathbb{E}\left[|s_{i-1}|\,\Big|\,s_{i-1}\sigma_i\geq0\right]+\mathbb{P}(s_{i-1}\sigma_i<0)\mathbb{E}\left[|s_{i-1}|\,\Big|\,s_{i-1}\sigma_i<0\right]\nonumber\\
&\quad+\mathbb{P}(s_{i-1}\sigma_i\geq0)-\mathbb{P}(s_{i-1}\sigma_i<0)\nonumber\\
&= \mathbb{E}[|s_{i-1}|]+\mathbb{P}(s_{i-1}\sigma_i\geq0)-\mathbb{P}(s_{i-1}\sigma_i<0)&\nonumber\\
&= \mathbb{E}[|s_{i-1}|]+\mathbb{P}(s_{i-1}=0)+\mathbb{P}(s_{i-1}\sigma_i>0)-\mathbb{P}(s_{i-1}\sigma_i<0).&\label{eq:siBound1New}
\end{align}
Intuitively, because $\mathbb{P}(s_{i-1}\sigma_i>0)\approx \mathbb{P}(s_{i-1}\sigma_i<0)$, then $\mathbb{E}[|s_i|]\approx\sum_{p=1}^i \mathbb{P}(s_{i}=0)$. This is relatively easy to compute using combinatorics and gives that $\mathbb{E}[|s_i|]\approx \Omega(\sqrt{i})$. We do the exact calculations next. 

Continuing on from Eq.~\eqref{eq:siBound1New} and decomposing the probabilities as sum of conditional probabilities,
\begin{align*}
\mathbb{E}[|s_i|]&= \mathbb{E}[|s_{i-1}|]+\mathbb{P}(s_{i-1}=0)+\mathbb{P}(s_{i-1}\sigma_i>0)-\mathbb{P}(s_{i-1}\sigma_i<0)&\\
&= \mathbb{E}[|s_{i-1}|]+\mathbb{P}(s_{i-1}=0)+\sum_{p=1}^{i-1}\mathbb{P}(|s_{i-1}|=p)\mathbb{P}\left(s_{i-1}\sigma_i>0\,\Big|\,|s_{i-1}|=p\right)\\
&\quad -\sum_{p=1}^{i-1}\mathbb{P}(|s_{i-1}|=p)\mathbb{P}\left(s_{i-1}\sigma_i<0\,\Big|\,|s_{i-1}|=p\right).&
\end{align*}
The term $\mathbb{P}\left(s_{i-1}\sigma_i>0\,\big|\,|s_{i-1}|=p\right)$ has a closed form solution: Assume WLOG that $s_i =0 $. Then, $\mathbb{P}(s_{i-1}\sigma_i>0|s_{i-1}=p)=\mathbb{P}(\sigma_i=+1|s_{i-1}=p)$ is just the probability of sampling a $+1$, when out of $i-1$ previous samples there were $p$ `$+1$'s more than `$-1$'s. This is just $\frac{(n-i -p+1)/2 }{n-i+1}$. Similarly, we can handle the other cases. This gives,
\begin{align*}
\mathbb{E}[|s_{i}|]&= \mathbb{E}[|s_{i-1}|]+\mathbb{P}(s_{i-1}=0)+\sum_{p=1}^{i-1}\mathbb{P}(|s_{i-1}|=p)\left(\frac{(n-i -p+1)/2 }{n-i+1}\right)\\
&\quad-\sum_{p=1}^{i-1}\mathbb{P}(|s_{i-1}|=p)\left(\frac{(n-i+p+1)/2 }{n-i+1}\right)&\\
&= \mathbb{E}[|s_{i-1}|]+\mathbb{P}(s_{i-1}=0)-\sum_{p=1}^{i-1}\mathbb{P}(|s_{i-1}|=p)\left(\frac{p}{n-i+1}\right)&\\
&= \mathbb{E}[|s_{i-1}|]+\mathbb{P}(s_{i-1}=0)-\mathbb{E}(|s_{i-1}|)\frac{1}{n-i+1}&\\
&= \mathbb{E}[|s_{i-1}|](1-1/(n-i+1))+\mathbb{P}(s_{i-1}=0)&\\
&\geq \mathbb{E}[|s_{i-1}|](1-2/n)+\mathbb{P}(s_{i-1}=0)&\tag*{[Since $i\leq n/2$]}\\
&= \mathbb{E}[|s_{i-1}|](1-2/n)+\mathds{1}_{i-1 \text{ is even}}\mathbb{P}(s_{i-1}=0).&\tag*{[$s_{i-1}$ can be 0 only if $i-1$ is even.]}
\end{align*}
To compute $\mathbb{P}(s_{i-1}=0)$, we first see that it is just the ratio of the number of ways of choosing $(i-1)/2$ positions for `$+1$'s in the first $i-1$ positions and then choosing $(n-i+1)/2$ positions from the remaining $n-i+1$ positions for the remaining `$+1$'s, to the total number of ways of choosing $n/2$ positions for `$+1$'s from $n$ positions. Thus,
\begin{equation}
\mathbb{P}(|s_{i-1}|=0)=\frac{\binom{i-1}{ (i-1)/2}\binom{n-i+1} {(n-i+1)/2}}{\binom{n}{n/2}}\label{eq:sum0probablitynew}
\end{equation}

Using this equality and continuing on,
\begin{align*}
\mathbb{E}[|s_{i}|]&\geq \mathbb{E}[|s_{i-1}|](1-2/n)+\mathds{1}_{i-1 \text{ is even}}\mathbb{P}(|s_{i-1}|=0)&\\
&= \mathbb{E}[|s_{i-1}|](1-2/n)+\mathds{1}_{i-1 \text{ is even}}\frac{{\binom{i-1}{(i-1)/2}}{\binom{n-i+1}{(n-i+1)/2}}}{{\binom{n}{n/2}}}.\tag*{[Using Eq.~\eqref{eq:sum0probablitynew}]}
\end{align*}
To bound the combinatorial expression above, we use the following approximation by \citet[Theorem~1]{cristinel11}: $\sqrt{\pi(2k+0.33)}(\frac{k}{e})^{k}\leq k! \leq \sqrt{\pi(2k+0.36)}(\frac{k}{e})^{k}$. Using this and continuing on the sequence of inequalities,
\begin{align*}
\mathbb{E}[|s_{i}|]&\geq \mathbb{E}[|s_{i-1}|](1-2/n)+\mathds{1}_{i-1 \text{ is even}}\frac{1}{16\sqrt{i-1}}\\
&= \mathbb{E}[|s_{i-2}|](1-2/n)^2+\mathds{1}_{i-2 \text{ is even}}\frac{1}{16\sqrt{i-2}}(1-2/n)+\mathds{1}_{i-1 \text{ is even}}\frac{8\pi^2}{e^5}\sqrt{\frac{1}{i-1}}&\\
&\geq \mathbb{E}[|s_{i-2}|](1-2/n)^2+\frac{1}{16\sqrt{i}}(1-2/n)&\\
&\geq \frac{(1-2/n)}{16\sqrt{i}}\sum_{p=0}^{\lfloor i/2 \rfloor - 1}(1-2/n)^{2p}&\\
&= \frac{(1-2/n)}{16\sqrt{i}}\frac{1-(1-2/n)^{2(\lfloor i/2 \rfloor)}}{\frac{4}{n}-\frac{4}{n^2}}.&
\end{align*}
We will use the inequality : $\forall x\in[0,1], m \geq 0: \ (1-x)^m\leq\frac{1}{1+mx}$ to upper bound the term $(1-2/n)^{\lceil i/2 \rceil}$. Thus, we get
\begin{align*}
\mathbb{E}[|s_{i}|]&\geq \frac{(1-2/n)}{16\sqrt{i}}\frac{1-\frac{1}{1+\frac{4}{n}\lfloor i/2\rfloor}}{\frac{4}{n}-\frac{4}{n^2}}&\\
&\geq \frac{(1-2/n)}{16\sqrt{i}}\frac{1-\frac{1}{1+\frac{i}{n}}}{\frac{4}{n}}&\\
&= \frac{(1-2/n)}{16\sqrt{i}}\frac{in}{n+i}&\\
&= \frac{(n-2)}{16}\frac{\sqrt{i}}{n+i}&\\
&\geq \frac{(3n/4)}{16}\frac{\sqrt{i}}{3n/2}&\\
&= \frac{\sqrt{i}}{32}.
\end{align*}
This proves the lower bound of the lemma.

Next we prove that $\mathbb{P}\left( \sum_{p=1}^{i}\sigma_p^j< 0\right)\geq \frac{1}{4}$ and $\mathbb{P}\left( \sum_{p=1}^{i}\sigma_p^j> 0\right)\geq \frac{1}{4}$. Firstly, we can see that $\mathbb{P}\left( \sum_{p=1}^{i}\sigma_p^j< 0\right)=P\left( \sum_{p=1}^{i}\sigma_p^j> 0\right)$ by symmtery and hence it is sufficient to show $\mathbb{P}\left(\sum_{p=1}^{i}\sigma_p^j= 0\right) \leq 1/2$. 

For odd $i$, $\mathbb{P}\left(\sum_{p=1}^{i}\sigma_p^j= 0\right)=0$ trivially. Thus, we focus only on even $i$. Towards that end,
\begin{align*}
\mathbb{P}\left( \sum_{p=1}^{i}\sigma_p^j= 0\right) &=\mathbb{P}\left( s_i= 0\right)\\
&= \frac{{\binom{i}{i/2}}{\binom{n-i}{(n-i)/2}}}{{\binom{n}{n/2}}}\tag*{[Using Eq.~\eqref{eq:sum0probablitynew}]}\\
&\leq \frac{4}{\sqrt{i}}&\tag*{[Using \citet[Theorem~1]{cristinel11} mentioned earlier.]}\\
&\leq \frac{8}{\sqrt{n}}\\
&\leq \frac{1}{2}.
\end{align*}

Thus, $\mathbb{P}\left( \sum_{p=1}^{i}\sigma_p^j< 0\right)=\mathbb{P}\left( \sum_{p=1}^{i}\sigma_p^j> 0\right) \geq 1/4$.

\end{proof}

\subsection{Proof of Corollary \ref{cor:2}}
\begin{align*}
\mathbb{E}[|x_h^j-x_0^j|]&=\mathbb{P}\left(\sum_{p=1}^i \sigma_p^i \leq 0\right)\mathbb{E}\left[|x_h^j-x_0^j|\middle|\sum_{p=1}^i \sigma_p^i \leq 0\right]+\mathbb{P}\left(\sum_{p=1}^i \sigma_p^i > 0\right)\mathbb{E}\left[|x_h^j-x_0^j|\middle|\sum_{p=1}^i \sigma_p^i > 0\right]\\
&\geq \mathbb{P}\left(\sum_{p=1}^i \sigma_p^i \leq 0\right)\mathbb{E}\left[|x_h^j-x_0^j|\middle|\sum_{p=1}^i \sigma_p^i \leq 0\right].
\end{align*}
Therefore,
\begin{align*}
\mathbb{E}\left[\left|x_h^j-x_0^j\right|\Bigg|\sum_{p=1}^i \sigma_p^i \leq 0\right] &\leq \frac{\mathbb{E}[|x_h^j-x_0^j|]}{\mathbb{P}(\sum_{p=1}^i \sigma_p^i \leq 0)}\\
&\leq 4\mathbb{E}[|x_h^j-x_0^j|]\tag*{[Using Lemma \ref{lem:expDev1}]}\\
&\leq 4(\sqrt{5h}G\alpha + |x_0^j|\sqrt{h\alpha L})\tag*{[Using Lemma \ref{lem:jainLem}]}.
\end{align*}

The other inequality can be proved similarly.

\newpage
\section{Proof of Corollary \ref{cor:lowerBoundExtension}}\label{app:cor}

Let $F_1(x)$ be the 1-Dimensional function from Theorem \ref{thm:lowerBound} with $L=2^{17}$ and $F_2(x)$ be the 1-Dimensional function from Proposition 1 in \citet[p.~10-12]{safran2019good} with $\lambda=1$. In particular under this setting, the function from Proposition 1 in their paper is the following:
\begin{align*}
F_2(x)=\frac{1}{n}\sum_{i=1}^n f_{2,i}(x) = \frac{x^2}{2}, \text{ where}\\
\forall i \in[n]: \quad f_{2,i}(x)=
\left\{
\begin{array}{cl}
\dfrac{ x^2}{2}+\dfrac{Gx}{2}, &\text{ if }x\geq 0\\
\dfrac{ x^2}{2}+\dfrac{Gx}{2},& \text{ if }x < 0.
\end{array}
\right.
\end{align*}
In particular, it is the same function as Theorem \ref{thm:lowerBound} from this paper if $L$ was set to 1. Then, $F_1(x)$ and $F_2(x)$ satisfy the following properties (for each $j\in \{1,2\}$):
\begin{enumerate}
	\item $F_j(x) = \frac{1}{n}\sum_{i=1}^n f_{j,i}(x)$.
	\item $F_j(x)$ satisfies Assumption \ref{ass:stronglyConvex}, $\forall i, f_{j,i}(x)$ satisfy Assumptions \ref{ass:convex} and \ref{ass:lipschitz}; and that in the 1-Dimensional space, Assumptions \ref{ass:boundedD} and \ref{ass:boundedG} are satisfied. 
	We prove these for our function $F_1(x)$, in the proof of Theorem \ref{thm:lowerBound}.
	These can be similarly also proved for $F_2(x)$.
\item There is a step size range $A_j\subset \mathbb{R}$ and an initialization $x_{j,0}$ such that after $K$ epochs of SGDo on $F_j(x)$ with any constant step size $\alpha \in A_j$,
\begin{equation}
\mathbb{E}[\|x_T-x_j^*\|^2]=\mathbb{E}[\|x_T\|^2]\geq C_l\frac{G^2 n}{T^2},\label{ineq:cor2newfinal}
\end{equation}
where $x^*_j=0$ is the minimizer of $F_j(x)$ and $C_l$ is a universal constant. When $L=2^{17}$, we have shown that $A_1=\left[\frac{1}{T}, \frac{C}{n}\right]$ for a universal constant $C$. The proof of Proposition 1 in \citet[p.~10-12]{safran2019good} can be modified slightly so that $F_2$ satisfies Ineq.~\eqref{ineq:cor2newfinal} with $A_2=[0,\infty)\setminus A_1$.
\end{enumerate}

For the rest of this proof, we will be working in a 2-Dimensional space. A 2-Dimensional vector will be represented as $x=[x^{(1)},x^{(2)}]$. The superscript in this section \textbf{does not} denote the epoch, it denotes the co-ordinate.

$\forall i$, we define the 2-Dimensional functions $f_i(x) = f_{1,i}(x^{(1)})+f_{2,i}(x^{(2)})$ and $F(x) = \frac{1}{n}\sum_{i=1}^n f_{i}(x)$. 
Then, $F(x)$ satisfies Assumption \ref{ass:stronglyConvex} and $\forall i, f_{i}(x)$ satisfy Assumptions \ref{ass:convex} and \ref{ass:lipschitz}. 
Let $x^*=[x_1^*,x_2^*]=[0,0]$ denote the minimizer of $F(x)$.
Consider the 2-Dimensional square such which is defined by $\{x:\|x-x^*\|_\infty \leq D\}$. Then, in this domain, Assumptions \ref{ass:boundedD} and \ref{ass:boundedG} are satisfied with constants $D':=D\sqrt{2}$ and $G':=G\sqrt{2}$ respectively. 

Note that the gradient of $f_i(x)$ with respect to $x^{(1)}$ is just $\nabla f_{1,i}(x^{(1)})$ and with respect to $x^{(2)}$ is just $\nabla f_{2,i}(x^{(2)})$. Thus, SGDo essentially operates independently along each of the two co-ordinates. Now, for any step length $\alpha\in \mathbb{R}^+$, we know that $\alpha \in A_1$ or $\alpha \in A_2$. Then using Ineq.~\eqref{ineq:cor2newfinal}, at least one of the following is true: 
\begin{equation*}
\mathbb{E}[\|x_T^{(1)}-(x^*)^{(1)}\|^2]=\mathbb{E}[\|x_T^{(1)}-x^*_1\|^2]\geq C_l\frac{G^2n}{T^2}
\end{equation*}
or
\begin{equation*}
\mathbb{E}[\|x_T^{(2)}-(x^*)^{(2)}\|^2]=\mathbb{E}[\|x_T^{(2)}-x^*_2\|^2]\geq C_l\frac{G^2 n}{T^2}.
\end{equation*}
Therefore, the overall error
\begin{equation*}
\mathbb{E}[\|x_T-x^*\|^2]\geq C_l\frac{G^2n}{T^2}.
\end{equation*}

\newpage
\section{Numerical results}\label{app:exp}

{\small
\begin{figure}[H]
	\centering
	\subfigure[$\alpha = {1}/{T}$]{
	\includegraphics[trim={2em 0em 2em 2em},clip,width=0.35\textwidth]{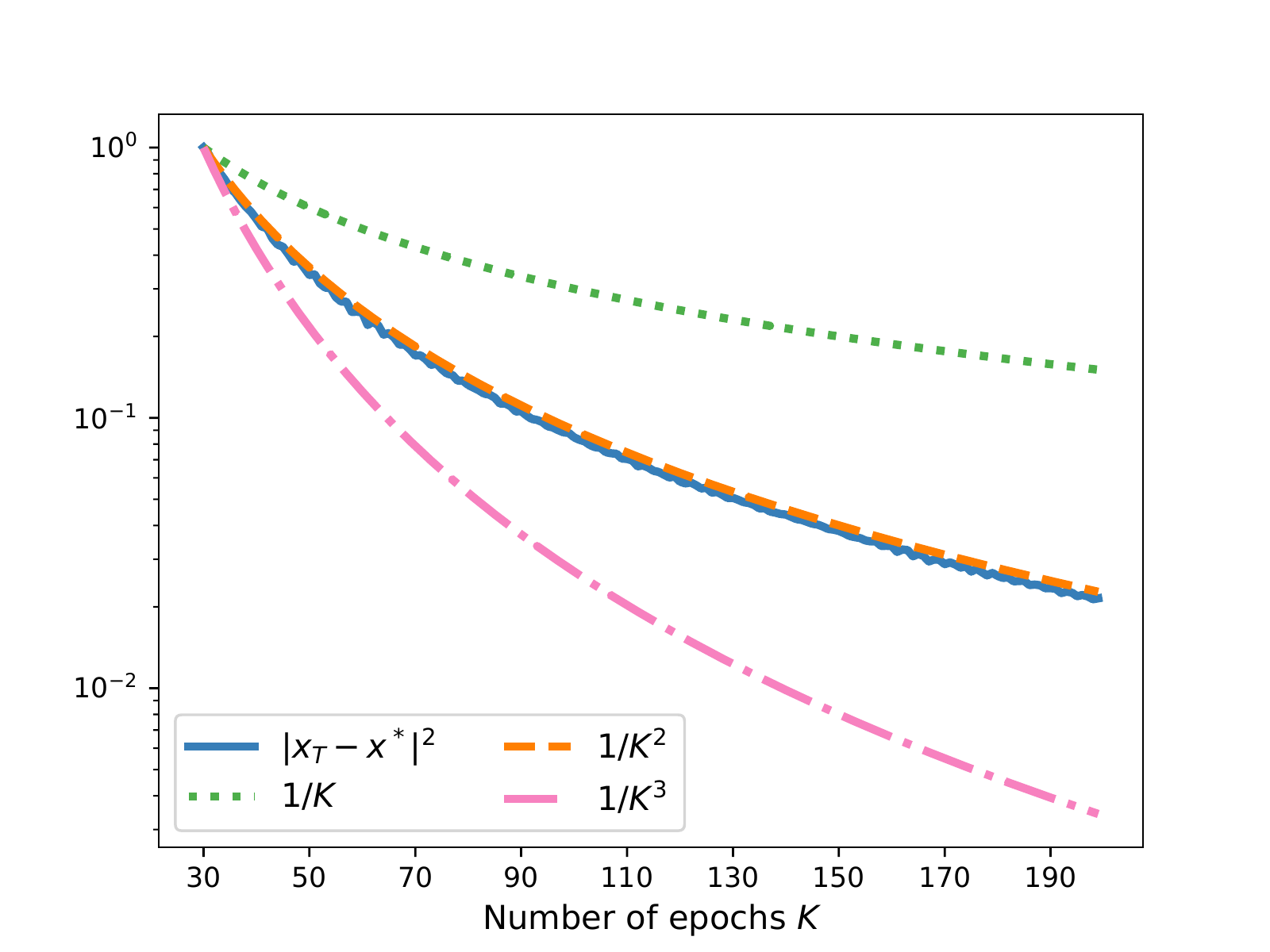}
	\includegraphics[trim={2em 0em 2em 2em},clip,width=0.35\textwidth]{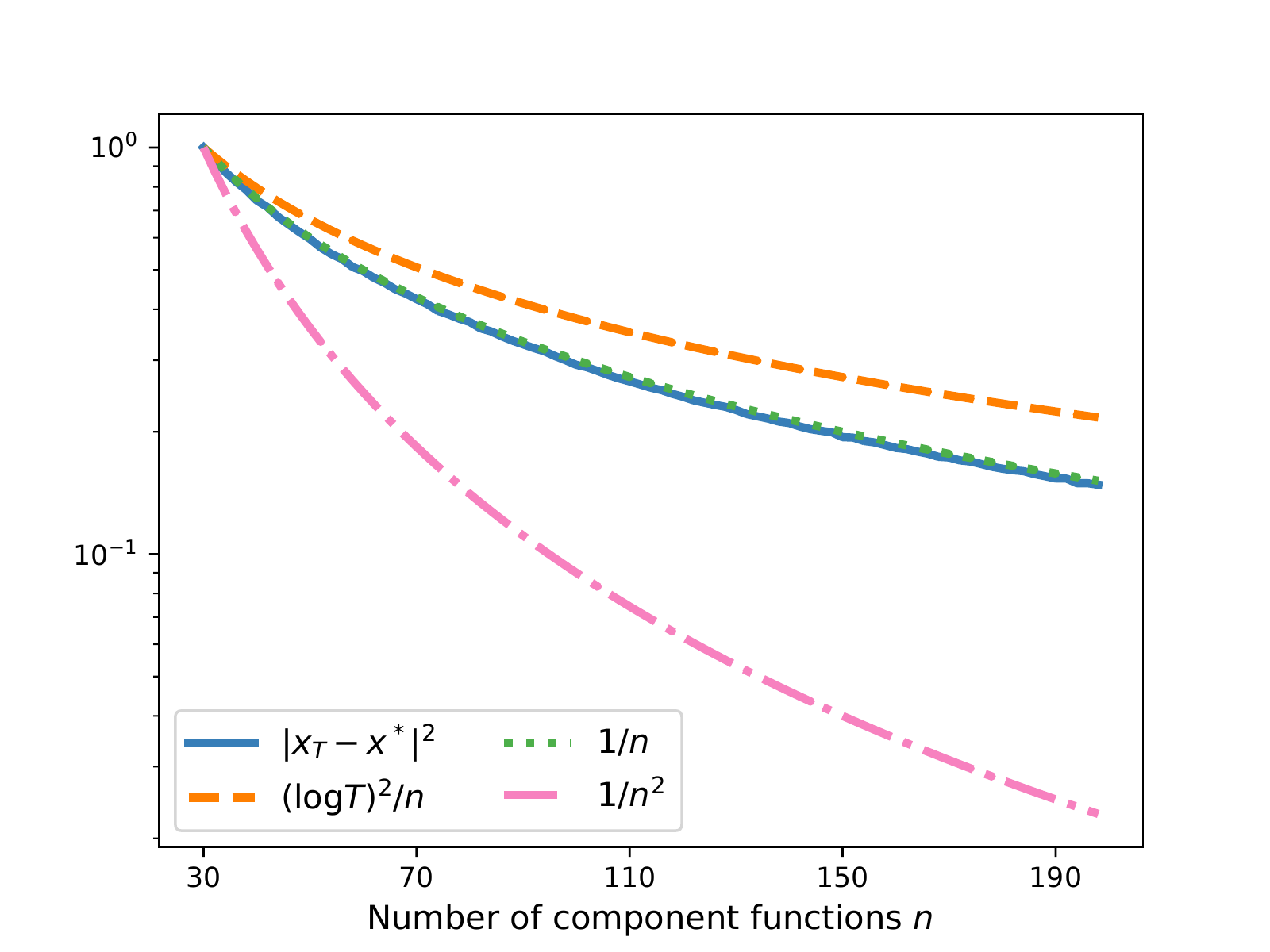}
}
\vspace{-1.0 em}
	\subfigure[$ \alpha = {2 \log{T}}/{T}$]{
	\includegraphics[trim={2em 0em 2em 2em},clip,width=0.35\textwidth]{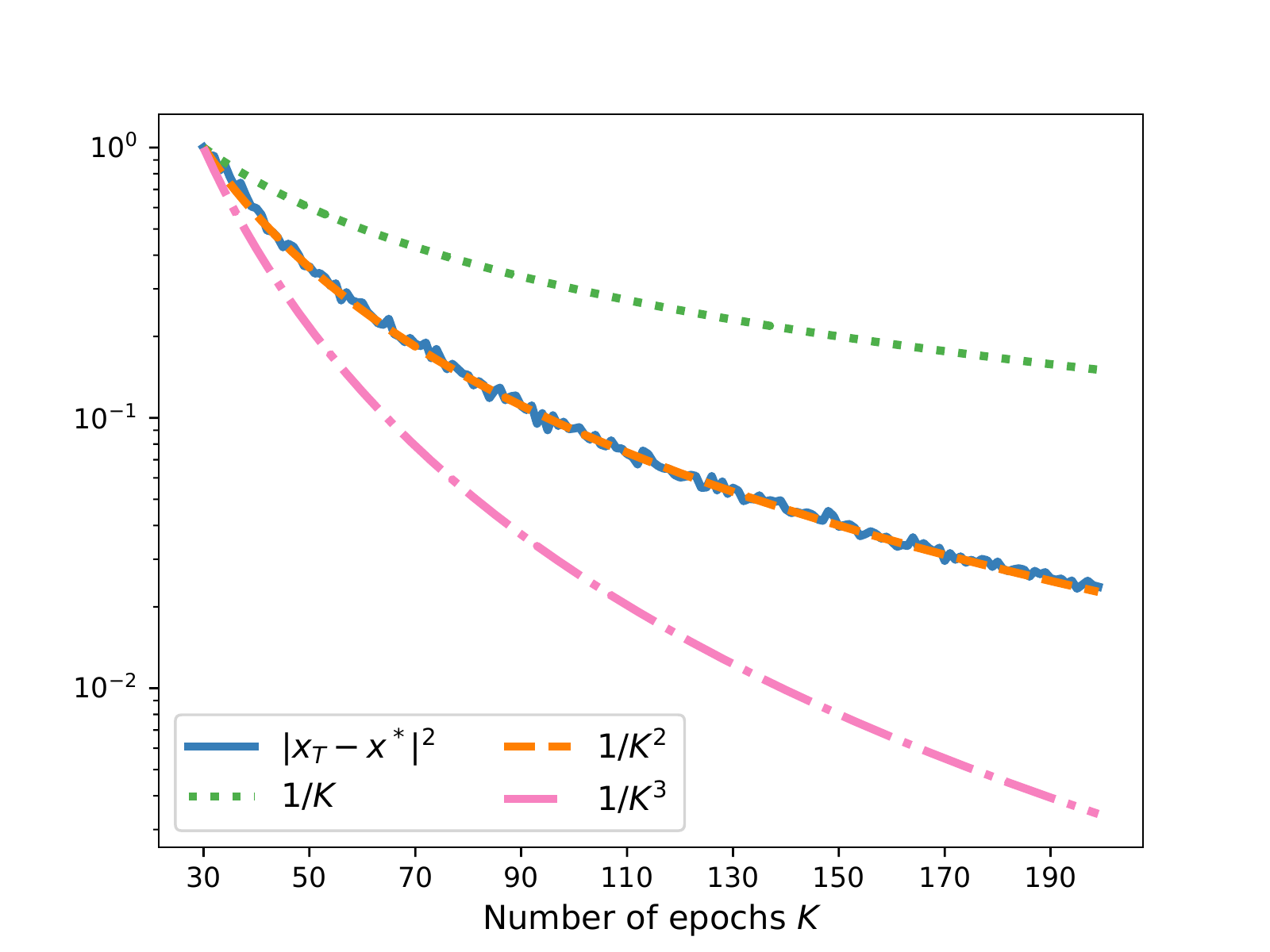}
	\includegraphics[trim={2em 0em 2em 2em},clip,width=0.35\textwidth]{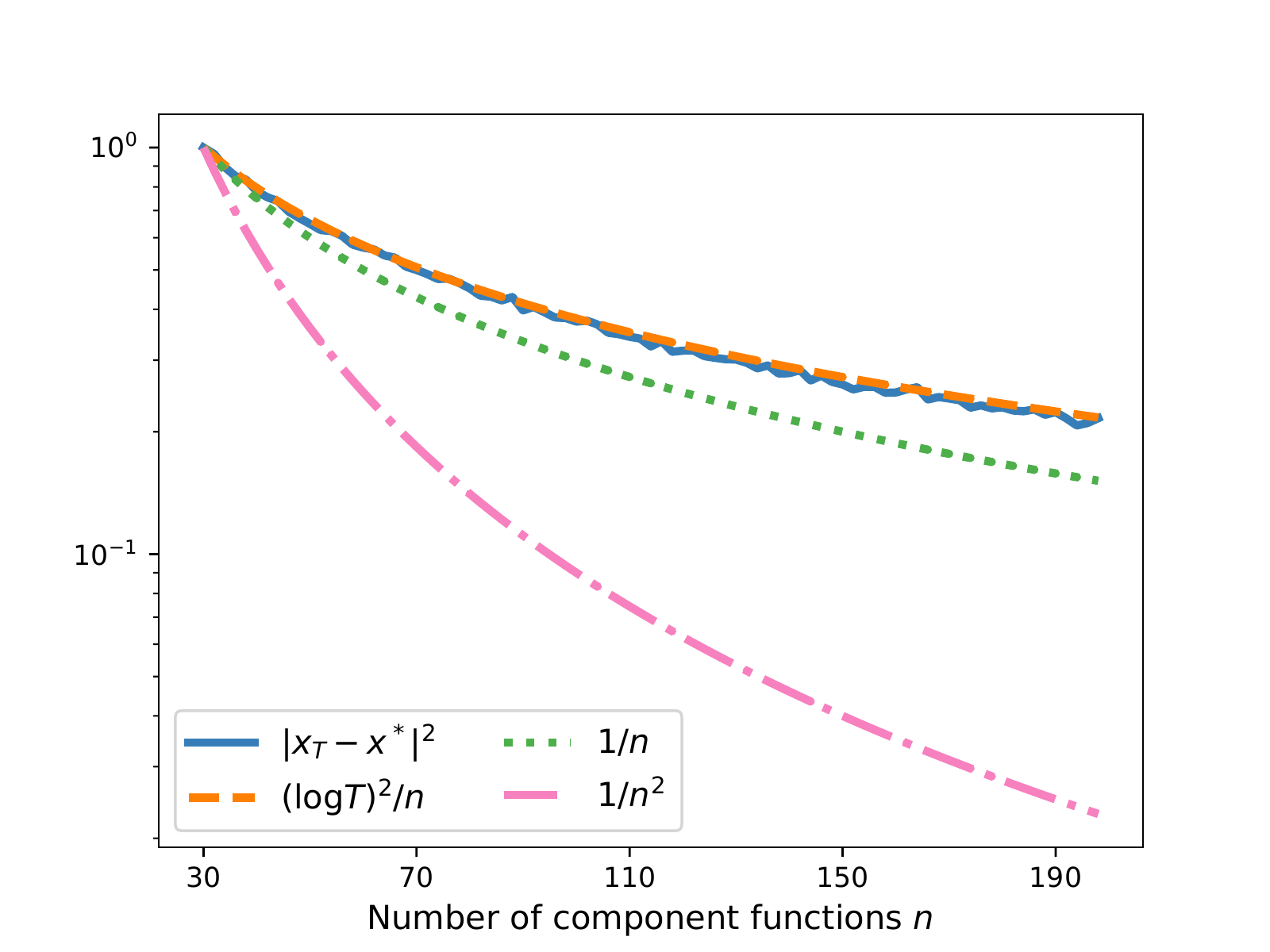}
}
\vspace{-1.0 em}
	\subfigure[$ \alpha = {8 \log{T}}/{T}$]{
	\includegraphics[trim={2em 0em 2em 2em},clip,width=0.35\textwidth]{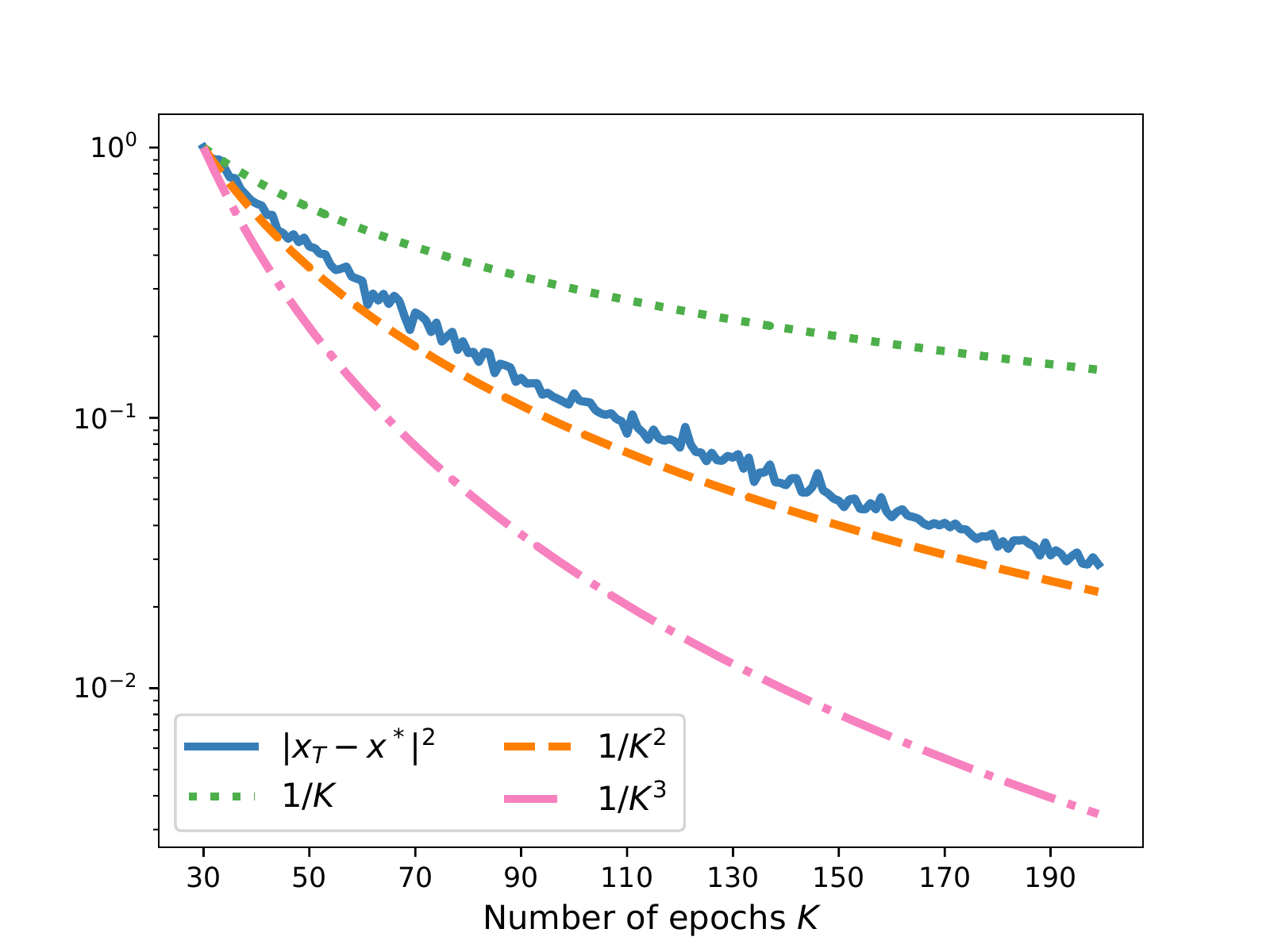}
	\includegraphics[trim={2em 0em 2em 2em},clip,width=0.35\textwidth]{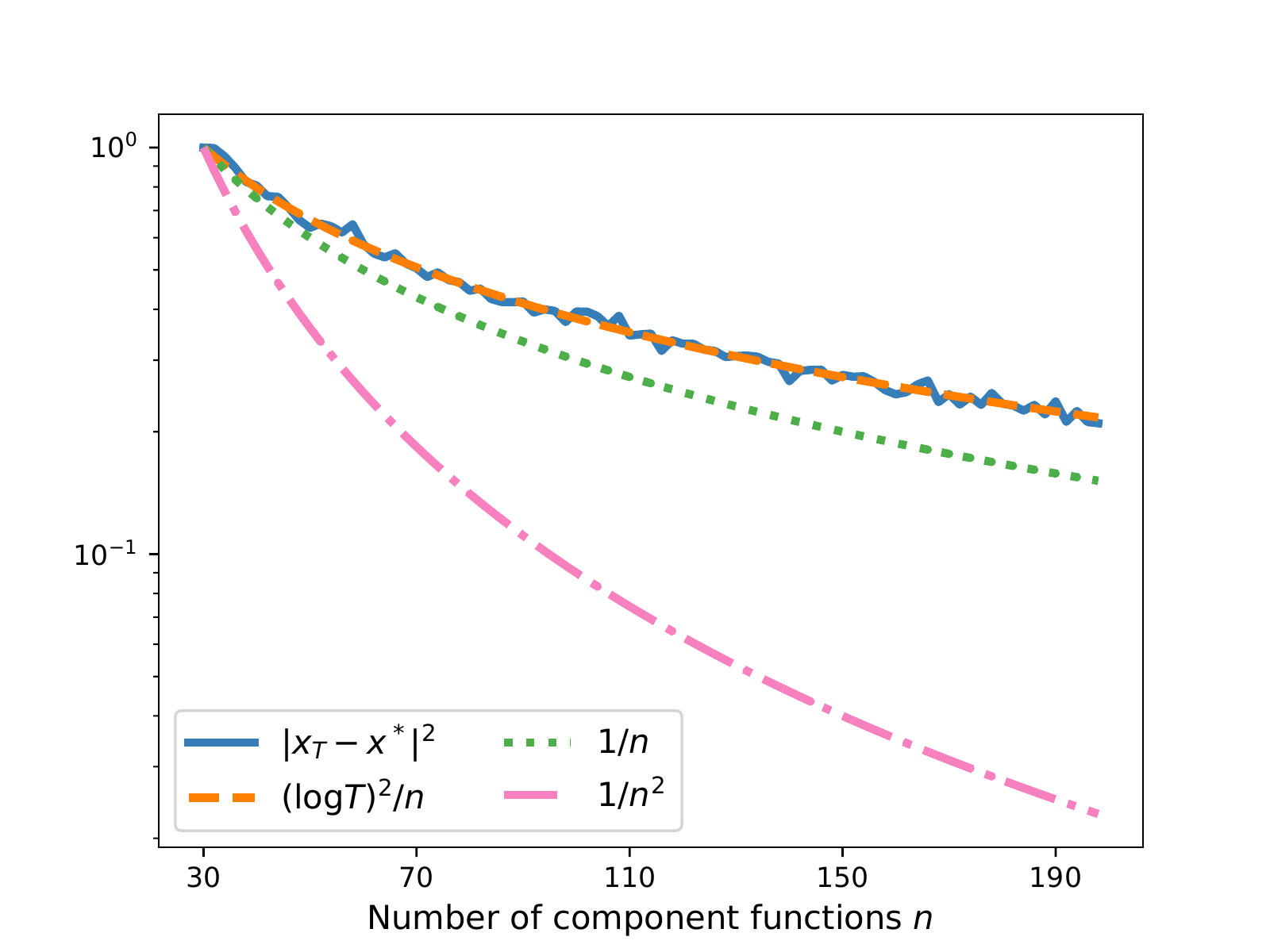}}
	\vspace{-1.0 em}
	\subfigure[$ \alpha = {1}/{n}$]{
	\includegraphics[trim={2em 0em 2em 2em},clip,width=0.35\textwidth]{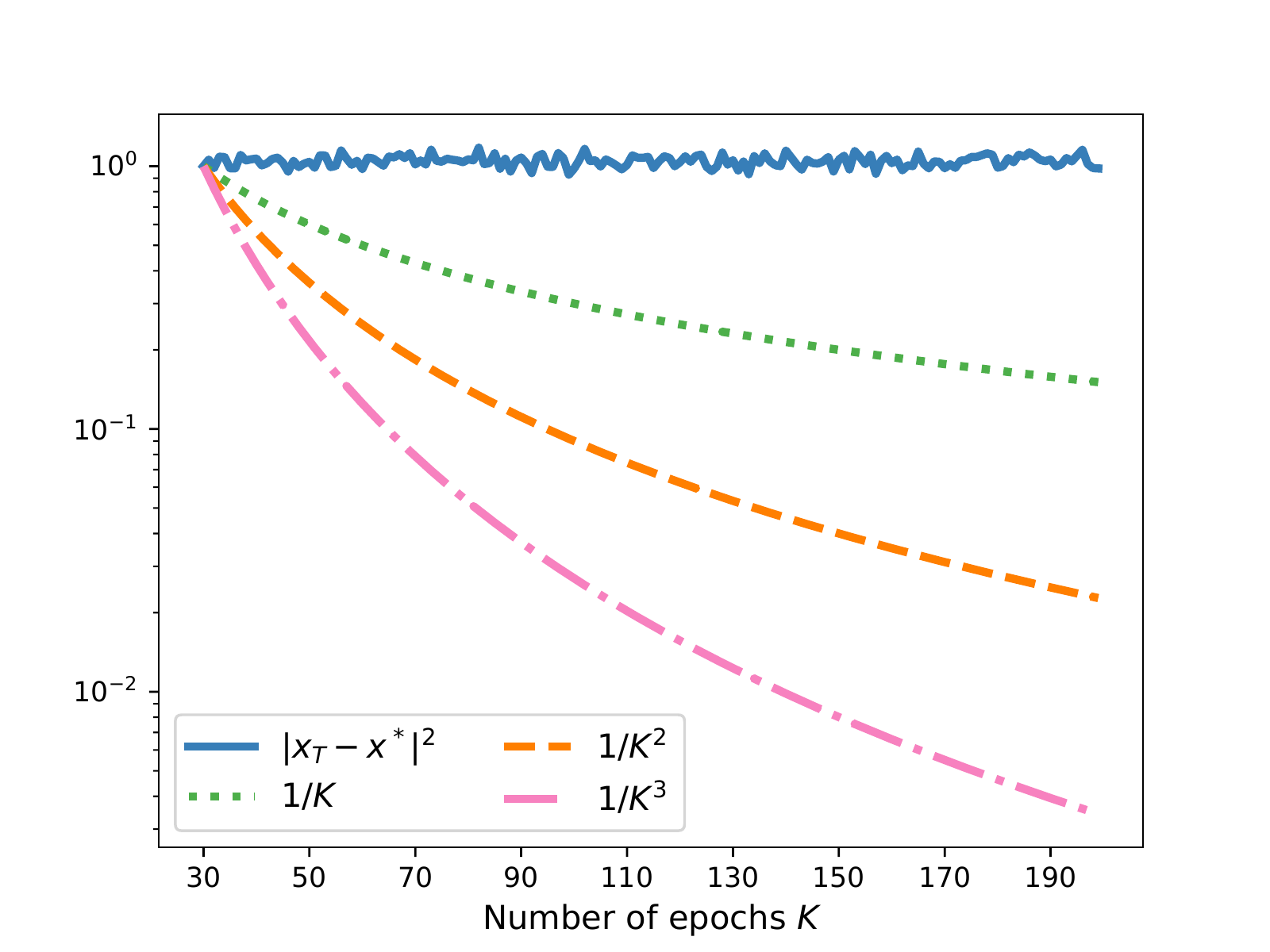}
	\includegraphics[trim={2em 0em 2em 2em},clip,width=0.35\textwidth]{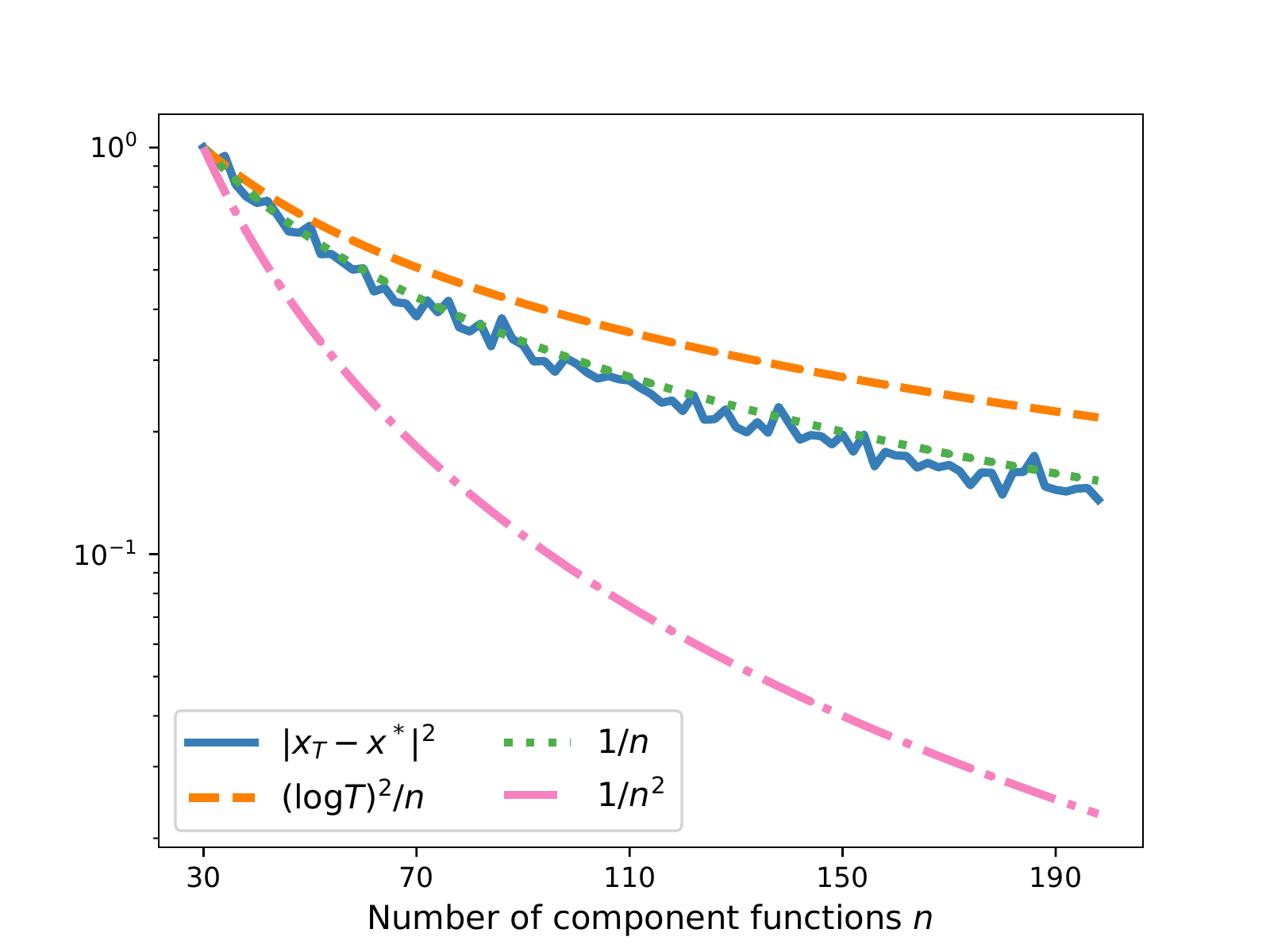}}
	\caption{Running SGDo with different step size regimes on the function $F$ used in our lower bound, Theorem~\ref{thm:lowerBound}. The setup is the same as described in Subsection~\ref{sec:numVerificationnew}.}
\end{figure}
}

\end{document}